\documentclass{article}
\usepackage{iclr2027_conference,times}
\usepackage{times}
\usepackage{amsmath,amssymb,amsthm,bm}
\usepackage{array}
\usepackage{booktabs}
\usepackage{float}
\usepackage{graphicx}
\usepackage{subcaption}
\usepackage{microtype}
\usepackage{xcolor}
\usepackage{hyperref}
\usepackage{url}
\hypersetup{hidelinks}

\title{GRPO-QPS: Target-Preserving Reinforcement Learning for Quantum Posterior Sampling}

\author{Yufeng Wang \thanks{These authors contributed equally to this work.}\\
  Stony Brook University \\
  \And
  Parivesh Priye $^*$\thanks{Corresponding author.} \\
  Georgia Institute of Technology \\
  \And
  Lu Wei \\
  Stony Brook University \\
  \And
  Haibin Ling \\
Westlake University \\}

\iclrfinalcopy
\makeatletter
\g@addto@macro\@maketitle{\lhead{}}
\makeatother

\newcommand{\method}{GRPO-QPS}
\newcommand{\base}{BuresTomFlow}
\newcommand{\flowgrpo}{Flow-GRPO}
\newcommand{\fixed}{matched control}
\newcommand{\bures}{d_{\mathrm B}}
\newcommand{\tr}{\operatorname{Tr}}
\newcommand{\E}{\mathbb{E}}
\newcommand{\KL}{\operatorname{KL}}
\newcolumntype{L}[1]{>{\raggedright\arraybackslash}p{#1}}
\newtheorem{proposition}{Proposition}

\begin{document}
\maketitle

\begin{abstract}
Bayesian quantum tomography requires efficient inference while preserving a posterior fixed by the prior and Born likelihood. Learned transport provides fast amortized samples, but reward tuning can reshape the generated distribution rather than improve exploration of this fixed target. We introduce \method{}, a target-preserving framework in which GRPO learns proposal behavior and an exact Metropolis correction preserves the posterior after training. Across the evaluated reconstruction benchmarks, \method{} improves over \base{} and \flowgrpo{} on thermal, cat, Dicke, and cluster families, and it closely matches an exact two-qubit reference posterior. Tuned conventional MCMC is slightly stronger on several original continuous benchmarks where the available fixed proposals already match the posterior geometry well. To test whether this reflects a fundamental limitation of learned exploration, we evaluate a more challenging multimodal thermal posterior. At six qubits and 800 shots, the learned proposal achieves a minimum effective sample size of 102 per $1{,}000$ likelihood calls, compared with 28 for prior independence, 27 for a tuned fixed mixture, and 20 for Haario adaptive Metropolis. A record-conditioned policy also transfers to unseen 3,000-shot records, matching or exceeding the strongest conventional baseline in all nine held-out seed-record comparisons. These results show that \method{} combines target-preserving Bayesian inference with broad gains over learned transport baselines and a sampling advantage when efficient exploration requires proposal geometry beyond the evaluated conventional kernels.
\end{abstract}

\section{Introduction}

Quantum tomography reconstructs an unknown density matrix from finite measurement records. The problem is intrinsically uncertain: many physical states can explain the same observations while predicting different outcomes for measurements that were not performed. Bayesian tomography represents this uncertainty through a posterior over states rather than a single point estimate \citep{granade2016practical}. This probabilistic view is especially important in low-data regimes, where posterior width and multimodality can be scientifically meaningful. It also imposes a strict requirement on learned inference. Once the prior and Born likelihood are specified, they define where posterior probability belongs. A learning algorithm may improve computational efficiency or exploration, but it should not silently replace that physical target with a distribution preferred by its reward.

Amortized generative models address the computational cost of solving a new inference problem for every measurement record \citep{torlai2018neural,carrasquilla2019reconstructing,zhu2022flexible}. Flow Matching and Riemannian Flow Matching learn continuous transports between simple source distributions and data-dependent targets \citep{lipman2023flow,chen2024riemannian}; \base{} specializes this idea to density matrices through Bures geometry \citep{burestomflow2026,bures1969extension,uhlmann1976transition,jozsa1994fidelity,bhatia2019bures}. These methods can generate states quickly, but their output distribution is determined by the learned transport. Reward-based refinement, including \flowgrpo{}, adds another degree of freedom by favoring trajectories with desirable endpoint properties \citep{shao2024deepseekmath,liu2025flowgrpo}. That flexibility is useful when the reward defines the goal. In Bayesian inference, however, directly rewarding the generated distribution can mix two different objectives: exploring the posterior more effectively and changing the posterior itself.

We introduce \method{} to separate these objectives. The learned component operates inside a Metropolis-corrected posterior sampler rather than replacing the posterior with a reward-defined generator. \base{} supplies measurement-conditioned initial states, and a GRPO policy adapts how new states are proposed from the current state and record. The complete forward and reverse proposal probabilities enter the Metropolis acceptance ratio, and posterior samples are collected only after the policy is frozen. Under the stated assumptions, the prior and Born likelihood therefore continue to determine the stationary posterior while learning changes only the efficiency and geometry of exploration. This design also lets us ask a more informative question than whether learning helps on average: when fixed physical proposals already explore well, adaptation may have little to add, whereas more structured posterior geometry can create a genuine role for a learned proposal.

Our contributions are threefold. First, we formulate a target-preserving GRPO framework for Bayesian quantum tomography, combining learned proposal adaptation with an exact Metropolis correction so that reward optimization does not redefine the scientific target. Second, we show that the resulting inference pipeline improves reconstruction and posterior agreement over \base{} and \flowgrpo{} across several physical state families and evaluation settings. Third, we compare learned exploration with matched conventional samplers and identify two regimes: the original benchmarks, where tuned fixed proposals remain highly competitive, and a more challenging multimodal thermal posterior, where a richer learned drift and covariance proposal surpasses every conventional MCMC baseline tested and transfers to unseen measurement records. Together, these results establish both the practical value of \method{} and the posterior structures for which learned adaptation becomes most useful.

\section{Related work}

\textbf{Posterior inference and geometric transport.}
Maximum likelihood estimation returns a single best-fit state, whereas Bayesian
tomography assigns probability to all states according to the prior and observed
record \citep{granade2016practical}. Repeating full Bayesian inference for every
record can be expensive, which motivates amortized generators. Neural network
tomography learns flexible state models from measurement statistics
\citep{torlai2018neural,carrasquilla2019reconstructing}, generative query networks
infer unmeasured outcomes through a shared latent representation
\citep{zhu2022flexible}, and classical shadows estimate many properties from
randomized measurements \citep{huang2020predicting}. These methods can be
computationally attractive, but a physically valid generator does not
necessarily reproduce a specified posterior distribution. \base{} addresses the
geometry of density matrices through Bures-based transport, while our method uses
that learned transport to initialize a Markov chain whose retained samples are
corrected toward the declared posterior. Section~\ref{sec:maintrack} compares the
resulting estimates with maximum likelihood, classical shadows, a neural density
operator, and prior averaging under matched measurement records.

\textbf{Reward tuning and target preservation.}
GRPO forms relative advantages within a sampled group without learning a separate
value model \citep{shao2024deepseekmath}. \flowgrpo{} applies this idea to
stochastic flow trajectories \citep{liu2025flowgrpo}, while Adjoint Matching
uses a related stochastic-control formulation \citep{domingo2024adjoint}. These
methods intentionally modify the generated distribution in response to a reward.
That is appropriate when the reward defines the target, but Bayesian tomography
already has a target specified by the prior and Born likelihood. \method{}
instead learns a proposal and then applies the complete Metropolis correction, so
the frozen chain preserves the same stationary posterior
\citep{hastings1970monte}. This design connects to structured and learned
proposal methods, including preconditioned Crank--Nicolson updates
\citep{cotter2013pcn}, adaptive MCMC \citep{roberts2007adaptive}, locally informed
mixtures \citep{zanella2020informed}, neural or policy-guided dynamics with an
acceptance correction \citep{levy2018l2hmc,bojesen2018policy}, flow-assisted MCMC
\citep{cabezas2024markovian}, and reinforcement-learned Metropolis--Hastings
kernels \citep{wang2025rlmh}. Recent work also reports that acceptance-rate and
squared-jump rewards can provide weak learning signals
\citep{wang2025harnessing}, which is consistent with our analysis of movement
based rewards. Our contribution is a probability-consistent integration of these
ideas for Bayesian mixed-state tomography, together with controlled experiments
that distinguish two regimes: one in which tuned conventional proposals already
explore effectively, and another in which a richer learned proposal improves the
minimum effective sample size. In every case, the learned kernel is frozen before
posterior samples are retained.

\section{Preliminaries and Methods}
\label{sec:preliminariesmethods}

\subsection{The target distribution and its parameterization}

Let $x$ be the declared parameter and $\rho(x)\succeq0$, $\tr\rho(x)=1$, its
density matrix; a record $y$ gives effects $E_m$, counts $c_m$, and Born
probabilities $p_m(x)=\tr[E_m\rho(x)]$. For a prior $\mu_0$ the posterior is
\begin{equation}
  \Pi_X(dx\mid y) = \frac{1}{Z(y)}
  \prod_m p_m(x)^{c_m}\,\mu_0(dx),
  \label{eq:posterior}
\end{equation}
with density-matrix posterior the pushforward
$\Pi_\rho(B\mid y)=\Pi_X(\{x:\rho(x)\in B\}\mid y)$, a distinction that matters
when several parameters map to one density matrix. The parameter $x$ is the dense
state, a complex Gaussian factor $F$, or bounded physical coordinates depending on
the study; Appendix~\ref{app:protocol} lists every prior. We measure
reconstruction and posterior error by the Bures distance
\citep{bures1969extension,uhlmann1976transition,bhatia2019bures},
\begin{equation}
  \bures^2(\rho,\sigma)=2-2\tr\!\left[
  \left(\rho^{1/2}\sigma\rho^{1/2}\right)^{1/2}\right],
  \label{eq:bures}
\end{equation}
the root-fidelity distance, which also defines the posterior center error and
credible balls. Throughout the paper, reconstruction accuracy and posterior
sampling quality are treated as distinct questions. A method can place its
posterior center close to the simulated truth while assigning the wrong
probability to nearby states, and conversely a sampler can represent the declared
posterior accurately even when finite data leave that posterior broad. We
therefore report point-reconstruction metrics together with distributional or
chain-based diagnostics whenever an appropriate reference is available.

\begin{figure}[h]
  \centering
  \includegraphics[width=\textwidth]{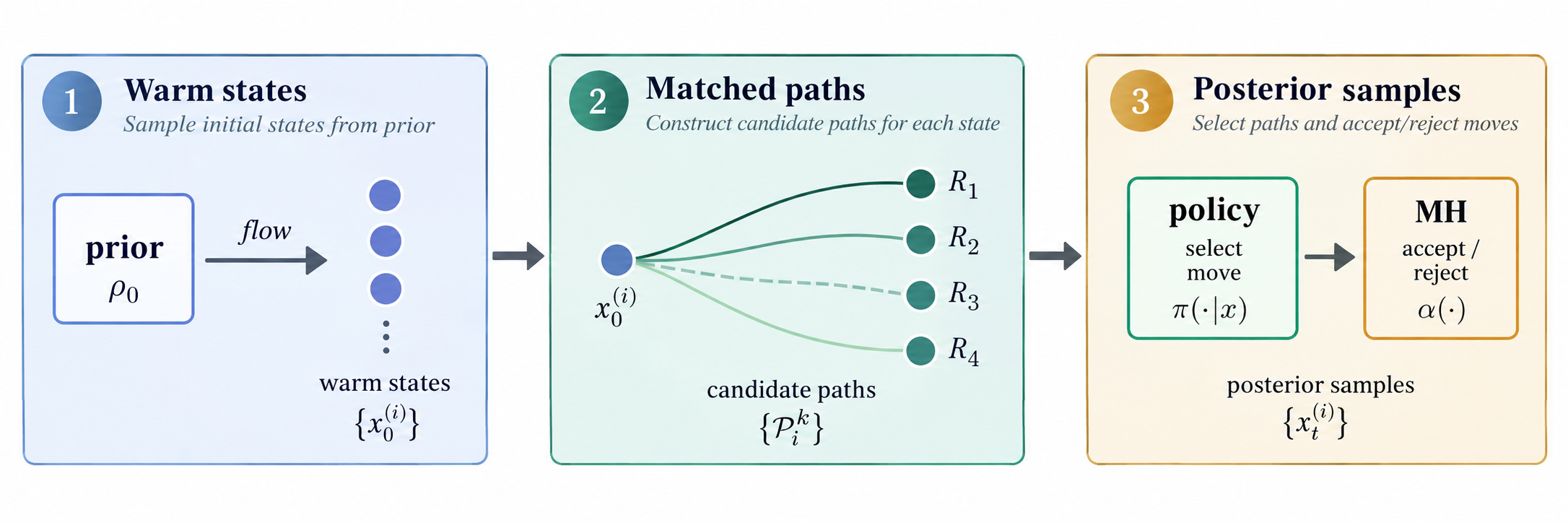}
  \caption{Method overview. \base{} supplies measurement-conditioned initial
    states, while \method{} learns proposal behavior from matched stochastic
    trajectories and freezes the policy before posterior collection. The
    schematic shows the restricted move-selection configuration. The multimodal
    experiment uses the same target-preserving construction with a learned drift
    and covariance proposal. In both cases, the complete forward and reverse
    proposal densities enter the Metropolis accept or reject calculation, so the
    declared parameter-space posterior remains stationary.}
  \label{fig:method}
\end{figure}

\subsection{BuresTomFlow provides flow-initialized states}
\label{sec:warm}

The base model encodes the record $y$ as a fixed-size embedding $h(y)$ and
learns a conditional velocity $v_\theta(t,\rho,h(y))$. Integrating
\begin{equation}
  \frac{d\rho_t}{dt}=v_\theta(t,\rho_t,h(y)),\qquad \rho_0\sim q_0,
  \label{eq:flow}
\end{equation}
produces the flow-initialized distribution $q_\theta(\rho\mid y)$. The dense
construction uses Bures paths and a Bures tangent loss
\citep{burestomflow2026}, while larger systems transport purification factors $F$
with $\rho(F)=FF^\dagger/\tr(FF^\dagger)$, which enforces positivity and unit
trace by construction (Appendix~\ref{app:protocol}). The flow is fast and
amortized, but it can assign too little probability mass to some posterior
regions. Reweighting already generated samples cannot recover a region that the
flow rarely reaches. Figure~\ref{fig:method} therefore uses the flow to provide
initial states, after which a physical Markov chain explores the declared
posterior.

\subsection{Flow-GRPO tunes the transport distribution}

\flowgrpo{} perturbs the flow velocity along $G$ stochastic trajectories that
share a record and starting distribution. With endpoint $\rho_g$ and a finite
posterior comparator $\{\tilde\rho_j\}_{j=1}^{M}$, the energy score
\begin{equation}
  R_g=-\frac{2}{M}\sum_j\bures(\rho_g,\tilde\rho_j)
  +\frac{1}{G-1}\sum_{g'\ne g}\bures(\rho_g,\rho_{g'})
  \label{eq:energyreward}
\end{equation}
rewards agreement with the reference in the first term and discourages the
trajectories from concentrating at the same endpoint in the second. A group
advantage $A_g=(R_g-\bar R)/(\operatorname{sd}(R)+\epsilon)$ and a clipped
likelihood ratio update the velocity policy, while a penalty keeps the updated
policy near the base flow. This procedure changes how the flow distributes its
samples, but it does not provide a posterior-invariance guarantee.

\subsection{GRPO-QPS learns target-preserving proposal geometry}
\label{sec:grpqmproposal}

\method{} learns how to explore the posterior while leaving its target unchanged.
For a generic parameter $x$, let $Q_\phi(x'\mid x,y)$ be a learned proposal with
a defined reverse density. After freezing $\phi$, every proposal is corrected by
\begin{equation}
  \alpha(x,x')=\min\!\left\{1,
  \frac{\pi_X(x'\mid y)Q_\phi(x\mid x',y)}
       {\pi_X(x \mid y)Q_\phi(x'\mid x,y)}\right\}.
  \label{eq:mh}
\end{equation}
Thus learning can change local scale, direction, covariance, or mixture weight,
but it cannot change the stationary target when the assumptions of
Appendix~\ref{app:math} hold. This separation is important because proposal
quality and target correctness play different roles. A poor proposal may yield
highly correlated samples, but the Metropolis correction still prevents its
learned preferences from redefining the posterior. Conversely, a more expressive
proposal can improve exploration only if its forward and reverse probabilities
are evaluated correctly. GRPO-QPS therefore uses learning to search over sampling
strategies while leaving posterior validity to the acceptance rule.

We use two proposal parameterizations for two different questions. The first is a
restricted scale scheduler used in the reconstruction, continuous-family, and
enumerated diagnostics. For factor states the chain variable is the raw
purification factor $F$ of Section~\ref{sec:warm}, mapped to
$\rho(F)=FF^\dagger/\tr(FF^\dagger)$ only for likelihoods and observables. An
action $a$ selects a preconditioned Crank--Nicolson scale $\beta_a$
\citep{cotter2013pcn},
\begin{equation}
  F'=\sqrt{1-\beta_a^2}\,F+\beta_a\Xi,
  \qquad \Xi_{ij}\sim\mathcal{CN}(0,1),
  \label{eq:pcn}
\end{equation}
where $\Xi$ is a fresh standard complex Gaussian factor. Because this proposal is
reversible under the factor prior, the corresponding prior and kernel terms
cancel from Eq.~\ref{eq:mh}; the acceptance probability retains the likelihood
and the state-dependent forward and reverse action probabilities. Bounded
physical coordinates analogously use reflected local scales or an exact prior
refresh. This restricted class isolates whether learning \emph{which existing
move to use} can outperform tuning those same moves.

The multimodal study uses a more expressive proposal to test whether learning
the proposal geometry itself can outperform conventional MCMC.
Let $z$ denote logit coordinates for the bounded thermal parameters and
$\widetilde\pi(z\mid y)$ the corresponding posterior density, including the
change of variables. The policy produces a state-dependent Gaussian mixture,
\begin{equation}
  Q_\phi(z'\mid z,y)=\sum_{k=1}^{K}w_{\phi,k}(z,y)
  \mathcal N\!\left(z';z+d_{\phi,k}(z,y),
  L_{\phi,k}(z,y)L_{\phi,k}(z,y)^\top\right),
  \label{eq:learnedmixture}
\end{equation}
with learned mixture weights, drift, and Cholesky factors. The full mixture
density is evaluated in both directions in Eq.~\ref{eq:mh}. This is essential:
a state-dependent drift that crosses modes is precisely the component that would
bias the chain if it were used without the reverse-density correction.

The reward also differs between the diagnostic and final configurations. The
restricted scheduler initially uses accepted physical movement with a small
acceptance bonus. This reward is intentionally local and does not use the hidden
truth, but Sections~\ref{sec:rewarddesign} and
Appendix~\ref{sec:movementcounterexample} show that movement can be large while a
physical quantity remains correlated. The final multimodal proposal therefore
uses an estimation-aligned trajectory objective over the thermal coordinates and
a mode indicator. For the single-record analysis, the reference moments are computed from the same
posterior being studied. For transfer, the reference moments are estimated only
from long prior-independence chains on the \emph{training} records; no exact
reference or held-out posterior moment enters training. Training combines records
with broader and narrower posterior modes so that optimization does not drive the
acceptance probability toward zero. Appendix~\ref{app:multimodal} gives the full
protocol.

In all cases, GRPO compares conditionally matched trajectories and updates the
proposal parameters from relative trajectory quality. For the historical
scale-scheduling experiments the clipped objective is
\begin{equation}
  \mathcal L_{\mathrm{QM}}=-\E\!\left[
    \min(r_{g,t}A_g,\operatorname{clip}(r_{g,t},1-\epsilon,1+\epsilon)A_g)
  \right]+\lambda\KL(q_\phi\Vert q_{\mathrm{fixed}}),
  \label{eq:grpo}
\end{equation}
with the event probability including both action choice and acceptance. The
matched-objective appendix specifies the corrected trajectory-score variants used
for the exact diagnostics. Posterior samples are collected only after training is
complete and the policy is frozen.

The group compares proposal trajectories, not the eigenvalue weights of a
quantum state: every path evaluates the full density matrix
(Appendix~\ref{sec:groups}). Turning learning off while holding initialization,
proposal family, likelihood, acceptance rule, and budget fixed yields a
\fixed{} that isolates learned adaptation. Tuned MCMC, Haario adaptive Metropolis,
prior independence, symmetry-aware reflection, sequential Monte Carlo, and HMC
are then used where appropriate to determine whether the learned proposal offers
an advantage over a strong conventional alternative rather than only over its own
initialization.

\section{Results}

\subsection{GRPO-QPS improves reconstruction across physical families}

\begin{figure}[h]
\centering
\begin{subfigure}[b]{0.45\textwidth}
\centering
\includegraphics[width=\textwidth]{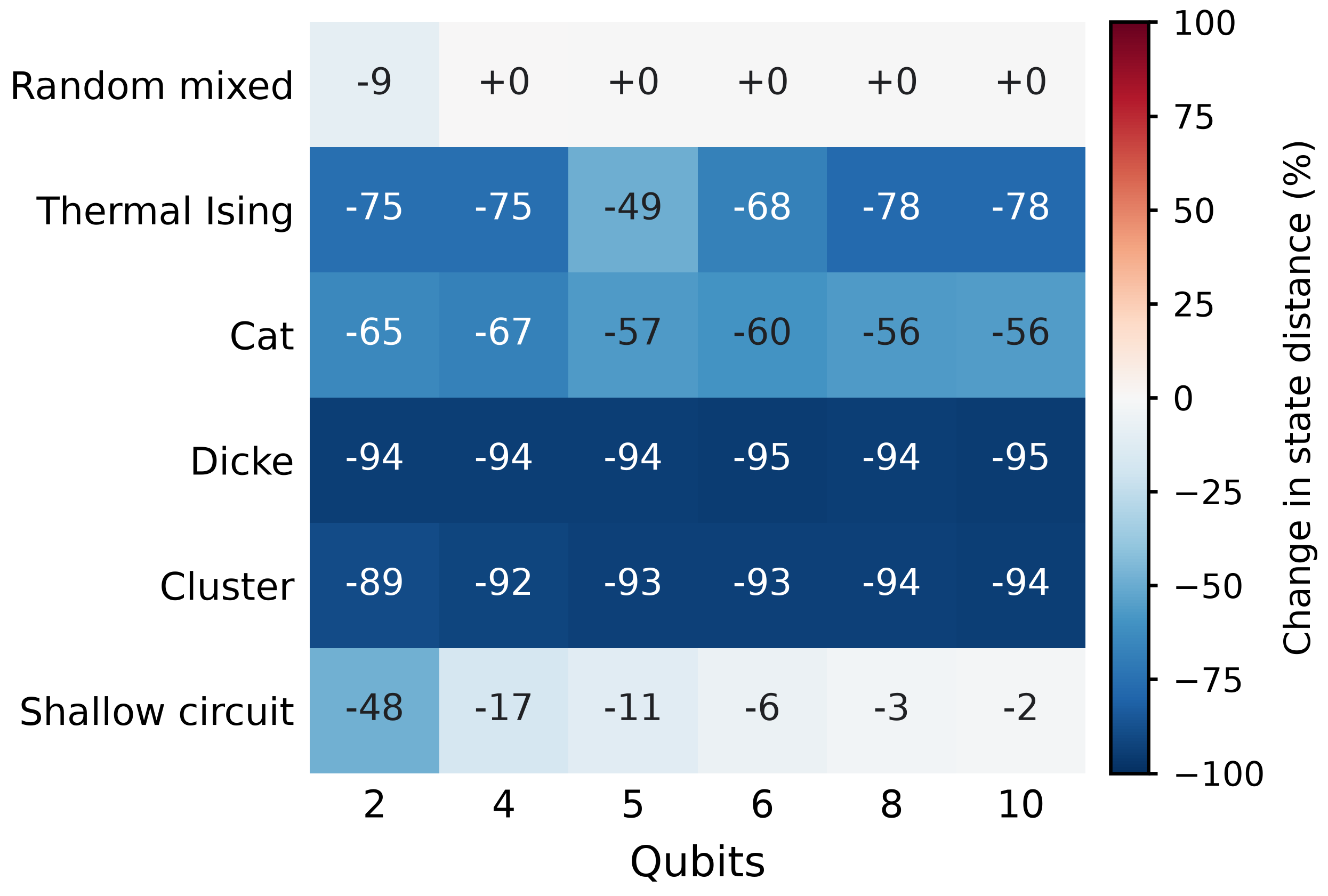}
\caption{Relative to \base{}}\label{fig:landscape-base}
\end{subfigure}\hfill
\begin{subfigure}[b]{0.45\textwidth}
\centering
\includegraphics[width=\textwidth]{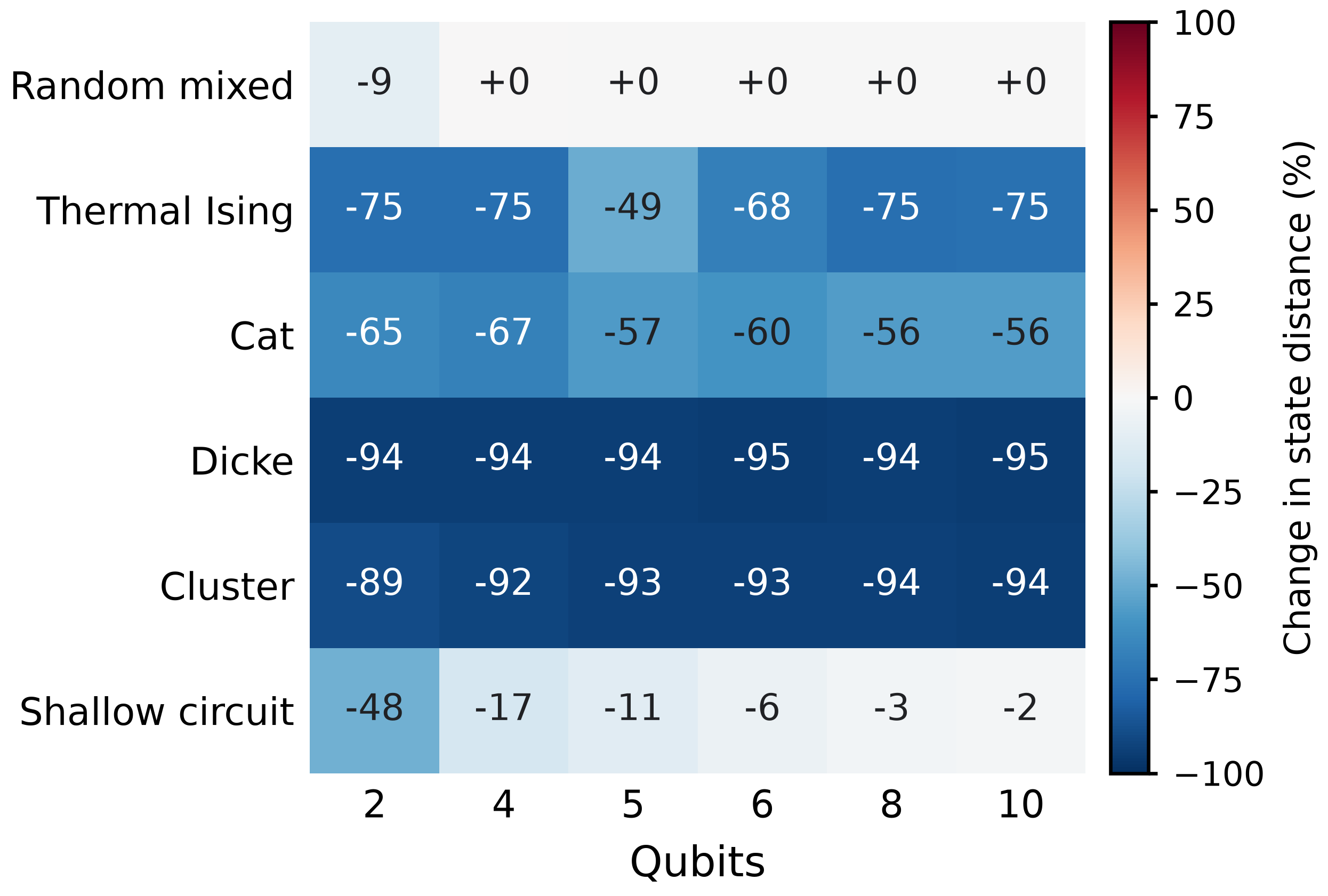}
\caption{Relative to \flowgrpo{}}\label{fig:landscape-flow}
\end{subfigure}
\caption{Reconstruction comparisons across six physical families and six system sizes. Each cell averages the predefined 8-shot and 32-shot studies. Negative values indicate lower Bures state distance for \method{}. Comparisons use the representation defined for each study.}
\label{fig:landscape}
\end{figure}

Across the reconstruction benchmarks, \method{} achieves lower Bures state distance than both flow baselines on thermal, cat, Dicke, and cluster states from two to ten qubits (Figure~\ref{fig:landscape}). Relative to \base{}, the reduction is approximately $94\%$ on Dicke states and above $89\%$ on cluster states. The advantage is smaller for random mixed factors beyond two qubits and decreases with system size for shallow circuits. Because different families use different representations, these are within-study comparisons of complete inference pipelines rather than one common scaling curve over unrestricted density matrices.

State-distance improvements are accompanied by distribution-level improvements
where an exact reference is available. On a specified two-qubit family with a
quadrature posterior, \method{} attains a Bures energy statistic of $0.00264$,
compared with $0.3570$ for \base{} and $0.2533$ for \flowgrpo{}, while its
diversity ratio is $1.009$ rather than $5.9$ and $3.2$. Wasserstein-1 and
Kolmogorov--Smirnov discrepancies give the same ordering. This comparison is
important because a low reconstruction error alone does not show that the
inference method represents posterior uncertainty correctly. Here the retained
samples reproduce both the location and dispersion of the reference posterior
rather than only approaching a favorable point estimate. Appendix~\ref{app:historicalreconstruction}
reports the complete construction and seeded results.

A dense four-qubit benchmark over 40 matched random mixed-state cells gives a
complementary test outside the structured physical families. There, \method{}
reduces mean Bures state distance from $0.3661$ for \base{} and $0.3696$ for
\flowgrpo{} to $0.3318$. Credible-coverage error decreases from $0.1000$ to
$0.0750$, and held-out observable error decreases from $0.0406$ and $0.0324$ to
$0.0286$. These historical comparisons evaluate the complete inference pipeline. A matched
nonlearned control reaches nearly the same dense four-qubit values, which already
suggests that part of the gain comes from the target-preserving posterior-sampling
stage rather than from the restricted scale scheduler alone. Because these runs
predate the corrected fixed-feature protocol, we use them to establish
pipeline-level performance and rely on the matched experiments below to isolate
learning. Full values and protocol details are reported in
Appendices~\ref{app:historicalreconstruction} and~\ref{app:protocol}.

\subsection{What produces the reconstruction gain?}
\label{sec:maintrack}

To separate prior information, posterior sampling, and learned adaptation, we compare four-qubit and six-qubit Dicke, cluster, cat, and random mixed-state problems under the same eight random-Pauli outcomes per record. Maximum likelihood uses the same physical family or factor representation as the samplers; a projected classical shadow and purified neural density operator provide family-independent comparisons; and the prior mean uses no outcomes. For these experiments, policies are retrained with feature scales fixed before training, and the \fixed{} shares the same initialization, proposal family, likelihood, and sampling budget as \method{} (Appendix~\ref{app:maintrackprotocol}).

Figure~\ref{fig:reconstructiondecomposition} shows that the prior already explains a substantial fraction of the low-shot reconstruction advantage. On four-qubit Dicke states, for example, the prior mean has Bures error $0.045$, compared with $0.101$ for maximum likelihood, about $1.07$ for the classical shadow and the RBM density operator, and $0.177$ for an amortized neural operator; posterior sampling gives $0.044$ for the matched control and $0.041$ for \method{}. Random mixed states show the same broader lesson: with only eight outcomes, prior averaging can already be competitive with posterior point estimates. The comparison therefore prevents us from attributing the entire improvement over the flow or non-Bayesian baselines to reinforcement learning. The amortized neural operator, a single network trained once to map a record to a state rather than the per-record fit of the RBM, is the strongest family-agnostic estimator we tested: it improves on both the classical shadow and the RBM density operator on every structured family at four and six qubits, and at six qubits it also passes family-aware maximum likelihood on those states (Table~\ref{tab:maintrack}). It nonetheless stays far above the prior mean and the Bayesian pipeline, confirming that a generic amortized reconstruction does not replace the physical prior and posterior sampling.

The incremental effect of the restricted learned scale scheduler is smaller and varies across seeds. This is an informative control rather than a negative result about learned proposals in general. In this configuration, learning selects among moves that were already designed to respect the physical prior, so a well-tuned fixed mixture can capture much of the same geometry. The experiment isolates what is gained by learning \emph{which existing move to use}; it does not test whether learning a new drift and covariance can create transitions unavailable to that library. Section~\ref{sec:multimodalmain} tests the latter question directly, while Appendix~\ref{app:maintrackprotocol} reports the complete seed-level decomposition.

A strong prior in the low-shot regime does not remove the need for posterior inference. The prior mean is a useful point estimator when measurements are sparse, but it contains no record-specific description of uncertainty and cannot reveal whether several physically distinct explanations remain plausible after observing the data. Posterior samples provide this additional information and support credible regions, held-out observable prediction, and downstream averages that depend on more than one representative state. The decomposition therefore has two purposes. It prevents us from crediting reinforcement learning for improvements already supplied by the physical prior, and it identifies the more demanding role of the sampler: representing how the measurement record reshapes that prior while preserving uncertainty. GRPO-QPS is evaluated against this stronger objective, which is why the distribution-level and sampling-efficiency comparisons below are necessary in addition to reconstruction error.

\begin{figure}[tbp]
\centering
\includegraphics[width=\textwidth]{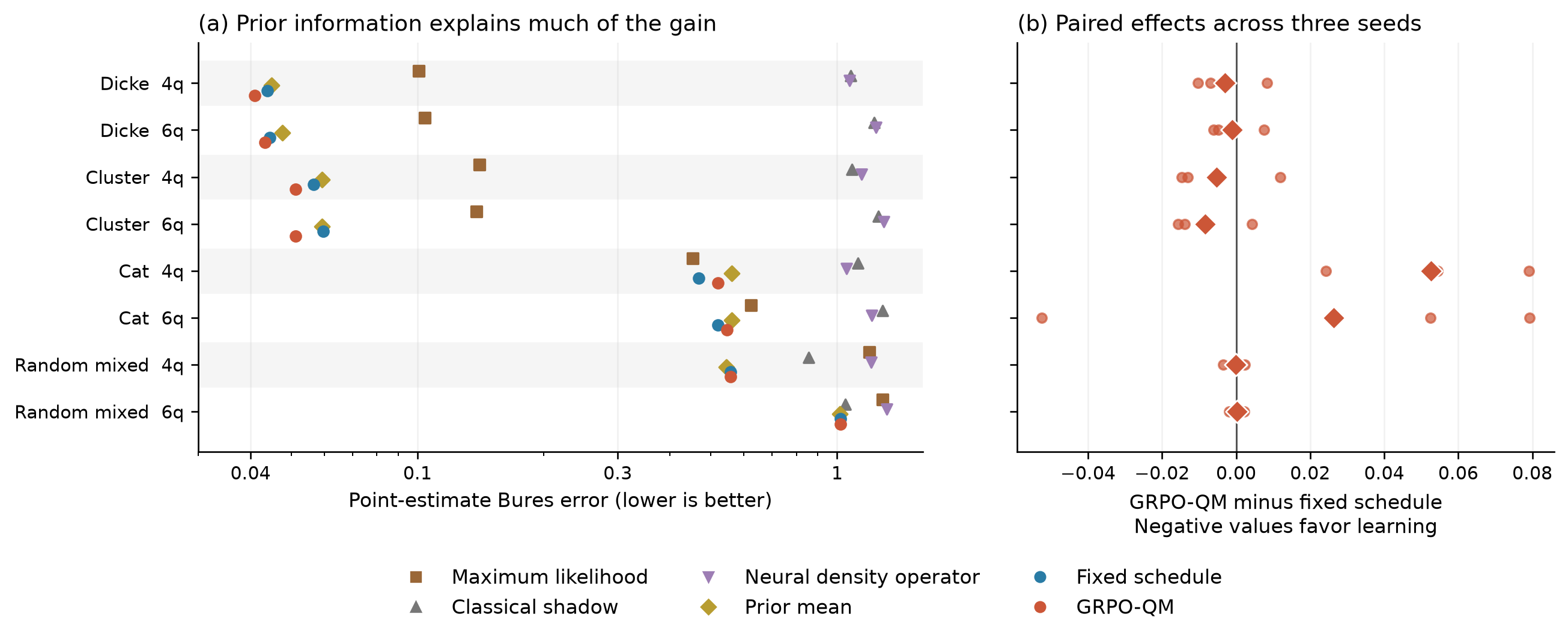}
\caption{Separating prior information, posterior sampling, and learned adaptation. Points show mean Bures error over three seeds and eight truths per seed. The right panel reports paired \method{} minus \fixed{} effects; negative values favor learning.}
\label{fig:reconstructiondecomposition}
\end{figure}

\subsection{Why tuned MCMC is competitive on the original benchmarks}
\label{sec:matched}

The gains over \base{} and \flowgrpo{} do not by themselves show that restricted learned proposal selection is better than a carefully tuned conventional kernel. We therefore compare fixed, error-tuned, acceptance-tuned, directly optimized, and GRPO-QPS mixtures that share the same proposal library on continuous Dicke, cluster, and cat posteriors. Collection uses matched likelihood work, and all learned or tuned choices are frozen before evaluation on new records (Appendix~\ref{app:continuous}).

On these benchmarks, learned scale selection does not consistently surpass tuned MCMC. After $1{,}024$ transitions from shared prior starts, \method{} has lower terminal coordinate Wasserstein-1 error than error-tuned MCMC in 18 repetition means, ties in 8, and is higher in 28; the exact sign test over the 46 discordant repetitions gives $p=0.18$. When intermediate states are pooled, \method{} improves over the untuned fixed mixture in all three family averages, but acceptance-tuned MCMC remains lower in each family. Its chain diagnostics are also stronger, with fewer coordinates exceeding $\widehat R=1.01$. Appendix~\ref{app:continuous} reports the complete paired analysis and uncertainty.

These results identify the limitation of this particular policy class. The available fixed moves already capture most of the useful structure of the original continuous posteriors, so learning only their scale or selection has limited room to improve. This observation also provides a direct testable hypothesis rather than an explanation made after seeing the multimodal results: if the MCMC advantage is caused by sufficient fixed proposal geometry, then the ordering should change when successful sampling requires transitions that the fixed library does not contain. The next experiment is designed specifically to test that hypothesis.

\subsection{GRPO-QPS surpasses conventional MCMC on a more challenging multimodal posterior}
\label{sec:multimodalmain}

We test this question on the thermal transverse-field Ising family with a widened field prior, $h\in[-1.8,1.8]$, and $Z$-basis records. The posterior forms two exact mirror modes at opposite signs of $h$. Local moves explore one mode but rarely cross between them, while global prior refresh becomes inefficient as the modes narrow. We evaluate the minimum effective sample size over $(J,h,\beta)$ and a mode indicator per $1{,}000$ likelihood calls, so a sampler must mix both within and between modes to score well. Every Metropolis proposal uses one new likelihood evaluation, and adaptive warm-up is charged to the corresponding baseline.

At four qubits and 200 shots, four of five training seeds outperform the strongest conventional sampler on the same record, with ESS ratios from $1.17\times$ to $2.40\times$; one seed collapses as its acceptance rate approaches zero. The separation becomes clearer at six qubits and 800 shots, where prior-refresh acceptance is approximately $3\%$. As shown in Table~\ref{tab:multimodal6q}, \method{} reaches a minimum ESS of 102, compared with 28 for prior independence, 27 for a tuned fixed mixture, and 20 for Haario adaptive Metropolis.

The symmetry-aware baseline is particularly informative. We give a reflection mixture the exact transformation $h\mapsto-h$, so it does not need to discover how to cross the two modes. This produces very high ESS for the mode indicator and for $h$, but its minimum ESS remains only 17 because movement in $J$ and $\beta$ is still slow. \method{}, in contrast, must learn the cross-mode transition from data, yet also learns a local covariance aligned with the narrow posterior region. Its advantage therefore cannot be explained by mode switching alone. Efficient sampling in this example requires both global transitions and local movement, and the learned proposal combines the two inside one Metropolis-corrected kernel.

\begin{table}[h]
\centering\small
\caption{Six-qubit multimodal thermal posterior with 800 $Z$ shots. Values are ESS per $1{,}000$ likelihood calls; the final column is the minimum across coordinates and the mode indicator.}
\label{tab:multimodal6q}
\begin{tabular}{lrrrrr}\toprule
Sampler & $J$ & $h$ & $\beta$ & Mode & Min.\\\midrule
Prior independence & 28 & 32 & 29 & 28 & 28\\
Fixed mixture, $90\%$ refresh & 27 & 33 & 28 & 31 & 27\\
Haario adaptive Metropolis & 20 & 20 & 21 & 23 & 20\\
Reflection mixture, known symmetry & 17 & 148 & 18 & 156 & 17\\
\method{} learned proposal & \textbf{102} & \textbf{135} & \textbf{119} & 116 & \textbf{102}\\
\bottomrule
\end{tabular}
\end{table}

We next train one record-conditioned policy per seed on five records and evaluate it, frozen, on three unseen 3,000-shot records. Training uses only statistics and reference moments derived from the five training records; no exact or held-out posterior information enters training or model selection. Across three independent training seeds, all nine held-out comparisons match or exceed the strongest conventional sampler, with mean ESS ratios of $1.49\times$, $1.47\times$, and $1.24\times$ (Table~\ref{tab:multimodaltransfer}). On held-out record 5, for example, the three learned policies obtain minimum ESS values of $44.8$, $40.8$, and $33.0$, compared with $27.6$, $28.2$, and $27.4$ for prior independence and approximately $16.5$ for Haario. None of the three transfer runs collapses. The relevant result is therefore not only that one trained kernel fits one difficult posterior, but that record-conditioned proposal behavior learned from the training records remains useful when the measurement record changes. This transfer is also what can amortize training cost: under the present protocol, training uses about $25.6$ million likelihood calls, and the saved evaluation work reaches the same order after roughly fifty served records. Appendix~\ref{app:multimodal} gives the complete accounting.

The comparison is designed so that this improvement cannot be obtained by changing the inference target. Every learned proposal is evaluated with the same Born likelihood as the conventional samplers, its complete forward and reverse density enters the Metropolis ratio, and performance is normalized by likelihood evaluations rather than by raw transition count. The learned policy may therefore decide where and how far to propose, but proposals that would distort posterior probability are still corrected by the same acceptance principle used throughout the paper. This distinction is central to the result. GRPO-QPS is not rewarded for producing samples that merely resemble a preferred reconstruction, and the reported ESS advantage is not purchased by relaxing Bayesian correctness. Within the tested multimodal family, the gain instead reflects a more efficient transition kernel: the policy learns record-dependent long-range motion between separated modes together with local covariance that follows the narrow posterior geometry.

This also clarifies the relationship between the original and multimodal benchmarks. The original experiments ask whether learning can improve the selection of an already adequate set of physical moves, and the answer is modest. The multimodal experiments ask whether learning can construct proposal geometry that the fixed library does not already provide, and here the answer is clearly positive. The two results are therefore complementary rather than conflicting. They locate the value of learning in the structure of the posterior and suggest a practical criterion for using GRPO-QPS: additional training effort is most justified when standard kernels show complementary failures, such as good local exploration but poor mode crossing, or successful global jumps but inefficient movement within narrow modes.

\begin{table}[h]
\centering\small
\caption{Transfer to unseen six-qubit, 3,000-shot records. Entries are the ratio of \method{} minimum ESS to the strongest conventional sampler at equal likelihood cost.}
\label{tab:multimodaltransfer}
\begin{tabular}{lrrrr}\toprule
Training seed & Record 5 & Record 6 & Record 7 & Mean\\\midrule
0 & 1.62 & 1.12 & 1.71 & 1.49\\
1 & 1.45 & 1.03 & 1.93 & 1.47\\
2 & 1.20 & 1.05 & 1.47 & 1.24\\
\bottomrule
\end{tabular}
\end{table}

The original and multimodal comparisons therefore answer different questions. Learning only the scale or selection of a sufficient move library provides little benefit, whereas learning new proposal geometry can matter when the posterior requires both long-range mode changes and efficient local movement. In the multimodal study, the policy learns state-dependent drift and covariance and is trained with an estimation-aligned trajectory objective; training on a mixture of broader and narrower posteriors also prevents the collapse observed when optimization begins directly on the narrowest modes. Appendix~\ref{app:multimodal} provides the complete construction, cost accounting, and scope.

\subsection{What the training diagnostics reveal}
\label{sec:rewarddesign}

The earlier restricted scheduler also reveals why a learned sampler can fail even when the target is preserved. On an exactly enumerated 48-state posterior, learning produces a small increase in spectral gap but no resolved improvement in effective samples per unit transition work, and the change in estimation variance depends on the physical observable. A closed-form four-state example explains the mismatch: larger accepted Bures movement can coexist with stronger autocorrelation in the quantity being estimated. Movement is therefore not, by itself, a reliable objective for physical estimation. Appendix~\ref{app:finite} and Appendix~\ref{sec:movementcounterexample} provide the exact transition and covariance analyses.

When reward ambiguity is removed and exact and sampled optimization minimize the same finite-trajectory variance objective, the policy class contains a measurable learning opportunity. At 1,024 evaluation states, exact differentiation reduces mean physical estimation variance by $12.63\%$ relative to the tuned mixture, while sampled trajectory-score training recovers a $6.32\%$ reduction across 45 enumerated posteriors and three training repetitions. An inconsistent rescaling of the trajectory score and regularization penalty reduces the gain to $0.43\%$, while scaling both terms restores a $7.28\%$ improvement. The effect is not uniform: sampled training improves the four-shot setting but degrades the sixteen-shot setting, whereas the exact gradient improves both. These diagnostics therefore identify a real opportunity inside the policy class while also showing that optimization details determine whether that opportunity is realized. The full reward-ranking study, optimizer interventions, uncertainty analysis, and complete numerical tables are moved to Appendix~\ref{app:rewarddesign}.

Together, these diagnostics motivate the final multimodal configuration used above: a proposal class expressive enough to change both drift and covariance, an objective tied to finite-trajectory estimation rather than movement alone, and a training curriculum spanning posterior widths. The main performance claim therefore rests on the reconstruction and multimodal evaluations, while the enumerated studies explain which design choices make target-preserving learning useful and which simpler choices can hide that advantage.

\section{Conclusions and limitations}
\method{} addresses a central difficulty in applying reward-based learning to Bayesian scientific inference: exploration can be optimized while preserving the target posterior defined by the prior and physical likelihood. Across quantum-tomography problems, GRPO-QPS improves reconstruction and posterior agreement over \base{} and \flowgrpo{} by learning proposal strategies inside a Metropolis-corrected chain, such adaptation changes how the posterior is explored without changing the posterior itself. The matched MCMC comparisons further clarify when the learned component matters. Conventional samplers remain strong when a fixed proposal library already matches the posterior geometry, but on the more challenging multimodal thermal posterior, GRPO-QPS learns proposal structure that improves both local exploration and movement between modes, surpasses the tested conventional MCMC baselines, and transfers to unseen records. Taken together, these results support GRPO-QPS as a target-preserving learned inference framework whose advantage appears when posterior geometry requires adaptation beyond what fixed proposals already provide.

These advantages have clear limits. On a smooth 512-dimensional unimodal factor posterior, tuned HMC remains superior, and the learned policy incurs a training cost justified primarily by reuse across many records. We do not claim GRPO-QPS dominates every MCMC algorithm; rather, the results establish a more specific conclusion within the experimental scope of this paper: target-preserving proposal learning can provide a substantial sampling advantage when conventional kernels do not adequately capture the structure of a challenging posterior.

\section*{AI use statement}

A large language model assisted with language polishing. The authors verified all claims, proofs, mathematical formulations, and reported values against the implementation and experimental results and take full responsibility for the manuscript.

\section*{Ethics statement}

This work uses simulated quantum systems and does not involve human subjects,
personal data, or autonomous physical control.  The main ethical risk is
misleading scientific attribution.  We address it through matched comparisons,
family-specific results, and a clear separation between development-set
evidence and future confirmation.

\section*{Reproducibility statement}

Appendix~\ref{app:reproducibility} describes the numerical records, verification
checks, and remaining release work. The corrected experiments preserve
configurations, logs, checkpoints, samples, source snapshots, and content hashes;
the appendix identifies gaps in the historical records. The supplied ICLR
2027 style files are unchanged.

\bibliography{references}
\bibliographystyle{iclr2027_conference}

\appendix

\section{Mathematical validity of target-preserving learning}
\label{app:math}

The central mathematical claim of this paper is that learning changes how
proposals are selected without altering the stationary distribution of the
frozen chain. This section states the required conditions, proves stationarity,
and then clarifies why the GRPO path group must not be confused with the spectral
decomposition of a density matrix.

\subsection{Posterior invariance of the learned proposal schedule}

Let $x$ be a physical state parameter, $a$ a discrete kernel action, and
$K_a(x'\mid x)$ the action-conditional proposal.  The joint forward proposal
is $Q_\phi(x',a\mid x)=q_\phi(a\mid x)K_a(x'\mid x)$.  The complete acceptance
probability is
\begin{equation}
  \alpha(x,x',a)=\min\left\{1,
    \frac{\pi_X(x'\mid y)q_\phi(a\mid x')K_a(x\mid x')}
  {\pi_X(x \mid y)q_\phi(a\mid x )K_a(x'\mid x)}\right\}.
  \label{eq:general_mh}
\end{equation}

\begin{proposition}[Posterior stationarity]
  Assume that $q_\phi(a\mid x)>0$ for every enabled action.  Assume also that each
  $K_a(x'\mid x)$ has a defined reverse density on its support and that $\phi$ is
  frozen during sample collection.  The transition defined by
  Eq.~\ref{eq:general_mh} is reversible with respect to $\Pi_X(dx\mid y)$.
  Consequently, $\Pi_X(dx\mid y)$ is stationary for the resulting mixture kernel.
\end{proposition}

\begin{proof}
  For a fixed action, the accepted joint probability flow from $x$ to $x'$ is
  \begin{align}
    &\pi_X(x\mid y)q_\phi(a\mid x)K_a(x'\mid x)\alpha(x,x',a) \\
    &\quad=\min\!\left\{
      \pi_X(x\mid y)q_\phi(a\mid x)K_a(x'\mid x),
      \pi_X(x'\mid y)q_\phi(a\mid x')K_a(x\mid x')
    \right\}.
  \end{align}
  The right side is unchanged when $x$ and $x'$ are exchanged.  Summing this
  identity over actions gives detailed balance for the accepted transitions.
  The rejection probability supplies the remaining diagonal mass, so it also
  preserves $\Pi_X$.  Stationarity follows from detailed balance.
\end{proof}

Equation~\ref{eq:general_mh} also covers the continuous Gaussian-mixture proposal in
Eq.~\ref{eq:learnedmixture} by treating the mixture itself as $Q_\phi$ rather
than introducing a discrete action. In logit coordinates, the target density
includes the change-of-variables factor and the full forward and reverse mixture
densities are evaluated explicitly. The same detailed-balance argument therefore
applies to the frozen drift-and-covariance policy.

The policy must depend only on the individual state and record, with fixed
feature normalization; a feature computed from other chains being sampled at the same time would require a
joint-chain argument. The frozen policy condition is substantive: if the policy
changed while samples were being retained, the chain would become
time-inhomogeneous and the proposition would no longer apply directly. We
therefore train first, freeze $q_\phi$, discard burn-in, and only then collect
posterior samples. The proposition does not establish irreducibility, geometric
ergodicity, a convergence rate, or adequate burn-in; those finite chain
properties require separate analysis.

\begin{proposition}[Convergence on the enumerated graph]
  Let $\mathcal S$ be a finite-state space on which $\pi_X(x\mid y)>0$.  Suppose
  one enabled action is a global refresh with $K_{\mathrm r}(x'\mid x)>0$ for all
  $x,x'\in\mathcal S$, and suppose its frozen policy probability is positive at
  every state.  If the corrected transition also has a positive self loop at
  least one state, then the GRPO-QPS chain is irreducible and aperiodic.  Its unique
  stationary distribution is $\pi_X$, and the distribution of the chain converges
  to $\pi_X$ from every initialization.
\end{proposition}

\begin{proof}
  For any $x,x'\in\mathcal S$, the refresh proposal probability is positive.
  Every factor in the Metropolis ratio is also positive on the target support, so
  the accepted refresh transition from $x$ to $x'$ has positive probability.
  Thus every state communicates with every other state and the finite chain is
  irreducible.  A positive self loop at one state makes an irreducible chain
  aperiodic.  Proposition~1 supplies stationarity.  The standard finite-state
  Markov chain convergence theorem then gives uniqueness of $\pi_X$ and
  convergence from every initial distribution.
\end{proof}

This proposition applies to the audited 48-state exact graph. Its
positive policy hypothesis holds by construction: the bounded $1.5\tanh(\cdot)$
residual keeps every action logit within $1.5$ of its predefined value, ensuring
$q_\phi(a\mid x)\ge e^{-3}\,q_{\mathrm{fixed}}(a)$ for every action. Because
the predefined refresh probability is positive, the refresh action has a uniform
positive lower bound at every state. The proposition does not establish a useful
convergence rate, so we instead compute the full transition matrix and report the
absolute spectral gap
\begin{equation}
  \gamma_{\mathrm{abs}}(T)
  =1-\max_{j\geq2}|\lambda_j(T)|.
  \label{eq:absolutegap}
\end{equation}
Here $\lambda_1=1$ is the stationary eigenvalue.  The remaining eigenvalues are
ordered by magnitude.  Every evaluated transition is reversible, so the
eigenvalues are real after the similarity transform
$D_\pi^{1/2}TD_\pi^{-1/2}$.  We use no lazification.  For numerical stability,
states below $10^{-14}$ times the maximum posterior mass are excluded only from
the eigensolve, and their outgoing mass is returned to the diagonal.  The
largest excluded mass over 360 evaluation cells is $3.19\times10^{-14}$.  The
largest detailed balance and stationarity residuals are
$1.39\times10^{-17}$ and $4.44\times10^{-16}$, respectively.  The mean absolute
spectral gap across the 360 evaluation cells is approximately $0.31$. This mean
does not bound the slowest cell or establish that $1{,}024$ burn-in transitions
suffice from every initialization.

Equation~\ref{eq:absolutegap} does not extend automatically to continuous factor states,
where minorization or drift conditions would be needed.

\subsection{Physical factor states and prior reversible moves}

For a complex factor $F\in\mathbb C^{d\times r}$, the map
\begin{equation}
  \rho(F)=\frac{FF^\dagger}{\tr(FF^\dagger)}
  \label{eq:factorphysical}
\end{equation}
is positive semidefinite and has unit trace whenever $F\neq0$.  The factor
target is
\begin{equation}
  \Pi_F(dF\mid y)\propto L[y\mid\rho(F)]\,\mu_{\mathrm{CN}}(dF),
  \label{eq:factorposterior}
\end{equation}
where $\mu_{\mathrm{CN}}$ is the complex Gaussian factor measure.  The density
matrix posterior is the pushforward of Eq.~\ref{eq:factorposterior}; it is not
silently identified with a separately chosen density matrix prior.  The proposal in
Eq.~\ref{eq:pcn} preserves the complex Gaussian factor prior because the pair
$(F,F')$ is exchangeable under that prior.  Its prior and kernel terms therefore
cancel from Eq.~\ref{eq:general_mh}, leaving only the likelihood and learned action probabilities in the
specialized scale-scheduler ratio.  Bounded thermal coordinates instead use
reflected symmetric local moves or an exact prior refresh, so their acceptance
rule retains the corresponding prior ratio explicitly.

Physicality and posterior correctness are distinct properties.
Equation~\ref{eq:factorphysical} prevents negative eigenvalues and trace drift,
while the acceptance calculation determines whether the frequency of physical
states is consistent with the prior and Born likelihood. Neither property alone
implies accurate finite chain exploration.

\subsection{Credit for accepted and rejected transitions}

Let $z\in\{0,1\}$ record whether a proposal was accepted.  Conditional on the
recorded proposal randomness, the policy dependent event probability is
\begin{equation}
  p_\phi(a,z\mid x,x')=q_\phi(a\mid x)
  \alpha_\phi(x,x',a)^z[1-\alpha_\phi(x,x',a)]^{1-z}.
  \label{eq:eventprobability}
\end{equation}
The ratio used in Eq.~\ref{eq:grpo} is
$r_{g,t}=p_\phi/p_{\phi_{\mathrm{old}}}$. The old event probability is saved
before two clipped policy epochs. During replay, the proposal draw and $z$ are
held fixed, likelihood values are treated as environment values, and gradients
pass through $q_\phi$ at both $x$ and $x'$ inside $\alpha_\phi$. Rejections
therefore receive their actual event probability rather than credit for an
unobserved move, and the terminal group advantage is shared across the four
transitions. This defines the implementation exactly but constitutes an
approximate credit rule; we do not claim an unbiased estimator for a long-horizon
mixing objective.

\paragraph{Training and collection procedure.}
For each seeded posterior set used during method development, the implementation follows five steps.
First, it initializes the policy at the predefined fixed probabilities and
starts $G$ matched paths from each flow-initialized state.  Second, it draws actions,
proposal noise, and accept or reject decisions under the old policy for four
transitions.  Third, it computes one terminal reward per path and normalizes
rewards only within paths sharing the same start and record.  Fourth, it stores
the old event probabilities from Eq.~\ref{eq:eventprobability} and performs two
clipped Adam epochs while replaying the recorded events.  Fifth, after the
predefined number of updates, it freezes the final policy, restarts from the
saved flow-initialized states, discards burn-in, and retains the terminal collection.
There is no checkpoint selection on evaluation outcomes.

\subsection{Group replicas, quantum mixtures, and the Bayesian target}
\label{sec:groups}

Three distinct probability structures enter the method, each serving a different
purpose that must not be conflated.

Within a single state, the eigenvalues weight the physical mixture
$\rho=\sum_i\lambda_i|\psi_i\rangle\langle\psi_i|$ and enter every Born
probability with their physical weight,
\begin{equation}
  p_m(\rho)=\tr(E_m\rho)=\sum_i\lambda_i\langle\psi_i|E_m|\psi_i\rangle .
  \label{eq:bornweights}
\end{equation}
Across states, the prior and Born likelihood weight the Bayesian target
$\Pi_X(\cdot\mid y)$ of Eq.~\ref{eq:posterior}. The GRPO group is neither of
these: it is simply a set of $G$ stochastic exploration replicas that share a
record and starting state. Their relative rewards supply a variance-reducing
baseline; some tested variants also divide by the group's reward standard
deviation. That optional normalization is itself evaluated in
Appendix~\ref{app:rewarddesign}. A naive adaptation that
treated the $r$ eigenvectors as equally weighted group members would replace
$\rho$ with $\tfrac1r\sum_i|\psi_i\rangle\langle\psi_i|$ and predict
$\tfrac1r\sum_i\langle\psi_i|E_m|\psi_i\rangle$ instead of
Eq.~\ref{eq:bornweights}, a different quantum state with different measurement
statistics. In \method{}, the factor $1/G$ enters only the relative credit
computation; every replica evaluates the full $\lambda_i$-weighted density
matrix, so neither the eigenvalue weights nor the posterior weights are ever
overwritten.

State-dependent selection also enters the acceptance probability and the
accepted or rejected event used for learning. Appendix~\ref{app:math} gives the
event probability and explains why an approximate training rule can change the
schedule while the correctly constructed frozen kernel retains its target.

\subsection{Energy score used by Flow-GRPO}

Equation~\ref{eq:energyreward} gives the per endpoint contribution of a sample
energy statistic. Its cross term attracts endpoints toward a finite posterior
comparator, while its within-group term rewards diversity. The second term uses
$G-1$ rather than $G$ because the self distance is zero and would otherwise
dilute the scale. The reward remains a finite-sample estimate: it neither
guarantees exact target preservation nor repairs a region that all stochastic
flow paths miss. This distinction explains why \flowgrpo{} serves as a
meaningful comparison but not a substitute for a
sampler that preserves its stationary distribution.

\section{Historical reconstruction comparisons}
\label{app:historicalreconstruction}

These experiments document the original two-to-ten-qubit comparison with the
flow methods. Their numerical outputs are unchanged. The earlier policy features
used statistics computed across chains in the same sampling batch, so freezing the policy weights did not
make each chain an independent frozen kernel. The scalar two-qubit and dense
40-cell four-qubit studies have not been rerun with fixed individual-state
features. Their terminal errors are descriptive results, not an application of
Proposition~1 or a certification of posterior convergence. The corrected
four-qubit and six-qubit comparison is reported separately in
Appendix~\ref{app:maintrackprotocol}.

The reconstruction gains depend strongly on the physical family
(Figure~\ref{fig:landscape}). On paired truths and measurement records,
\method{} lowers Bures state distance for thermal, cat, Dicke, and cluster states
at every tested size from two to ten qubits. The reduction reaches approximately
$94\%$ for Dicke states, exceeds $89\%$ for cluster states, and ranges from
$56\%$ to $67\%$ for cat states. In contrast, \flowgrpo{} generally remains
close to \base{}. The complete physical sampling pipeline thus reaches more
accurate reconstructions than reward tuning of the flow alone. Each comparison
uses the same representation, but representations differ across studies, so
these results do not form a common dense state scaling curve
(Appendix~\ref{app:protocol}).

This advantage has two important boundaries. Random mixed factors change little
beyond two qubits, while the shallow circuit gain falls from approximately
$48\%$ at two qubits to $2\%$ at ten qubits. Moreover, turning learning off
retains much of the improvement. The \fixed{} follows \method{} closely across
the complete system matrix and matches it within seed variation on the dense
four-qubit benchmark (Tables~\ref{tab:allsystems} and~\ref{tab:q4}). The broad
gain therefore belongs to the complete pipeline, not to learned selection alone.

Known family information helps explain why that distinction matters. Even
before posterior sampling, draws from the family prior have much lower Bures
error than either flow on four-qubit Dicke and cluster states
(Table~\ref{tab:starts}). Part of the reconstruction advantage is therefore
already present in the chosen state family. Section~\ref{sec:matched} holds
that information fixed to ask what learning adds to posterior exploration.

\subsection{An exact two-qubit posterior the flow methods miss}

Across eight seeded posterior sets used during method development, the exact
scalar posterior allows a direct comparison between terminal samples and
numerical quadrature on a declared one-dimensional family (Table~\ref{tab:q2main}). The Bures energy statistic for
\method{} is two orders of magnitude smaller than for either flow, and its
diversity ratio sits near one rather than the factor of several seen for the flows,
so \method{} recovers the reference spread without collapsing around a single
convenient reconstruction. The direct Wasserstein-1 and Kolmogorov--Smirnov errors over 768 records show
the same ordering. The \fixed{} reaches nearly the same agreement, indicating
that this simple one-dimensional posterior does not require the restricted learned
scale scheduler.
This establishes terminal sample agreement on the declared
family, not convergence from every initialization or chain length.

\begin{table}[tbp]
  \caption{Exact two-qubit posterior agreement over eight seeded posterior sets.  Values
    are mean $\pm$ standard deviation across set means; lower is better except
    that the ideal diversity ratio is one.  The energy statistic measures
    agreement with quadrature samples, and the Wasserstein-1 and KS columns are
    direct one-dimensional agreement metrics on the declared scalar family.
    Appendix~\ref{app:allresults} adds the truth distance, observable, calibrated
  total variation, and quantile metrics.}
  \label{tab:q2main}
  \centering
  \small
  \setlength{\tabcolsep}{5pt}
  \begin{tabular}{lrrrr}
    \toprule
    Method & Energy statistic & Diversity ratio & Wasserstein-1 & KS \\
    \midrule
    BuresTomFlow & $0.3570\pm0.0234$ & $5.908\pm0.301$ & $0.19001\pm0.01833$ & $0.65899\pm0.03678$ \\
    Flow-GRPO & $0.2533\pm0.0117$ & $3.186\pm0.172$ & $0.06358\pm0.00539$ & $0.46544\pm0.02111$ \\
    GRPO-QPS & $0.00264\pm0.00041$ & $1.009\pm0.021$ & $0.00991\pm0.00068$ & $0.10362\pm0.00330$ \\
    \bottomrule
  \end{tabular}
\end{table}

\subsection{Dense four-qubit reconstruction}

We next test whether the improvement extends from structured families to a
denser, less structured benchmark. The dense four-qubit study pairs 40 random mixed cells,
with variation reported as the standard deviation across cells rather than a
confidence interval (Table~\ref{tab:q4}). \method{} lowers the mean state distance
below both flows. State distance measures reconstruction accuracy, coverage error
reflects uncertainty width, and observable error tests predictions not used by the
primary score. Improvement across all three metrics indicates that the result is not limited
to the primary state-distance measure. The
\fixed{} reaches
essentially the same three values, locating the source of the gain in the target
preserving proposals rather than in the learned schedule.

\begin{table}[tbp]
  \caption{Dense four-qubit random mixed benchmark over 40 paired cells (lower is
    better).  Values are mean $\pm$ standard deviation across cells.  \method{}
    improves on both flow methods across state distance, credible coverage error,
    and observable reconstruction error, and the matched control reaches the same
  three values.}
  \label{tab:q4}
  \centering
  \small
  \begin{tabular}{lrrr}
    \toprule
    Method & State distance & Coverage error & Observable error \\
    \midrule
    BuresTomFlow & $0.3661\pm0.0059$ & $0.1000\pm0.0000$ & $0.0406\pm0.0041$ \\
    Flow-GRPO & $0.3696\pm0.0061$ & $0.1000\pm0.0000$ & $0.0324\pm0.0020$ \\
    GRPO-QPS & $0.3318\pm0.0057$ & $0.0750\pm0.0155$ & $0.0286\pm0.0010$ \\
    Matched control & $0.3317\pm0.0057$ & $0.0717\pm0.0181$ & $0.0285\pm0.0010$ \\
    \bottomrule
  \end{tabular}
\end{table}

The constant \base{} coverage error has a concrete meaning: its credible balls
include all 160 held-out truths in every cell, giving observed coverage $1.0$ rather
than the nominal $0.9$, which both \method{} and the \fixed{} reduce while improving
state and prediction errors. The complete 40-cell pairing prevents the result from being determined by a
small subset of records. Because these cases were available during method
development, we treat the comparison as development-set evidence rather than as
a confirmatory evaluation.

\subsection{Thermal reconstruction across system sizes}

Thermal Ising states yield lower \method{} state distance than both flow methods
at every tested size (Figure~\ref{fig:thermal-lines}). The eight-qubit and ten-qubit results use compressed thermofield factors, with
two seeds per condition and three at five qubits. Because these records were
available during method development, the comparisons are descriptive rather than
confirmatory. The paired
\method{} minus control differences (Figure~\ref{fig:thermal-paired}) favor
\method{} at five sizes and the \fixed{} at five qubits but are wide enough to
overlap zero everywhere, consistent with the matched-sampling finding below that
the restricted scale scheduler adds little on these original conditions.
Different posterior widths along the temperature, field, and coupling directions
could nevertheless favor richer state-dependent proposal geometry. The present experiments
do not include an intervention that isolates this mechanism, so we do not treat
it as an established explanation.

\section{Original reconstruction protocol}
\label{app:protocol}

\begin{table}[htbp]
\centering\small
\caption{Distinct GRPO-QPS implementations answer different questions. Freezing learned weights is necessary but does not remove dependence on other chains sampled at the same time.}
\label{tab:methodversions}
\begin{tabular}{p{0.21\linewidth}p{0.34\linewidth}p{0.35\linewidth}}
\toprule
Study & Policy and training & Evidence supplied\\
\midrule
Reconstruction & Four-step reward; batch normalized inputs; one policy per posterior set & Terminal errors, with historical kernel assumptions unresolved except for the 24 main-track cells, rerun with fixed feature scales (Appendix~\ref{app:maintrackprotocol})\\
Matched continuous & One-step reward; fixed individual state features; frozen before new records & Comparison with equally informed, tuned samplers\\
Enumerated posterior & Sixteen-step reward; exact posterior features; eight training repeats & Exact transition rates on 24 unseen evaluation sets\\
Matched objective diagnostic & 256 retained states; exact posterior features and means; three training seeds & Exact versus sampled optimization on 45 enumerated development posteriors; no transfer claim\\
\bottomrule
\end{tabular}
\end{table}

The experiments are designed to answer distinct questions at increasing physical
and computational scale. Treating all cells as interchangeable would confound
exact posterior recovery with dense reconstruction, factor representations, and
thermofield compression. We therefore interpret each study within its stated
role, as summarized in Table~\ref{tab:evidenceladder}.

\begin{table}[h]
  \caption{Evidence ladder and the conclusion permitted by each study.  A cell
    is one matched truth, measurement record, representation, and evaluation
  condition.}
  \label{tab:evidenceladder}
  \centering
  \scriptsize
  \begin{tabular}{L{0.18\textwidth}L{0.12\textwidth}L{0.25\textwidth}L{0.29\textwidth}}
    \toprule
    Study & Independent cells & Scientific role & Permitted conclusion \\
    \midrule
    One-dimensional quadrature posterior & 8 seeded posterior sets & Compare generated scalar coordinates with numerical quadrature & Establish empirical terminal sample agreement on the declared family \\
    Enumerated 48-state transitions & 24 evaluation sets with 15 posteriors each & Calculate the complete corrected transition matrix in closed form & Verify exact stationarity and convergence of the frozen corrected chain \\
    Dense four qubits & 40 & Reproduce the full density matrix benchmark under paired records & Compare reconstruction, coverage, and observable reconstruction for the dense architecture \\
    Common physical families & 10 per qubit count & Test coherence, conserved sectors, stabilizers, and unfamiliar entanglement & Identify family-specific gains and limitations within a fixed representation \\
    Thermal Ising systems & 4, 4, 12, 8, 4, and 4 from two to ten qubits & Test mixed-state directions across dense and compressed descriptions & Measure the posterior correction gain within each representation \\
    \bottomrule
  \end{tabular}
\end{table}

\subsection{Study design and implementation overview}
\label{app:designoverview}

\textbf{Evaluation logic.}
An apparent gain over a learned flow can arise from three distinct sources:
known family information restricts the state space, physical proposals explore
it, and learning changes which proposals are selected. We disentangle these
through three questions: whether terminal samples agree with an exact posterior
that the initial flow misses, how large the reconstruction gain over both flow
baselines is across families and system sizes, and how much of that gain survives
when the learned schedule is replaced by the \fixed{}. Each study uses the
strongest available reference at its scale, ranging from exact quadrature at two
qubits to a finite importance comparator, rather than a converged
posterior, at larger sizes.

\textbf{Systems and representations.}
We study $n\in\{2,4,5,6,8,10\}$ qubits across six families: random mixed,
thermal transverse field Ising, cat, Dicke, cluster, and shallow circuit
states, chosen to probe correlated mixing, coherence, excitation sectors,
stabilizer structure, and unfamiliar entanglement patterns. An additional
two-qubit study covers Bell, cat, and Werner states with an exact posterior.
Dense matrices are retained through six qubits where practical and normalized
purification factors are used otherwise, so methods are always compared within
the same representation rather than along a single scaling curve.
Appendix~\ref{app:protocol} lists every prior, representation, and comparator.

\textbf{Pairing and measures.}
Within each study, methods are paired on the same truth and measurement record.
The broad flow comparisons also vary the use of known family information,
whereas the matched sampler comparison fixes the prior, likelihood, and initial
coordinates. We average within each seeded posterior set or training repetition before
comparing methods, since posterior draws from a single chain are not independent
experimental replicates. The reconstruction studies use the Bures state distance
as their primary metric; the matched continuous study instead evaluates
posterior coordinate Wasserstein and CDF errors against numerical integration.
Supporting measures include the held-out log likelihood, coverage, physical observables,
and physicality checks. Reported standard deviations reflect variation across
development-set cells, not confirmation on untouched data.

\textbf{Implementation.}
The original reconstruction studies use matched groups of four to sixteen paths,
four proposal steps, ratio clipping in $[0.8,1.2]$, one policy per seeded posterior set,
and one terminal draw per chain, measuring terminal behavior but not
autocorrelation. Every policy is frozen before collection;
Appendix~\ref{app:protocol} and Appendix~\ref{app:reproducibility} give the full
protocol, saved artifacts, and verification record. A separate comparison freezes
one policy per family and size before twelve new records and retains full
trajectories, using a one-step reward shared with direct optimization rather
than the original four-step reward (Appendix~\ref{app:continuous}). The
finite-state study instead uses a sixteen-step reward with enumerated posterior
features. Table~\ref{tab:methodversions} maps these three implementations to the
questions each can answer.

\subsection{Systems, representations, priors, and comparators}

The common matrix contains random mixed, thermal Ising, cat, Dicke, cluster,
and shallow circuit states, each probing a different source of posterior
difficulty. Cat states demand global coherence and parity, Dicke states impose an
excitation sector constraint, cluster states require stabilizer order, thermal
states combine several correlated mixed-state directions, and shallow circuits
test entanglement that is not aligned with any explicit symmetry.

Dense density matrices are retained through six qubits where practical; the
larger common family studies use the normalized factor of
Eq.~\ref{eq:factorphysical}, and the eight- and ten-qubit thermal studies use
compressed thermofield factors. These representations share physicality but not
identical approximation error, so all comparisons are made within a
representation. We do not interpret the combined table as a dense scaling curve.

Table~\ref{tab:priors} makes the target measure explicit.  The exact two-qubit
study uses three state maps with a uniform scalar prior on $[0.02,0.98]$.  The
original sampler used a 201-point quadrature, while the direct metric and
coverage reanalysis uses 2,001 points and checks stability at 4,001 points.  The
dense random mixed study draws a complex Gaussian
matrix for a Haar eigenbasis and independent unnormalized eigenvalues uniformly
on $[0,1]$, followed by trace normalization.  Factor random mixed studies use a
rank-16 complex Gaussian factor prior.  The remaining families use uniform
coordinate priors on the listed boxes.

\begin{table}[h]
  \caption{Declared parameter space priors.  Intervals are inclusive and all
  coordinates are independent before the physical state map is applied.}
  \label{tab:priors}
  \centering
  \scriptsize
  \begin{tabular}{L{0.19\textwidth}L{0.25\textwidth}L{0.43\textwidth}}
    \toprule
    Study & Parameter & Prior measure \\
    \midrule
    Exact two-qubit families & Scalar family coordinate $\theta$ & Uniform on $[0.02,0.98]$ \\
    Dense random mixed & Eigenbasis and unnormalized spectrum & Haar basis from complex Gaussian QR; independent uniform weights on $[0,1]$ \\
    Factor random mixed & Complex factor $F\in\mathbb C^{2^n\times16}$ & Independent standard complex Gaussian entries, then normalized through $\rho(F)$ \\
    Cat & Phase and dephasing & Uniform on $[0,2\pi]\times[0.02,0.20]$ \\
    Dicke & Noise weight & Uniform on $[0.01,0.10]$ \\
    Cluster & Noise weight & Uniform on $[0.01,0.13]$ \\
    Shallow circuit & $3n$ gate angles and noise & Uniform on $[0,2\pi]^{3n}\times[0.02,0.10]$ \\
    Thermal Ising & Coupling, field, inverse temperature & Uniform on $[0.5,1.5]\times[0.2,1.8]\times[0.2,1.2]$ \\
    \bottomrule
  \end{tabular}
\end{table}

The inference dimension varies sharply across families, and interpreting the
state distances in Table~\ref{tab:allsystems} depends on it. The structured
families are low dimensional: Dicke, and cluster states have a single noise weight
parameter, cat states have a phase and a dephasing strength, and thermal states
have a coupling, field, and inverse temperature. Their posteriors therefore lie
in a one- to three dimensional box around a known structured state, so a small
state distance, such as the ten-qubit Dicke and cluster errors near $0.06$ to
$0.08$, can reflect the restrictive prior before any coordinate is actually
learned from measurements. The initialization comparison in
Table~\ref{tab:starts} measures this effect directly. The random mixed and
shallow circuit families are high dimensional (a rank-$16$ complex factor and
$3n$ gate angles with noise); their larger errors do not isolate dimensionality
from prior, representation, and training effects. In the factor and thermofield
measurement generator, each listed shot is one random product Pauli setting and
one outcome, so an eight shot record contains eight outcomes in total, not eight
repetitions at every setting. Dense count based studies instead repeat each
selected setting. If $S$ is the shot label and $K$ the number of count settings,
their record totals are given below. The scalar two-qubit generator likewise
draws $S$ outcomes at each of its declared Pauli settings.

\begin{table}[htbp]
\centering\small
\caption{Measurement budgets are representation specific. A shot is one outcome from one state preparation. $S$ denotes the reported shot label; $K$ counts the settings at which outcomes are accumulated.}
\label{tab:shotsemantics}
\begin{tabular}{lll}
\toprule
Record design & Sampling rule & Total outcomes\\
\midrule
Factor and thermofield & One outcome at each random setting & $S$\\
Dense snapshot & One outcome at each random setting & $S$\\
Dense count and scalar Pauli & $S$ repetitions per count setting & $KS$\\
Dense mixed & Snapshot and count records combined & $(K+1)S$\\
\bottomrule
\end{tabular}
\end{table}

The historical result map links 184 runs to their measurement and policy
configurations. One retained policy checkpoint from each of 144 learned runs
matches its saved manifest; the remaining 40 runs are nonlearned comparisons.
Recorded source revisions expose batch normalized scalar, factor, dense
random mixed, and bounded policies, but a revision alone cannot prove the absence
of uncommitted execution changes. We therefore preserve their reconstruction
results without invoking the single chain theorem to certify their historical
kernels. The matched continuous study separately fixes individual state features
and tests their independence from other chains.

The original studies use scalar quadrature or an enumerated finite graph where
available, and the matched continuous comparison adds refined numerical
integration (Appendix~\ref{app:continuous}). The remaining broad studies draw a
finite prior candidate set, evaluate the Born likelihood, and resample according
to normalized weights. Factor family studies use 512 candidates per record at two
qubits and 256 at four through ten qubits; thermal factor studies use 256
candidates at two and four qubits and 64 in the compressed eight- and ten-qubit
studies; dense studies retain their original comparators. This finite
importance resampling distribution serves as a common comparator, not an
independently converged posterior. It enters evaluation and the \flowgrpo{}
energy reward where stated. \method{} does not use the simulated truth as a
training label.

The historical artifacts retain resampled states but not the raw unnormalized
importance weights.  Weight effective sample size, maximum normalized weight,
and duplicate fraction therefore cannot be reconstructed honestly.  We do not
use these finite candidate distributions to support posterior agreement claims.
They provide a shared descriptive target for within cell method comparisons
only.  Exact posterior language is reserved for the scalar quadrature and
finite transition studies.

\subsection{Enumerated transition study}

Appendix~\ref{app:finite} defines the separate 48-state graph, posterior,
proposal kernels, and repeated training comparison.

\subsection{Methods and information available to GRPO-QPS}

BuresTomFlow is the unchanged Bures-geometric transport that supplies
flow-initialized posterior samples.  Flow-GRPO fine-tunes the distribution of its stochastic
flow trajectories using the energy reward in Eq.~\ref{eq:energyreward}.
GRPO-QPS instead starts from the same BuresTomFlow samples, schedules a family
of reversible physical proposals, and applies a complete posterior correction
before collection.  The matched control shares this pipeline but replaces
the learned action probabilities with the predefined fixed mixture.

GRPO-QPS observes the current physical state summary, the measurement record
through its log Born likelihood, and the Born score when available.  It never
observes the simulated truth during training or collection.  This information
boundary prevents direct truth leakage.  The continuous policies are trained
separately for one seeded posterior set used during method development.

\subsection{Evaluation measures}

For posterior samples $\{\rho_s\}_{s=1}^{S}$ and simulated truth $\rho_\star$,
the mean state distance is
\begin{equation}
  \overline d_{\mathrm B}
  =\frac{1}{S}\sum_{s=1}^{S}\bures(\rho_s,\rho_\star).
  \label{eq:meanstate}
\end{equation}
This quantity measures reconstruction accuracy for one simulated instance.  It
does not determine whether the full posterior distribution is correct.

When a comparator $Q$ is available, we compare a generated distribution $P$
through the Bures energy statistic
\begin{equation}
  \mathcal E_{\mathrm B}(P,Q)
  =2\E\bures(X,Y)-\E\bures(X,X')-\E\bures(Y,Y'),
  \label{eq:energydistance}
\end{equation}
where $X,X'\sim P$ and $Y,Y'\sim Q$ are independent.  We have not established
that Bures distance has the negative type property needed to make this a
characteristic energy distance on the full density matrix space.  We therefore
interpret the finite-sample value as a descriptive agreement statistic.  A
near zero value alone does not prove equality of distributions.  The diversity
ratio is the mean within sample Bures distance under $P$ divided by the same
quantity under $Q$.  A value near one indicates matched spread, while a value
above one indicates an overly diffuse sample distribution.

The exact scalar coverage study uses equal tail intervals at nominal levels
$\gamma\in\{0.50,0.80,0.90,0.95\}$.  For each of three state families and three
shot counts, 400 truths are drawn independently from the declared prior and a
new measurement record is simulated for each truth.  Coverage is the number of
truth coordinates inside the exact posterior interval divided by 400.  Wilson
$95\%$ intervals quantify the binomial uncertainty for each cell.

The dense four-qubit diagnostic is different.  For each of 160 held-out truths
in a cell, the posterior center is the projected arithmetic mean of 32 samples.
The radius is the empirical 90th percentile of their Bures distances to that
center, multiplied by the predefined scale $1.2$.  Coverage is the fraction of
the 160 truths inside their corresponding balls, and coverage error is
$|\widehat C-0.9|$.  BuresTomFlow includes all 160 truths in every one of the 40
cells, so its observed coverage is $1.0$ and its error is exactly $0.1$ in every
cell.  The zero standard deviation therefore reflects repeated overcoverage,
not perfect calibration.  The 160 truths are nested diagnostics within a cell,
not 160 independent experiment replicates.

Family-specific observable error averages the error of predefined operators that
are not used as the primary state score: cat parity and coherence, Dicke
excitation number, cluster stabilizers, and thermal energy, magnetization, and
correlations.  These are analytic functions of the reconstructed state rather
than withheld measurements, and their Pauli components can overlap the
conditioning records, so we read them as observable reconstruction checks that
detect physical failures a single state distance can obscure, not as strictly
held-out predictions.  Only the held-out log likelihood, which is evaluated on a
separate measurement record not used for inference, is withheld in the
predictive sense.

\subsection{Pairing, replication, and uncertainty}

All within cell comparisons share the simulated truth, measurement record,
prior, likelihood, starting flow, and reported sample count. Posterior draws
within a cell describe a single conditional distribution and are not independent
replicates; the experimental cell is therefore the independent unit for all
reported means, standard deviations, and paired direction counts.

The standard deviations reported throughout this paper describe variation among
development-set cells; they are not confidence intervals for an untouched population.
Because the family and thermal studies contain different numbers of cells, a
pooled test would weight physical questions by campaign size rather than
scientific meaning. We therefore report family- and qubit-specific effects
instead. All experiments represent development-set evidence, and no result
is described as confirmation.

\subsection{Training and posterior collection}

Every continuous scheduler is a linear residual on top of fixed action logits.  A bounded
$1.5\tanh(\cdot)$ transform prevents the residual from producing unbounded
logits.  Table~\ref{tab:policies} reports the continuous inputs and actions.  The
scalar policy uses normalized coordinate, log likelihood, Born score, and a
constant.  Bounded coordinate policies replace the scalar coordinate with its
mean and standard deviation and use the score norm.  Factor and dense policies
use log likelihood, purity, spectral entropy, and a constant.  The final column
gives the action probabilities used to initialize and regularize the policy.

\begin{table}[h]
  \caption{Complete scheduler definitions.  A zero local scale denotes an exact
  prior refresh.  Dense action pairs give eigenbasis and spectrum scales.}
  \label{tab:policies}
  \centering
  \scriptsize
  \setlength{\tabcolsep}{3pt}
  \begin{tabular}{L{0.16\textwidth}L{0.22\textwidth}L{0.28\textwidth}L{0.17\textwidth}}
    \toprule
    Parameterization & State features & Actions & Initial probabilities \\
    \midrule
    Exact scalar & coordinate, log likelihood, Born score, constant & 0.008, 0.03, 0.12, 0 & 0.20, 0.30, 0.30, 0.20 \\
    Bounded coordinates & coordinate mean and standard deviation, log likelihood, score norm, constant & 0.01, 0.04, 0.14, 0 & 0.20, 0.30, 0.30, 0.20 \\
    Gaussian factor & log likelihood, purity, spectral entropy, constant & pCN $\beta$: 0.04, 0.12, 0.32, 1.0 & 0.20, 0.30, 0.30, 0.20 \\
    Dense random mixed & log likelihood, purity, spectral entropy, constant & basis: 0.04, 0.16; spectrum: 0.04, 0.18; joint: 0.12 and 0.10; refresh: 1 and 1 & 0.12, 0.18, 0.12, 0.18, 0.25, 0.15 \\
    \bottomrule
  \end{tabular}
\end{table}

The continuous policies use Adam with learning rate $0.01$ and two clipped
epochs per update.  They use probability ratio limits $[0.8,1.2]$, coefficient
$0.02$ toward the fixed action distribution, gradient norm limit $1.0$, and four
transitions per path.  The continuous scalar study uses 128 updates with
$G=16$, the dense four-qubit study uses 128 updates with $G=8$, and common
factor studies use 48 updates with $G=8$.  Table~\ref{tab:thermalconfig} gives
the exact thermal mapping.

\begin{table}[h]
  \caption{Thermal GRPO-QPS configurations.  Every run used an NVIDIA RTX A6000
    and the final saved checkpoint.  Seeds refer to independent development
  cells within each measurement condition.}
  \label{tab:thermalconfig}
  \centering
  \scriptsize
  \begin{tabular}{rL{0.28\textwidth}rrrr}
    \toprule
    Qubits & Representation & Updates & $G$ & Seeds per condition & Final checkpoint \\
    \midrule
    2, 4 & Thermofield factor interpolation & 40 & 16 & 2 & 40 \\
    5 & Dense full density matrix & 96 & 8 & 3 & 96 \\
    6 & Dense full density matrix & 64 & 8 & 2 & 64 \\
    8, 10 & Compressed thermofield factor interpolation & 32 & 8 & 2 & 32 \\
    \bottomrule
  \end{tabular}
\end{table}

Training and collection are separate phases.  Every policy is frozen before
the transitions summarized in Table~\ref{tab:collection}.  Collection retains
one terminal state from each parallel flow-initialized chain.  Consequently, the
saved acceptance rates describe transition behavior, but the design cannot
estimate within-chain autocorrelation, bulk or tail effective sample size, or
rank-normalized $\widehat R$.

\begin{table}[h]
  \caption{Collection budgets and acceptance rates.  Acceptance values are
  means across development runs, with the full run range in parentheses.}
  \label{tab:collection}
  \centering
  \scriptsize
  \setlength{\tabcolsep}{3pt}
  \begin{tabular}{L{0.17\textwidth}rrrrL{0.22\textwidth}}
    \toprule
    Study & Burn-in & Chains per run & Draws per chain & Runs & GRPO-QPS acceptance \\
    \midrule
    Exact two qubits & 128 & 6144 & 1 & 8 & $0.797$ $(0.791,0.806)$ \\
    Dense four qubits & 256 & 5120 & 1 & 40 & $0.499$ $(0.364,0.741)$ \\
    Common factor families & 96 & 96 & 1 & 50 & $0.750$ $(0.445,0.989)$ \\
    Dense thermal, two and four qubits & 96 & 96 & 1 & 8 & $0.686$ $(0.524,0.849)$ \\
    Thermal, five and six qubits & 96 & 192 to 576 & 1 & 20 & $0.567$ $(0.367,0.762)$ \\
    Compressed thermal, eight- and ten-qubit studies & 64 & 48 & 1 & 8 & $0.655$ $(0.565,0.743)$ \\
    \bottomrule
  \end{tabular}
\end{table}

\section{Extended quantitative results}
\label{app:allresults}

\subsection{Exact two-qubit posterior agreement}

Table~\ref{tab:q2exact} separates accuracy relative to the simulated truth from
agreement with exact quadrature. \base{} exhibited a large energy statistic and
was nearly six times more diffuse than the comparator. \flowgrpo{} reduced both
errors but retained a diversity ratio above three. \method{} reduced the energy
statistic by more than two orders of magnitude and restored the diversity ratio
to approximately one, recovering the full spread of the exact posterior.

\begin{table}[h]
  \caption{Exact two-qubit results over eight independently seeded posterior sets, each
    containing 96 evaluation queries.  Values are mean $\pm$ standard deviation
    across set means.  Lower is better except that the ideal diversity ratio is
  one.}
  \label{tab:q2exact}
  \centering
  \scriptsize
  \begin{tabular}{lrrrr}
    \toprule
    Method & Truth distance $\downarrow$ & Energy statistic $\downarrow$ & Diversity ratio & Observable error $\downarrow$ \\
    \midrule
    BuresTomFlow & $0.3480\pm0.0192$ & $0.3570\pm0.0234$ & $5.908\pm0.301$ & $0.2335\pm0.0116$ \\
    Flow-GRPO & $0.1653\pm0.0063$ & $0.2533\pm0.0117$ & $3.186\pm0.172$ & $0.0985\pm0.0055$ \\
    GRPO-QPS & $0.0663\pm0.0055$ & $0.00264\pm0.00041$ & $1.009\pm0.021$ & $0.0435\pm0.0041$ \\
    \bottomrule
  \end{tabular}
\end{table}

\begin{table}[h]
  \caption{Direct one-dimensional agreement metrics for the exact scalar
    posterior over 768 records.  Values are mean $\pm$ standard deviation across
    eight seeded-set means.  Quantile error averages absolute error at posterior
    probabilities $0.05$, $0.50$, and $0.95$.  The calibrated TV column subtracts
    the mean value 0.77554 obtained from independent exact samples with the same
  64-draw budget on the 2,001-point grid.  Lower is better.}
  \label{tab:exactmetric}
  \centering
  \scriptsize
  \setlength{\tabcolsep}{3pt}
  \begin{tabular}{lrrrr}
    \toprule
    Method & Wasserstein-1 & KS & Calibrated TV excess & Quantile error \\
    \midrule
    BuresTomFlow & $0.19001\pm0.01833$ & $0.65899\pm0.03678$ & $+0.17454$ & $0.18547\pm0.01755$ \\
    Flow-GRPO & $0.06358\pm0.00539$ & $0.46544\pm0.02111$ & $+0.09840$ & $0.07093\pm0.00518$ \\
    GRPO-QPS & $0.00991\pm0.00068$ & $0.10362\pm0.00330$ & $+0.00146$ & $0.01081\pm0.00050$ \\
    \bottomrule
  \end{tabular}
\end{table}

The raw histogram total variation is approximately 0.78 even for exact samples,
because only 64 draws occupy a 2,001-point grid; it is therefore not used without
calibration. The Wasserstein-1, KS, and quantile errors establish close
terminal sample agreement on this declared family, but they do not establish
convergence from arbitrary initializations. The supported conclusion is that
\method{} repairs the flow on this small target and recovers the exact posterior
spread, without establishing convergence from every initialization or chain
length.

\section{Reference stability and complete physical results}

\subsection{Exact scalar reference stability}

The corrected scalar reference uses trapezoidal endpoint weights on 2,001 grid
points.  Refining to 4,001 points changes the worst posterior mean by
$2.88\times10^{-6}$, and aligned grid total variation is at most
$1.85\times10^{-16}$, so the coverage results in Table~\ref{tab:scalarcoverage}
are stable under grid refinement.  At the nominal $90\%$ level the same
four-shot Bell cell that undercovers at $80\%$ recovers to $87.00\%$, which shows why aggregate reference coverage should be accompanied by
condition-specific counts.

\subsection{Dense four-qubit benchmark}

All 40 paired cells share the same truth, record, starting flow, and sample
count across methods.  \method{} improves state distance in every pairing
relative to \base{}.  Its simultaneous observable reconstruction and coverage changes
in Table~\ref{tab:q4} show that the correction does more than optimize the
reported state distance.  The matched control reaches the same three values,
so this benchmark locates the gain in the target-preserving sampling pipeline
rather than in the restricted scale scheduler used in this experiment.  Because these paired cells were available during method development, we report them as
descriptive development-set evidence rather than as confirmatory intervals.

\subsection{Thermal results across sizes}

Thermal systems show a large and persistent posterior correction gain relative
to both flow methods.  Table~\ref{tab:thermalmeans} reports the mean
state distance at each size.

\begin{table}[h]
  \caption{Thermal state distance.  Values are means over development-set
  cells, and lower is better.}
  \label{tab:thermalmeans}
  \centering
  \small
  \begin{tabular}{rrrrr}
    \toprule
    Qubits & BuresTomFlow & Flow-GRPO & Matched control & GRPO-QPS \\
    \midrule
    2 & 0.8058 & 0.8058 & 0.1987 & 0.1980 \\
    4 & 0.9419 & 0.9430 & 0.2405 & 0.2351 \\
    5 & 0.4047 & 0.4063 & 0.2023 & 0.2061 \\
    6 & 0.6643 & 0.6670 & 0.2205 & 0.2132 \\
    8 & 1.1022 & 0.9965 & 0.2609 & 0.2461 \\
    10 & 1.1405 & 1.0017 & 0.2618 & 0.2530 \\
    \bottomrule
  \end{tabular}
\end{table}

\method{} lowers the mean at every size and remains far below both flow methods.
The matched control is close to \method{} at every size, so the thermal gain
also belongs to the target-preserving proposals; \method{} is lower at five of
six sizes, with the control marginally lower at five qubits.  The eight- and
ten-qubit cells use compressed thermofield factors and do not define a
dense state scaling curve.

\begin{figure}[tbp]
  \centering
  \begin{subfigure}[b]{0.49\textwidth}
    \centering\includegraphics[width=\textwidth]{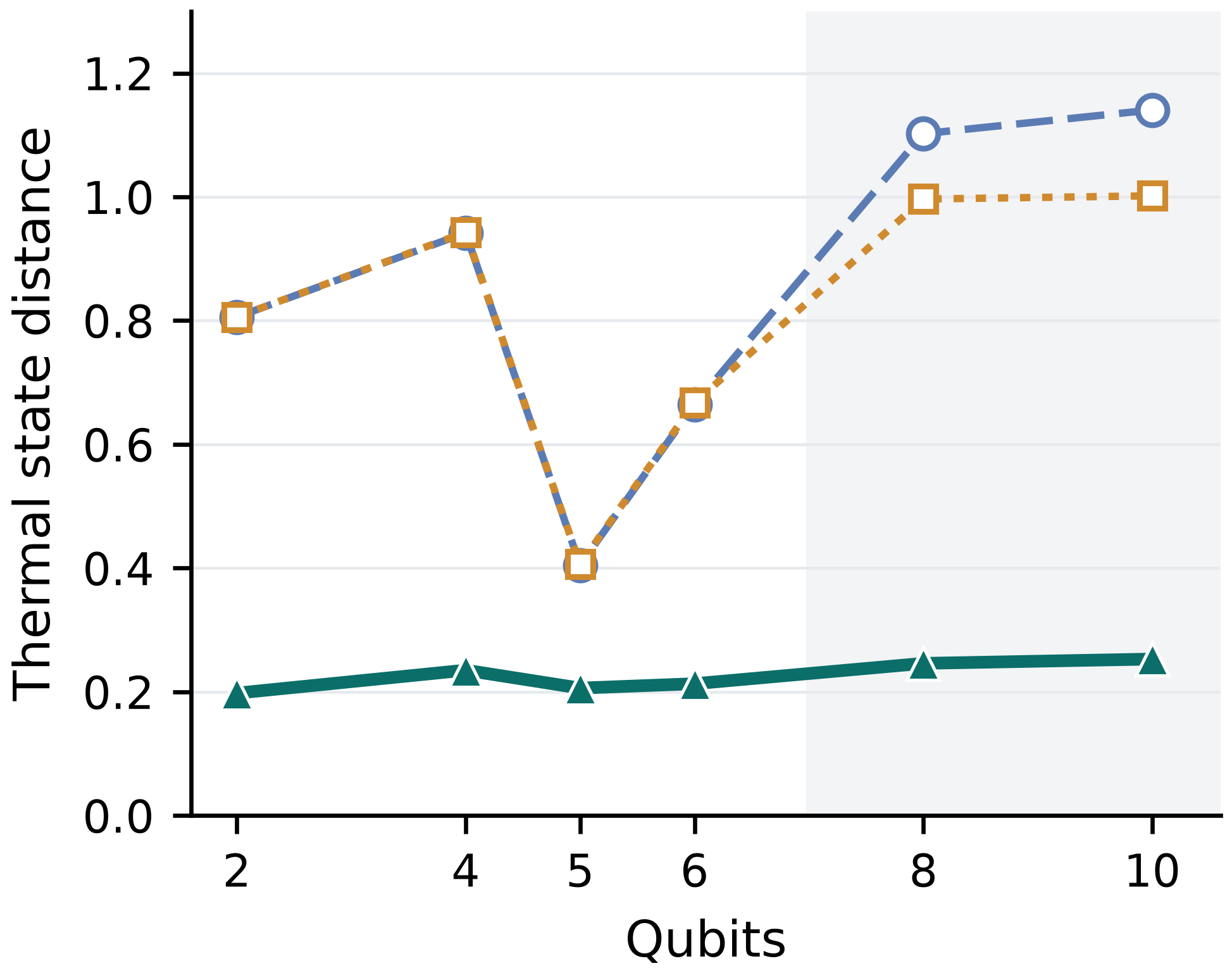}
    \caption{Mean state distance versus size}
    \label{fig:thermal-lines}
  \end{subfigure}\hfill
  \begin{subfigure}[b]{0.49\textwidth}
    \centering\includegraphics[width=\textwidth]{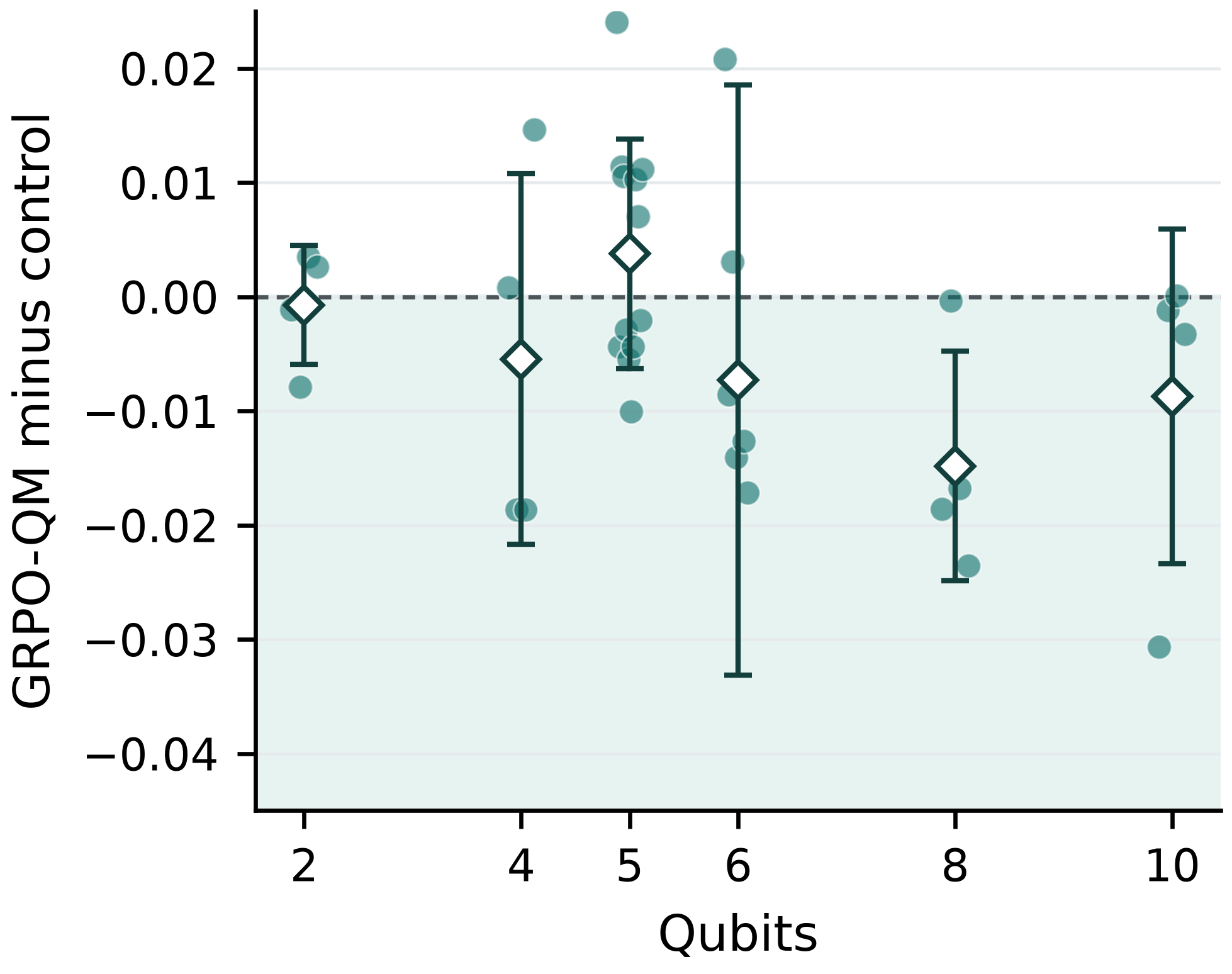}
    \caption{Paired per cell differences}
    \label{fig:thermal-paired}
  \end{subfigure}
  \caption{Thermal state distance.  In \subref{fig:thermal-lines}, \base{} is the
    blue dashed line with hollow circles, \flowgrpo{} the orange dotted line with
    hollow squares, and \method{} the solid teal line with filled triangles; the
    shaded region marks the compressed eight- and ten-qubit factors.
    \subref{fig:thermal-paired} shows the paired \method{} minus matched control
    difference in every thermal cell, which isolates the learned schedule; negative
    values favor \method{}, open diamonds are per size means, and bars are standard
  deviations across development-set cells.}
  \label{fig:thermal}
\end{figure}

\subsection{Complete system matrix}

Table~\ref{tab:allsystems} reports every compatible mean used in
Figure~\ref{fig:landscape}.  Nonthermal rows use the continuous normalized
purification factor representation.  Thermal rows use the specified dense or
thermofield representation.  Values from different representations answer
different physical questions and should not be pooled as one scaling curve.

\begin{table}[h]
  \caption{Complete common system matrix.  Values are mean state
  distances, and lower is better.}
  \label{tab:allsystems}
  \centering
  \scriptsize
  \begin{tabular}{llrrrr}
    \toprule
    Family & Qubits & BuresTomFlow & Flow-GRPO & Matched control & GRPO-QPS \\
    \midrule
    Random mixed & 2 & 0.6700 & 0.6703 & 0.6024 & 0.6074 \\
    Random mixed & 4 & 0.7003 & 0.7004 & 0.7042 & 0.7019 \\
    Random mixed & 5 & 0.9292 & 0.9291 & 0.9276 & 0.9288 \\
    Random mixed & 6 & 1.0817 & 1.0819 & 1.0804 & 1.0811 \\
    Random mixed & 8 & 1.2568 & 1.2570 & 1.2560 & 1.2556 \\
    Random mixed & 10 & 1.3374 & 1.3376 & 1.3373 & 1.3372 \\
    \addlinespace
    Thermal Ising & 2 & 0.8058 & 0.8058 & 0.1987 & 0.1980 \\
    Thermal Ising & 4 & 0.9419 & 0.9430 & 0.2405 & 0.2351 \\
    Thermal Ising & 5 & 0.4047 & 0.4063 & 0.2023 & 0.2061 \\
    Thermal Ising & 6 & 0.6643 & 0.6670 & 0.2205 & 0.2132 \\
    Thermal Ising & 8 & 1.1022 & 0.9965 & 0.2609 & 0.2461 \\
    Thermal Ising & 10 & 1.1405 & 1.0017 & 0.2618 & 0.2530 \\
    \addlinespace
    Cat & 2 & 0.9261 & 0.9264 & 0.2948 & 0.3283 \\
    Cat & 4 & 1.1874 & 1.1874 & 0.3612 & 0.3869 \\
    Cat & 5 & 1.2558 & 1.2567 & 0.5235 & 0.5414 \\
    Cat & 6 & 1.3049 & 1.3039 & 0.4896 & 0.5205 \\
    Cat & 8 & 1.3577 & 1.3558 & 0.5243 & 0.5935 \\
    Cat & 10 & 1.3879 & 1.3875 & 0.5546 & 0.6080 \\
    \addlinespace
    Dicke & 2 & 0.9235 & 0.9237 & 0.0513 & 0.0535 \\
    Dicke & 4 & 1.1323 & 1.1326 & 0.0633 & 0.0677 \\
    Dicke & 5 & 1.1974 & 1.1989 & 0.0635 & 0.0669 \\
    Dicke & 6 & 1.2534 & 1.2526 & 0.0611 & 0.0675 \\
    Dicke & 8 & 1.3252 & 1.3092 & 0.0640 & 0.0734 \\
    Dicke & 10 & 1.3725 & 1.3641 & 0.0614 & 0.0703 \\
    \addlinespace
    Cluster & 2 & 0.8700 & 0.8696 & 0.0896 & 0.0920 \\
    Cluster & 4 & 1.1305 & 1.1297 & 0.0799 & 0.0921 \\
    Cluster & 5 & 1.2105 & 1.2105 & 0.0775 & 0.0841 \\
    Cluster & 6 & 1.2702 & 1.2692 & 0.0779 & 0.0863 \\
    Cluster & 8 & 1.3407 & 1.3406 & 0.0774 & 0.0860 \\
    Cluster & 10 & 1.3769 & 1.3768 & 0.0813 & 0.0844 \\
    \addlinespace
    Shallow circuit & 2 & 0.9236 & 0.9234 & 0.4519 & 0.4832 \\
    Shallow circuit & 4 & 1.1541 & 1.1543 & 0.8668 & 0.9583 \\
    Shallow circuit & 5 & 1.2289 & 1.2287 & 1.0432 & 1.0922 \\
    Shallow circuit & 6 & 1.2811 & 1.2809 & 1.1652 & 1.2029 \\
    Shallow circuit & 8 & 1.3466 & 1.3465 & 1.2723 & 1.3125 \\
    Shallow circuit & 10 & 1.3807 & 1.3804 & 1.3474 & 1.3525 \\
    \bottomrule
  \end{tabular}
\end{table}

The complete matrix confirms that the \method{} pipeline strongly improves
thermal, cat, Dicke, and cluster states relative to both flow methods. The
\fixed{} column attributes most of this gain to the shared physical proposals and
posterior correction rather than to learning: it stays close to \method{} in
every cell and is in fact marginally lower on most cat, Dicke, cluster, and
shallow circuit cells. Random mixed factors change only slightly, and the
shallow circuit gain narrows with system size. This family dependence prevents a
universal scaling claim but clearly identifies where target-preserving physical
exploration is most useful.

\section{Mechanistic interpretation across physical families}

\subsection{Random mixed states separate the two representation regimes}

The dense four-qubit benchmark and the rank-16 factor study answer different
questions. The dense benchmark reveals a clear correction benefit under the full
density matrix architecture, while the scalable factor family changes little
beyond two qubits, where every method already achieves similar state distances.
Because the prior, rank, and representation differ between these two settings,
the contrast does not isolate optimization or support coverage as the sole cause.

\subsection{Cat, Dicke, and cluster states identify correction of flow coverage gaps}

For cat states, the \method{} pipeline lowers state distance by approximately
$56\%$ to $67\%$ relative to \base{} across the six sizes, and the same pattern
is even stronger for Dicke, and cluster states, where the flow error drops by
roughly $89\%$ to $95\%$. These results demonstrate that the complete pipeline
reaches states the initial flow underrepresents. The initialization comparison also
reveals that known family information explains much of the Dicke and cluster gap.
The data are consistent with the global refresh action supplying part of this
movement, but no action-level intervention has been conducted to test that
explanation.

\subsection{Thermal states motivate a scale selection hypothesis}

Thermal posteriors may contain temperature, field, and coupling directions with
different local widths. Because \method{} can adjust its proposal mixture using the current coordinate
and Born score, state-dependent proposal scales could be useful here. However,
these experiments did not measure local curvature and per-action acceptance
rates together, nor did they perform a controlled intervention that removes or
adds scale information. The observation therefore motivates a future test on new
thermal parameters rather than establishing the mechanism by itself.

\subsection{Shallow circuits reveal limits of the current representation}

The \method{} pipeline improves on \base{} by approximately $48\%$ at two qubits
but only about $2\%$ at ten qubits. This decreasing gain marks a limitation of
the tested representation and moves rather than identifying entanglement as the
cause. It does not imply that shallow circuit posteriors are intrinsically
inaccessible to a different representation or a nonlocal proposal.

\section{Validity boundaries and reproducibility}
\label{app:reproducibility}

\subsection{Physicality and numerical checks}

All reported factor states were reloaded and evaluated independently of their
training summaries. Across the four- to six-qubit factor studies, every output
has trace error below $1.12\times10^{-15}$ after normalization, and the maximum
ten-qubit thermal trace error is $4.45\times10^{-16}$. Positivity is guaranteed
by the factor map $FF^\dagger/\tr(FF^\dagger)$.

The ten-qubit implementation casts each complex channel to the model input type
before concatenation and releases temporary accelerator memory before posterior
collection. This ordering changes memory usage but not the scientific seeds,
measurement records, group sizes, flow steps, or proposal moves, and every
corrected ten-qubit calculation passed the same verification rules as the smaller
systems.

These checks establish that reported outputs are finite, positive, and trace one
under the declared representation. They do not establish posterior accuracy, which
is evaluated separately through exact distributional agreement at two qubits,
paired reconstruction at four qubits, and descriptive family-specific measures at
larger sizes.

\subsection{Representation dependent scope}

Dense global Bures distance is directly available through six qubits in the
predefined dense studies. The eight- and ten-qubit common family results use
rank-limited factors, and the corresponding thermal results use compressed
thermofield factors. A smaller distance within these representations does not
prove convergence to an unrestricted dense posterior; it establishes a comparison
among methods acting on the same specified physical description.

The scalable random mixed family uses a rank-16 Wishart prior, whereas the dense
four-qubit benchmark uses a full-rank prior. These two studies should not be
combined into a single random mixed scaling curve; their joint value is
diagnostic, showing how the correction behaves under two different posterior
geometries.

\subsection{Statistical scope and confirmation boundary}

All experiments were available during method development. They support mechanistic
comparisons and identify family-specific failure modes, but they do not estimate
performance on an untouched population. Coverage estimates are especially coarse
when only four or eight independent cells are available, so the paper reports
coverage error as a diagnostic and does not claim nominal calibration for the
large system studies.

The thermal effect is repeated across sizes, yet the same records used during method development also
informed method interpretation. A confirmatory test would freeze one policy, one
proposal family, and one evaluation protocol before exposing new thermal
parameters and state sets. The present result defines that testable claim but
does not replace it.

\subsection{Artifact integrity}

Artifact coverage differs across studies. The original reconstruction runs
saved terminal states rather than complete chains, and the first fixed-scale
correction discarded its returned policies. The new reproducibility replay
retains those policies and checkpoints in addition to the samples. The matched
continuous comparison retains complete chains, and the corrected optimizer study
retains final transition matrices, intermediate parameters, sampled training
paths, and optimizer states. Hash manifests bind the numerical artifacts to
their source snapshots. Public release remains a submission task; local
retention is not a completed public release.

\section{Supporting reconstruction and reference checks}
\label{app:additionalplots}

Figure~\ref{fig:familydistance} shows absolute mean state distances for the
family comparisons summarized in the main text.

\begin{figure}[h]
  \centering
  \begin{subfigure}[b]{0.32\textwidth}
    \centering\includegraphics[width=\textwidth]{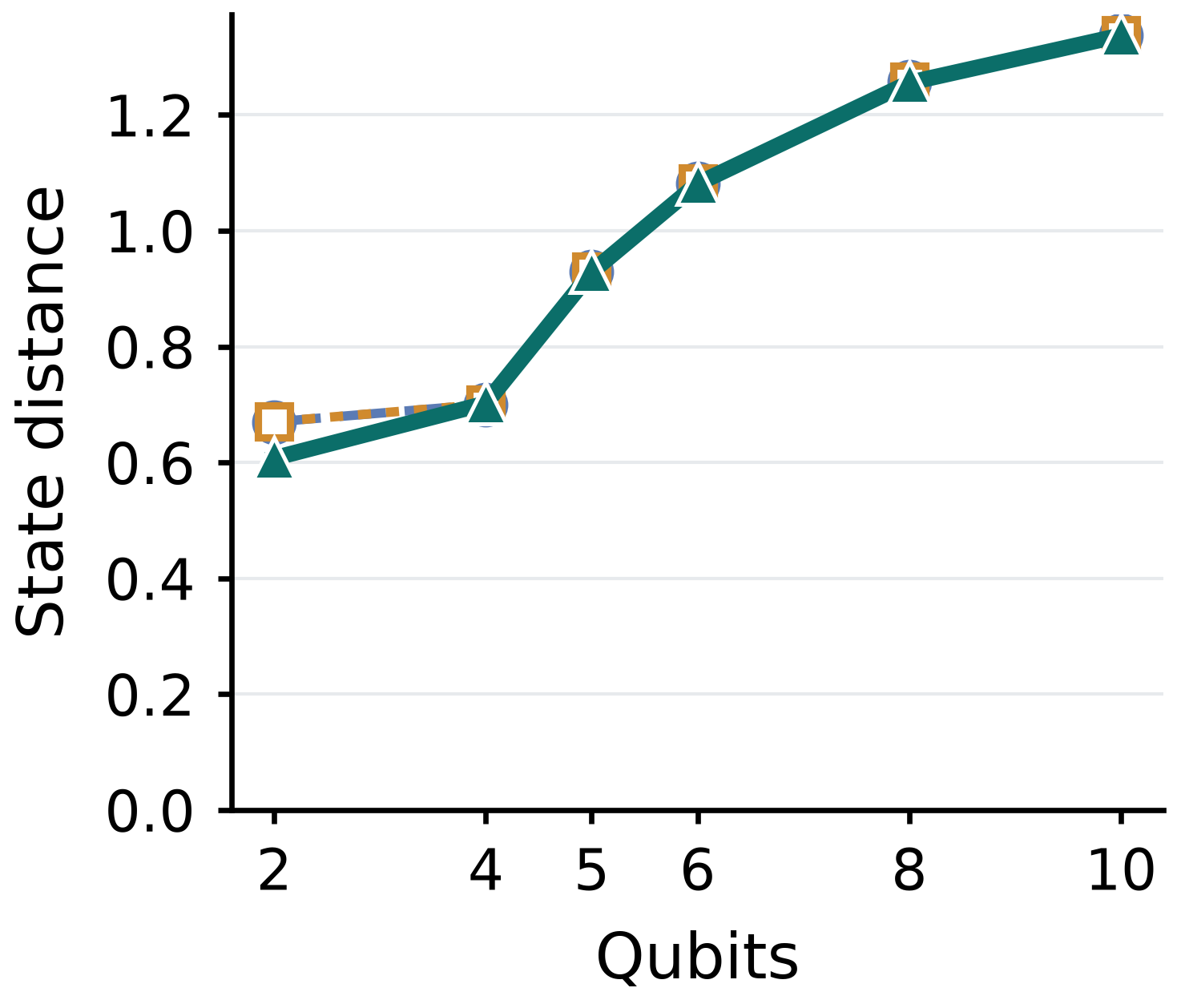}
    \caption{Random mixed}
  \end{subfigure}\hfill
  \begin{subfigure}[b]{0.32\textwidth}
    \centering\includegraphics[width=\textwidth]{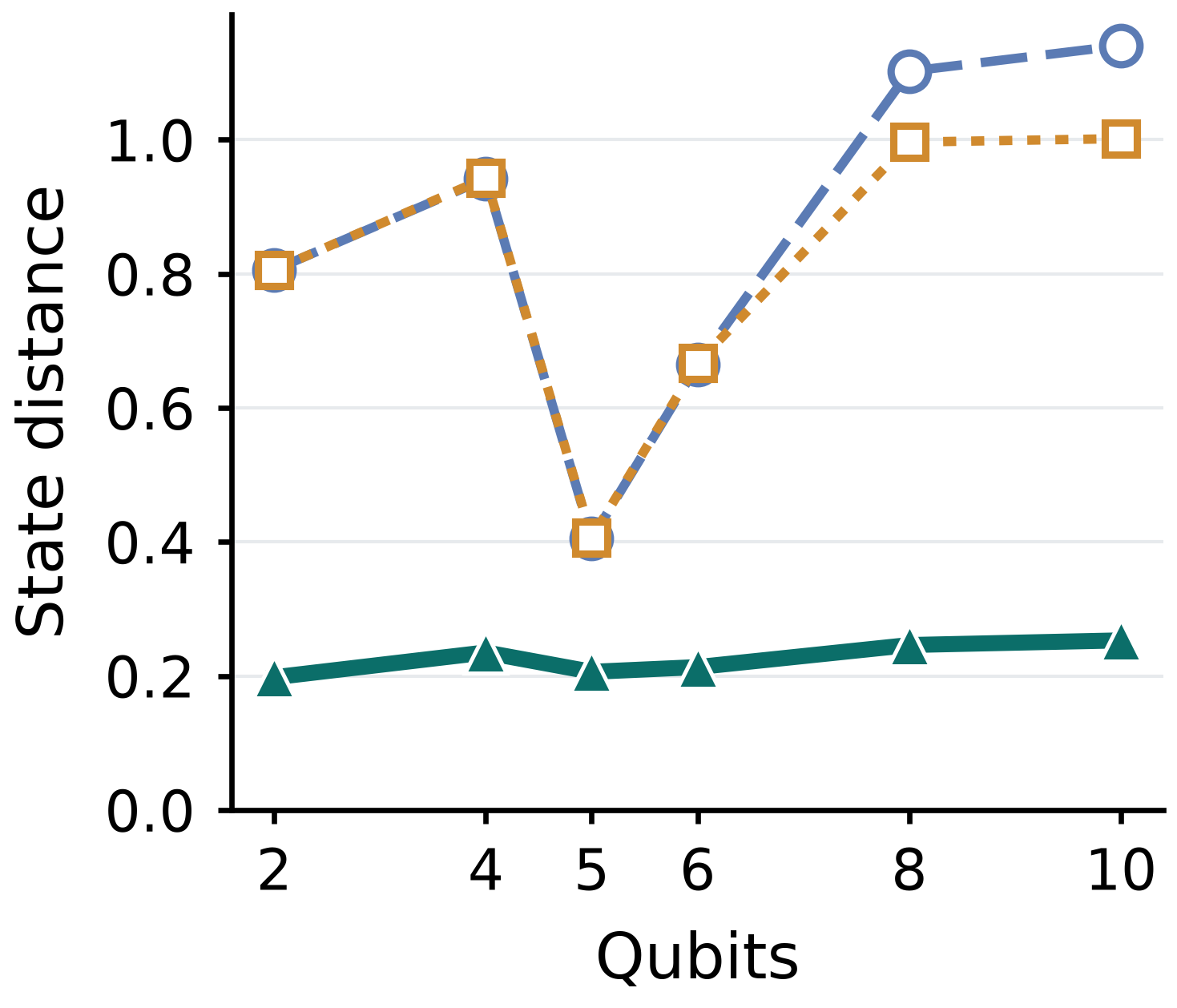}
    \caption{Thermal Ising}
  \end{subfigure}\hfill
  \begin{subfigure}[b]{0.32\textwidth}
    \centering\includegraphics[width=\textwidth]{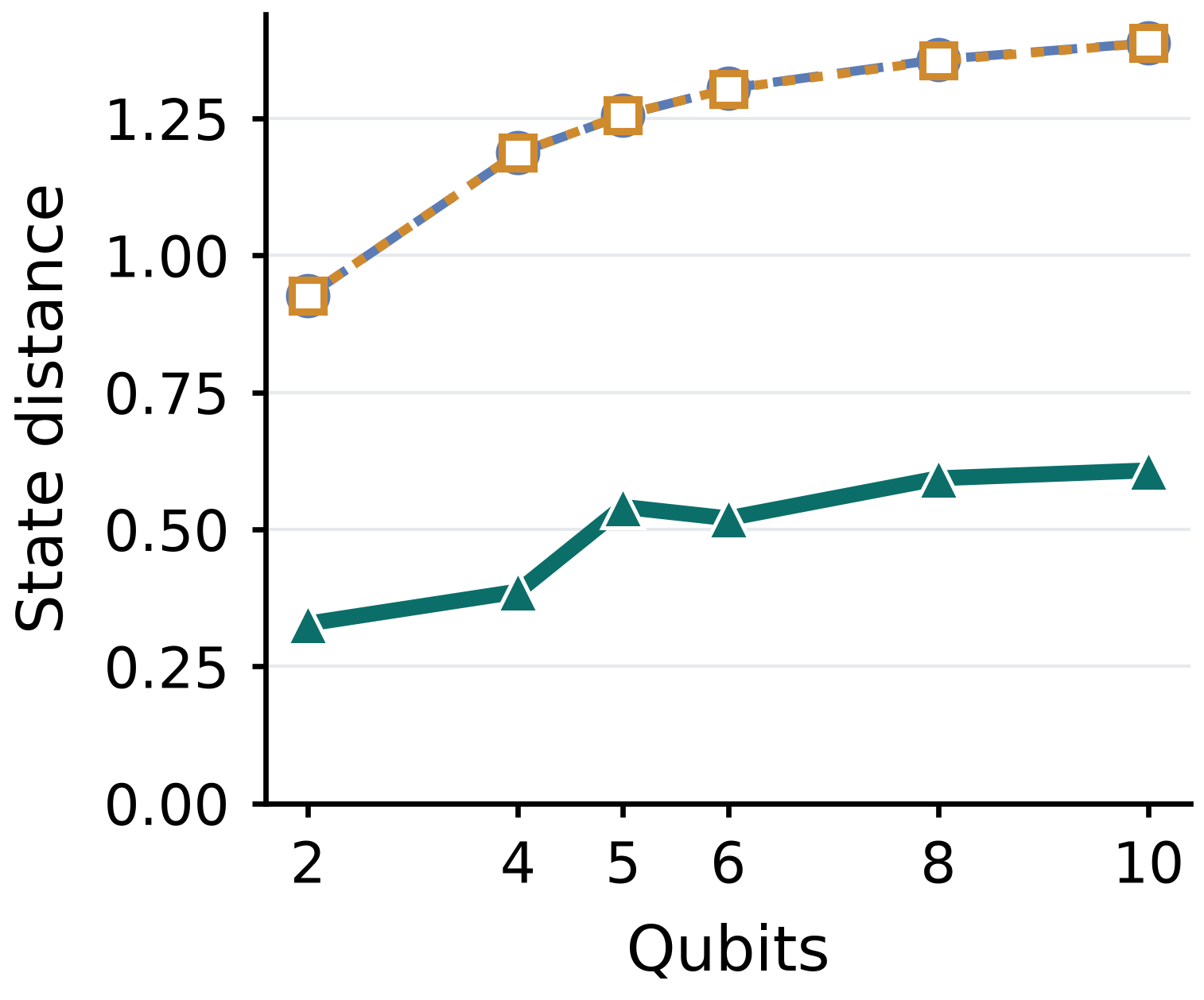}
    \caption{Cat}
  \end{subfigure}

  \vspace{3pt}
  \begin{subfigure}[b]{0.32\textwidth}
    \centering\includegraphics[width=\textwidth]{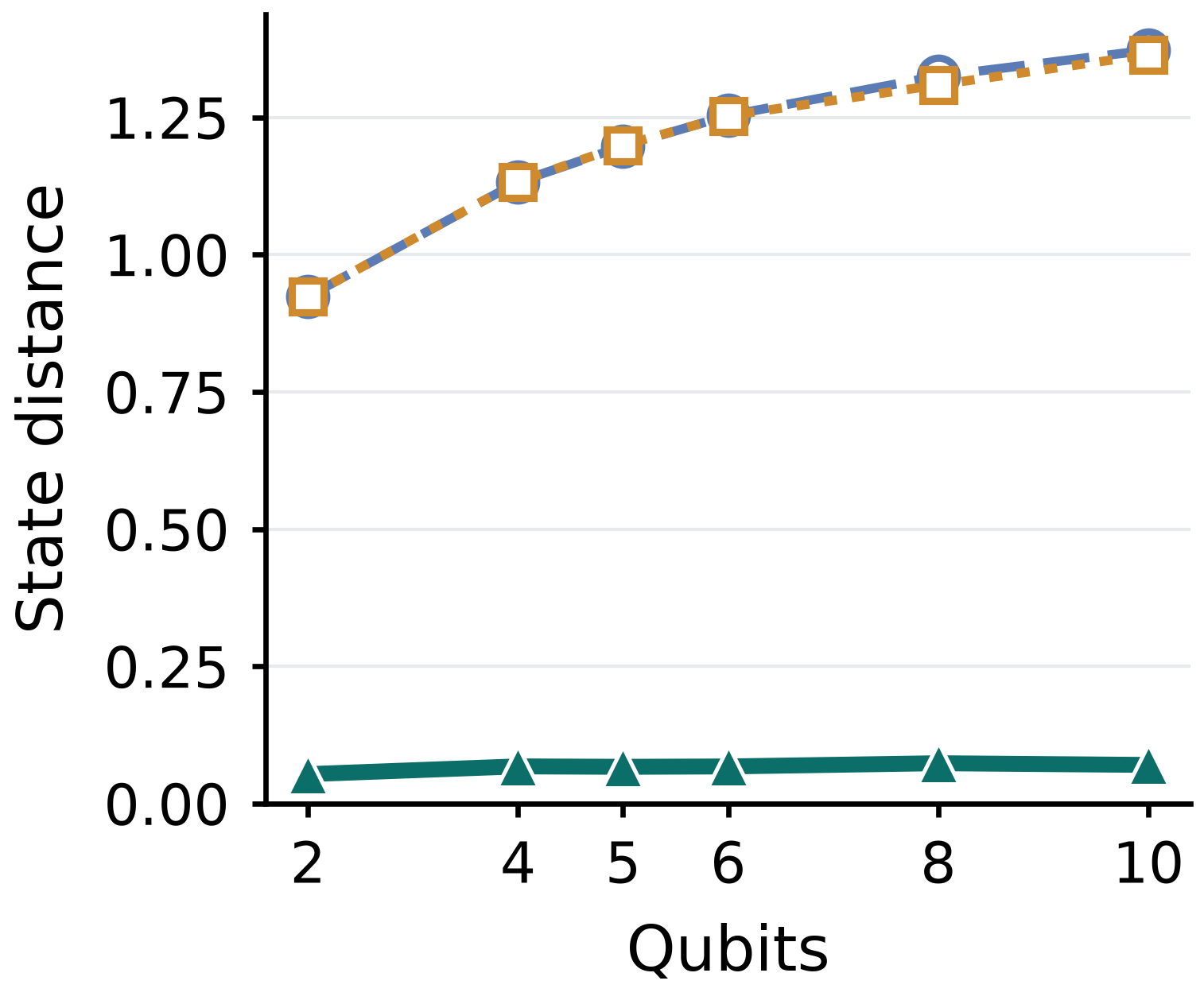}
    \caption{Dicke}
  \end{subfigure}\hfill
  \begin{subfigure}[b]{0.32\textwidth}
    \centering\includegraphics[width=\textwidth]{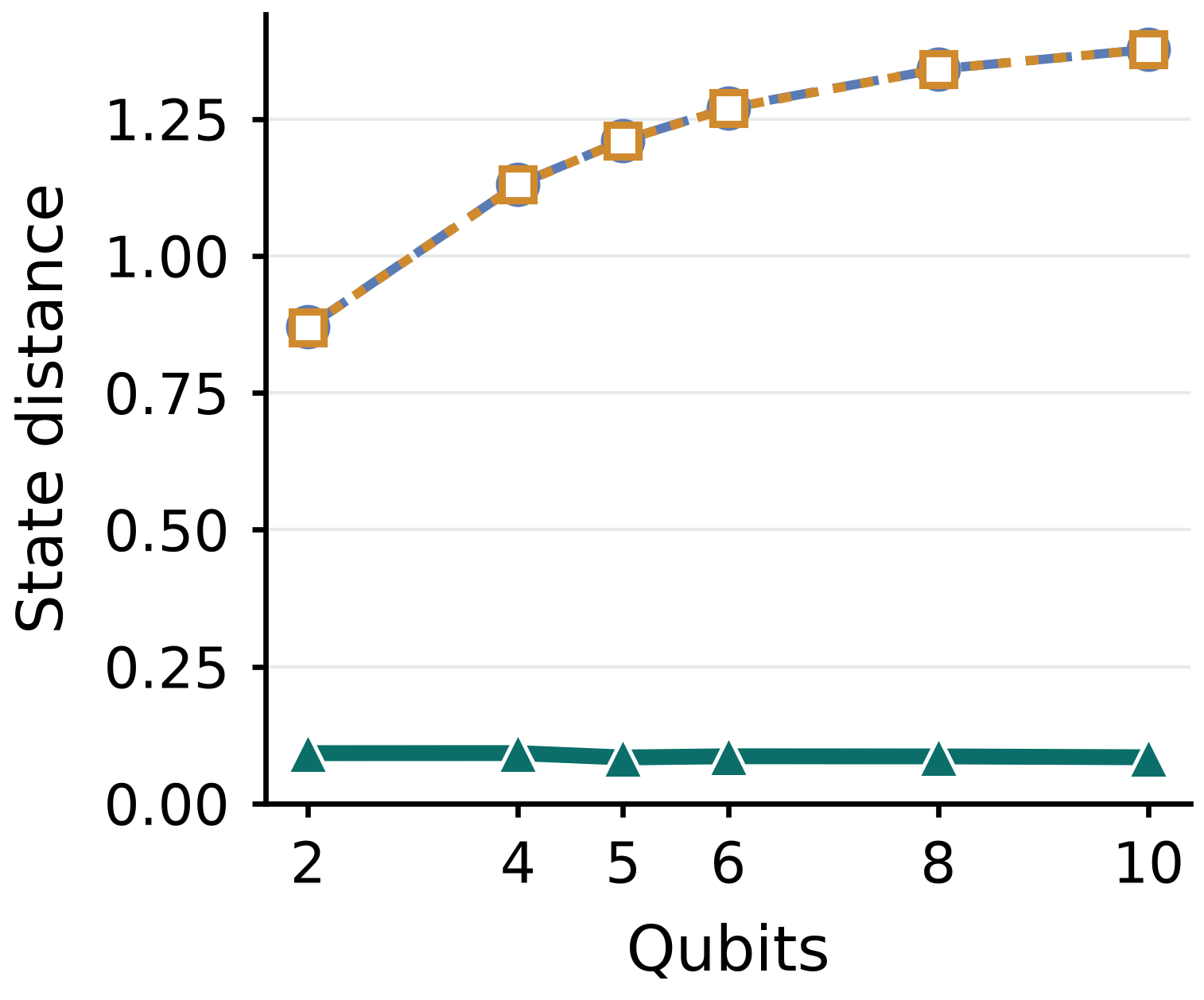}
    \caption{Cluster}
  \end{subfigure}\hfill
  \begin{subfigure}[b]{0.32\textwidth}
    \centering\includegraphics[width=\textwidth]{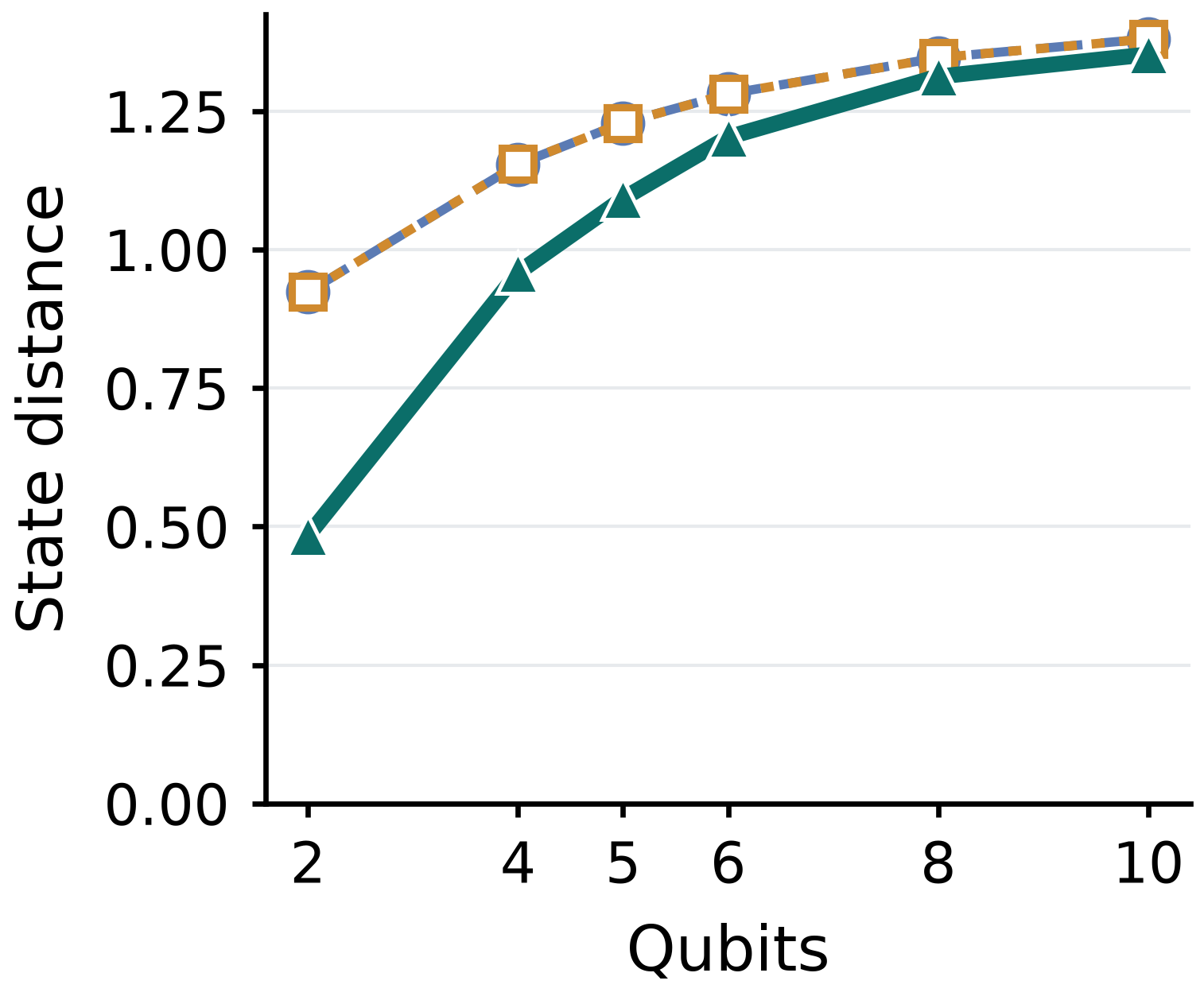}
    \caption{Shallow circuit}
  \end{subfigure}
  \caption{Absolute mean state distance by physical family and system size, each
    cell averaging the predefined 8-shot and 32-shot studies (lower is better).
    \base{} is drawn as blue dashed lines with hollow circles and \flowgrpo{} as
    orange dotted lines with hollow squares, so the two nearly coincident flow
    curves stay separable where they overlap; \method{} is the solid teal line
    with filled triangles.  \method{} separates most on Dicke, cluster, cat, and
    thermal states, overlaps the flow methods on random mixed factors, and
  converges toward them for large shallow circuits.}
  \label{fig:familydistance}
\end{figure}

\subsection{Coverage of numerical reference intervals}

The exact scalar family is also the one setting where credible intervals can be
checked against a known posterior rather than a finite comparator.  For each of
nine family by shot cells we draw 400 truths, simulate a fresh record for each,
form equal tail intervals from the numerical reference, not from GRPO-QPS samples, and
count how often the truth lands inside.  Table~\ref{tab:scalarcoverage} reports
the outcome: average coverage is within about one point of nominal at every
level.  The calibration is not uniform across subgroups.  The clearest exception
is the four-shot Bell cell, which covers $75.25\%$ at the nominal $80\%$ level,
where the record is least informative.  We therefore read this as
a check of reference intervals on the declared family, with a small development sample
subgroup warning, and we do not extend the claim to the larger systems, where
only a finite importance comparator is available.

\begin{table}[h]
  \caption{Coverage of numerical reference equal tail credible intervals.  The mean and range summarize nine cells, each with 400
  independently simulated truths and records.}
  \label{tab:scalarcoverage}
  \centering
  \small
  \begin{tabular}{ccc}
    \toprule
    Nominal level & Mean observed coverage & Cell range \\
    \midrule
    $50\%$ & $51.36\%$ & $48.50$ to $54.75\%$ \\
    $80\%$ & $80.33\%$ & $75.25$ to $83.50\%$ \\
    $90\%$ & $89.94\%$ & $87.00$ to $93.75\%$ \\
    $95\%$ & $94.50\%$ & $93.25$ to $96.00\%$ \\
    \bottomrule
  \end{tabular}
\end{table}

\subsection{Repetition on new seeds}
\label{sec:confirmation}

A further 145 cells repeat the training recipe with five new seeds.
Each posterior set retrains its policy, so this tests repetition of the recipe rather
than transfer of one frozen policy. Table~\ref{tab:confirmation} retains the
paired means. Earlier intervals used mean $\pm1.96$ standard errors, not
bootstrap resampling, and pooled cells reuse seed blocks across sizes.
We therefore report these results descriptively.

\begin{table}[tbp]
  \caption{Repetition on five new seeds. Policies were retrained for each posterior set.
Values are mean paired state distance differences across 25 cells per factor
family and 20 thermal cells. Negative values favor GRPO-QPS.}
  \label{tab:confirmation}
  \centering
  \small
  \setlength{\tabcolsep}{5pt}
  \begin{tabular}{lrll}
    \toprule
    Family & Cells & \method{} $-$ \base{} & \method{} $-$ control \\
    \midrule
    Random mixed & 25 & $-0.0003$ & $+0.0001$ \\
    Cat & 25 & $-0.7166$ & $+0.0122$ \\
    Dicke & 25 & $-1.1930$ & $-0.0009$ \\
    Cluster & 25 & $-1.1880$ & $-0.0024$ \\
    Shallow circuit & 25 & $-0.0416$ & $+0.0215$ \\
    Thermal Ising & 20 & $-0.7451$ & $-0.0006$ \\
    \bottomrule
  \end{tabular}
\end{table}

\section{Continuous physical families and reproducible comparison}
\label{app:continuous}

\subsection{Matched one-step learning objective}

The continuous comparison normalizes physical coordinates to \(u\in[0,1]^d\), and
its four actions use reflected Gaussian scales \(0.01,0.04,0.14\) or a uniform
prior refresh. The policy is a bounded linear residual on fixed logits:
\begin{equation}
 q_\phi(a\mid u,y)=
 \operatorname{softmax}_a\!\left(\log w_a+
 1.5\tanh[f_y(u)^\top W]_{a}\right).
 \label{eq:policy}
\end{equation}
Its features are the coordinates, their sine and cosine, the per-outcome log
likelihood, the per-outcome Born score, and a constant. The score is the
gradient of the log likelihood with respect to the coordinates. Fixed
hyperbolic tangent transforms bound the likelihood and score features. None of
these features requires the hidden truth or complete posterior enumeration.

For acceptance indicator \(z\), the one-step reward is
\begin{equation}
 R=z\{\delta(u,u')^2+0.02\},
 \qquad
 J(\phi)=\mathbb E_{u,y}
 \sum_a q_\phi(a\mid u,y)
 \mathbb E_{u'\sim K_a}
 [\alpha_\phi\{\delta(u,u')^2+0.02\}],
 \label{eq:objective}
\end{equation}
where \(\delta^2\) is the mean squared coordinate displacement, with a periodic
phase difference for cat. The objective deliberately tests immediate movement
rather than assuming an equivalence between movement and mixing. Direct
optimization integrates acceptance analytically and estimates \(J\) with
stratified action proposals; GRPO instead samples actions and acceptance events
at the same proposed state count.

Conditional on the recorded proposal, the policy-dependent factor in the event density is
\begin{equation}
 p_\phi(a,z;u,u')=
 q_\phi(a\mid u,y)
 \alpha_\phi(u,u',a)^z[1-\alpha_\phi(u,u',a)]^{1-z}.
 \label{eq:event}
\end{equation}
The full joint density also contains \(K_a(u'\mid u)\). It cancels from the replay
ratio only because the same recorded action and proposal appear in the numerator
and denominator and the proposal kernel is policy independent; it is not dropped
from the conditional event density itself. Gradients still pass through both
endpoint policy probabilities in \(\alpha_\phi\), so a rejected proposal keeps its
own event probability rather than being credited with an unobserved move. Groups
contain \(G=8\) proposals sharing a record and initial state, with advantages
\(A_g=R_g-(G-1)^{-1}\sum_{h\ne g}R_h\) computed without dividing by the
within-group standard deviation.
The update maximizes
\begin{equation}
 \widehat J_{\mathrm{clip}}=
 \frac1G\sum_g\min\{r_gA_g,
 \operatorname{clip}(r_g,0.8,1.2)A_g\},
 \qquad r_g=p_\phi/p_{\phi_{\mathrm{old}}}.
 \label{eq:matchedgrpo}
\end{equation}
At the collection policy, unclipped event differentiation targets the expected
reward under the stated fixed starts. The clipped replay used during training is
an approximate optimization update rather than an unbiased long-horizon mixing
estimator.

\subsection{What the direct and group-relative gradients estimate}

For fixed initial state and record, define the event density
\[
 h_\phi(a,u',z\mid u,y)=
 q_\phi(a\mid u,y)K_a(u'\mid u)
 \alpha_\phi^z(1-\alpha_\phi)^{1-z}.
\]
With proposal parameters fixed and the realized reward treated as an environment value, differentiation gives
\[
 \nabla_\phi\mathbb E_{h_\phi}[R]
 =\mathbb E_{h_\phi}[R\nabla_\phi\log h_\phi].
\]
Conditional on the common start, the other independently sampled group rewards
form a baseline independent of a member's own event, so the leave-one-out
baseline leaves this unclipped expectation unchanged. Because the discrete action
and Bernoulli acceptance terms both depend on \(\phi\), omitting the latter
generally changes the gradient. This identity is an expected reward statement at
the generating policy, not a guarantee across multiple clipped replay epochs.

Direct optimization instead uses \(\mathbb E_z[R]=\alpha_\phi(\delta^2+0.02)\);
for a uniform, balanced allocation of proposal samples across the four actions,
multiplying by \(4q_\phi(a\mid u,y)\) estimates the action-weighted objective.
Each initial state is replicated eight times, giving two proposal noises per
action for the direct-gradient variant, while the GRPO variant uses eight sampled actions from the same
start. The two variants match proposed state counts and the intended one-step
objective, but not the individual draws or the number of optimizer steps.

The policy has \(4(4d+2)\) trainable residual coefficients, giving 24 for one
coordinate and 40 for two. Adam runs for 48 updates with learning rate \(0.01\)
and gradient norm bound \(1\). Direct optimization takes one optimizer step per
update, whereas GRPO takes two clipped replay steps. Event log ratios are bounded
to \([-20,20]\) before exponentiation for numerical stability. This matched one-step experiment
uses no additional KL penalty. The historical four-step experiments instead use
terminal group normalization and regularization, so their outcomes are not
attributed to this revised objective.

\subsection{State maps and prior measures}

The physical maps are fixed before data generation.
For cat states, write
\[
 |\psi_\varphi\rangle=
 (|0\rangle^{\otimes n}+e^{i\varphi}|1\rangle^{\otimes n})/\sqrt2,
 \quad
 \rho_{\mathrm{cat}}=(1-\eta)|\psi_\varphi\rangle\langle\psi_\varphi|
 +\frac{\eta}{2}(|0^n\rangle\langle0^n|+|1^n\rangle\langle1^n|).
\]
The prior is uniform on \(\varphi\in[0,2\pi]\) and \(\eta\in[0.02,0.20]\). The
phase endpoints represent the same physical state, so the sampler uses reflected
moves on the coordinate interval and its reward displacement and joint
diagnostics account for periodicity; ordinary marginal W1 remains sensitive to
the coordinate cut and is interpreted accordingly.

The Dicke vector is the equal superposition of computational basis states with \(k=\min(2,n)\) excitations.
For \(m=\min\{\binom nk,15\}\), the implemented noisy state is
\[
 \rho_{\mathrm D}(\eta)=
 (1-\eta)|D_{n,k}\rangle\langle D_{n,k}|
 +\frac{\eta}{m}\sum_{b\in\mathcal N_m}|b\rangle\langle b|,
 \qquad \eta\sim\operatorname{Unif}[0.01,0.10].
\]
Here \(\mathcal N_m\) contains the first \(m\) occupied site combinations in
lexicographic order, so the specified nuisance mixture is not permutation
invariant once only part of the excitation sector is retained. For \(n=2\) and
\(m=1\) both terms are the same projector and the parameter is physically
unidentifiable; the prior remains a valid target for implementation checks, but
this case is not evidence of learning an unknown physical noise level.

For the linear cluster state
\(|C_n\rangle=\prod_{j=1}^{n-1}\mathrm{CZ}_{j,j+1}|+\rangle^{\otimes n}\), the map is
\[
 \rho_{\mathrm C}(\eta)=
 (1-\eta)|C_n\rangle\langle C_n|
 +\frac{\eta}{n}\sum_{j=1}^n Z_j|C_n\rangle\langle C_n|Z_j,
 \qquad \eta\sim\operatorname{Unif}[0.01,0.13].
\]
These maps admit factor representations of at most 3, 16, and \(n+1\) columns
respectively, and zero padding to the allocated factor width changes no state.
They preserve global structure by construction and do not estimate arbitrary
excitation sectors or cluster defects.

For Dicke and cluster each observed product Pauli outcome has probability affine
in the normalized noise coordinate, and for cat it has the form
\(a+(1-\eta)(b\cos\varphi+c\sin\varphi)\), with coefficients obtained by
evaluating explicit physical factors at basis parameter settings. Tests compare
these probabilities and their gradients against direct factor calculations, so
likelihood construction uses the actual physical family rather than a learned
posterior approximation.

\subsection{Records, reference integration and frozen selection}

For each repetition, twelve prior-drawn training truths receive alternating 8
and 32 observed outcomes, eight validation truths use the same two information
levels, and twelve evaluation truths use the sequence \(4,8,32,128\) repeated
three times. Every outcome draws an independent product Pauli setting and a
result sampled from the Born probabilities, and recorded truths are used for evaluation only, never for
policy features, rewards, or candidate selection. The train, validation, and
evaluation random streams differ by offsets \(0,1000,2000\) from a repetition
seed, which follows the rule \(82{,}000{,}000+10^6n+10^5f+10^4r\) with family
index \(f=0,1,2\) for Dicke, cluster, and cat and repetition \(r=0,1,2\). A
shared fixed-mixture reference calculation advances each training record for 64
steps and supplies two states per record as fixed training initializations, so training never
draws from the exact evaluation posterior.

Scalar references use trapezoidal grids of 2,049 and 4,097 points; cat uses a
\(129\times65\) coarse grid and a \(258\times129\) fine grid with phase midpoints
and trapezoidal noise weights. Every accepted reference keeps the maximum change
in ordinary, sine, and cosine coordinate moments below \(0.005\) between the two
resolutions, a numerical stability check rather than a proof of all posterior
tails or joint features. An initial cat grid had placed an atom at the phase cut,
creating a mean refinement artifact for boundary-concentrated posteriors; the
corrected midpoint rule was tested for reflection symmetry and rerun with the
same scientific seeds and acceptance threshold, with a largest moment difference
of \(0.0004869\) in the verification run. All eighteen original cat runs are
retained but superseded for this comparison, and their costs remain in the
accounting.

The predefined mixture set is
\begin{gather*}
 (0.2,0.3,0.3,0.2),\quad (0.9,0.03,0.03,0.04),\quad
 (0.03,0.9,0.03,0.04),\\
 (0.03,0.03,0.9,0.04),\quad (0.03,0.03,0.04,0.9).
\end{gather*}
Validation runs 32 chains and 256 transitions for each mixture, and the same
terminal coordinate W1 selects the mixture used by the direct, GRPO, and error-tuned methods. The
acceptance-tuned method reuses those same five conventional validation runs and picks the
average acceptance closest to \(0.234\), so it incurs no second validation
search. All selected settings are saved before evaluation collection.

\subsection{Sequential Monte Carlo and initialization}

SMC starts from 32 independent prior particles and anneals the likelihood with
inverse temperature coefficient \(\beta_k=[\min(1,k/16)]^2\) at step \(k\), so
incremental weights multiply by \(L_y^{\beta_k-\beta_{k-1}}\). Each of the first
sixteen updates is followed by systematic resampling and one Metropolis
rejuvenation proposal, which refreshes from the prior with probability \(0.1\)
and otherwise makes a reflected Gaussian move whose local scale is chosen from
\(0.01,0.04,0.14,0.3,0.5\) by validation W1; once the likelihood is fully
introduced, the remaining steps are ordinary posterior MCMC rejuvenation. This
structure is why a higher-budget SMC endpoint can perform similarly to a long
corrected chain.

The flow comparison evaluates existing BuresTomFlow and Flow-GRPO checkpoints on
the new records, projecting BuresTomFlow states by minimum Bures distance among
256 prior-drawn coordinate candidates and using those selected coordinates to
initialize every projected-flow Markov-chain method. The prior-start methods
likewise share exactly the same saved coordinate arrays. SMC keeps its own prior initialization, because
sequential weighting is intrinsic to the algorithm. These are therefore
controlled starting conditions among the Markov methods together with an
end-to-end particle comparison, not identical random samples for all algorithms.

\subsection{Distribution metrics, uncertainty and chain diagnostics}

Coordinate W1 compares each empirical marginal with the weighted fine reference,
with the CDF discrepancy evaluated at the reference grid locations. For cat, the
joint feature vector is \((\cos2\pi u_1,\sin2\pi u_1,u_2)\). We measure joint
discrepancy using the biased empirical squared maximum mean discrepancy, averaged
over Gaussian kernels with widths \(0.1,0.3,1.0\), against 256 independent
reference draws. Because this empirical statistic has a nonzero finite-sample
floor, we calibrate it using a second reference sample matched to the output size.

Posterior spread is the empirical coordinate variance divided by the reference
variance. For each record and coordinate, the sample 5th and 95th percentiles
form a nominal 90\% interval and inclusion of the simulated truth is kept as a
Boolean count; these are genuine sample-based intervals, unlike the earlier
scalar reference coverage study, but with only three evaluation records per
information level and training repetition the counts remain descriptive.

Full Markov histories retain the initial state and all 1,024 transitions, and the
final 512 states per chain supply rank-normalized split \(\widehat R\), bulk and
tail effective sample counts, and autocorrelation. The diagnostic unit is one
record, coordinate, and initialization rather than one training repetition,
giving 1,728 units per method across the continuous study; phase diagnostics
should be read alongside the periodic joint statistic, since an ordinary phase
coordinate contains a cut, and no record with a convergence warning is removed from the
accuracy summaries. We independently recomputed all $8{,}640$ diagnostic entries
with ArviZ~0.22.0, replacing the earlier effective sample calculation that split
chains twice: every $\widehat R$ agrees within $4.5\times10^{-16}$ and every
threshold count is unchanged. The corrected effective sample sizes reported below
describe the same retained trajectories, with no chain rerun or exclusion, and
the pooled chain W1 is computed directly from the coordinates and is likewise
unchanged.

\begin{figure}[htbp]
\centering
\includegraphics[width=\textwidth]{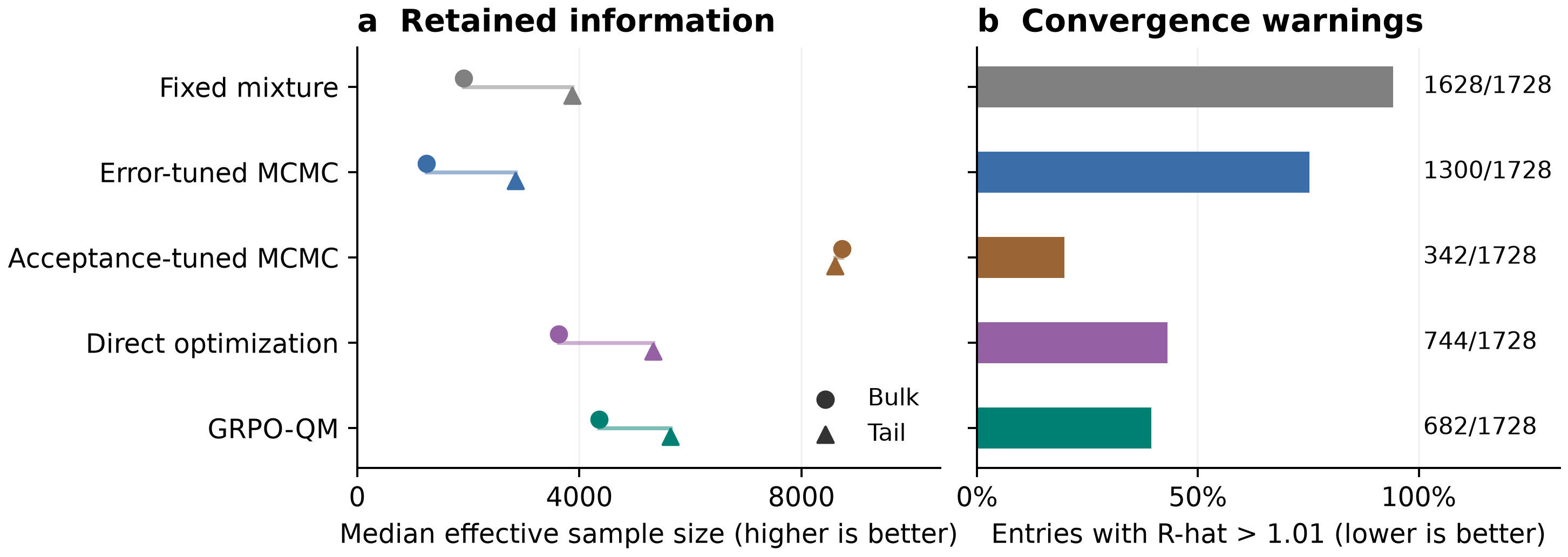}
\caption{Continuous trajectory diagnostics. (a) Circles and triangles show median bulk and tail effective sample sizes. (b) Bars show the fraction of entries with \(\widehat R>1.01\), with exact counts at right. Each method has 1,728 record, coordinate and initialization entries, not independent training repetitions. These descriptive summaries use every retained trajectory; convergence warnings do not trigger exclusions.}
\label{fig:diagnostics}
\end{figure}

Repetition-level pairing averages over the same twelve records for each method,
with ties defined by an absolute tolerance of $10^{-12}$. The 54 means combine
different physical conditions and are not treated as an independent, identically
distributed sample from a single population, so we report descriptive paired
effects, display every repetition, and attach no pooled significance claim to
selected family means. The secondary trajectory analysis was performed after
primary collection and did not change policy selection or the primary endpoint.

For each condition, let $d_r$ denote the difference between \method{}'s mean
error and the comparator's mean error on the same twelve records in training
repetition~$r$. We report all three values of $d_r$, their mean $\bar d$, and
the exploratory interval $\bar d\pm t_{0.975,2}\,s_d/\sqrt{3}$. This interval
assumes independent, approximately normal repetition effects, an assumption that
three repetitions cannot establish. The intervals are unadjusted descriptions,
not equivalence or confirmation tests. Figure~\ref{fig:paireduncertainty}
presents the four-qubit comparison with acceptance tuning. All sizes, both
starting conditions, both endpoints, and four paired comparators are retained in
the accompanying analysis.

\begin{figure}[htbp]
\centering
\includegraphics[width=\textwidth]{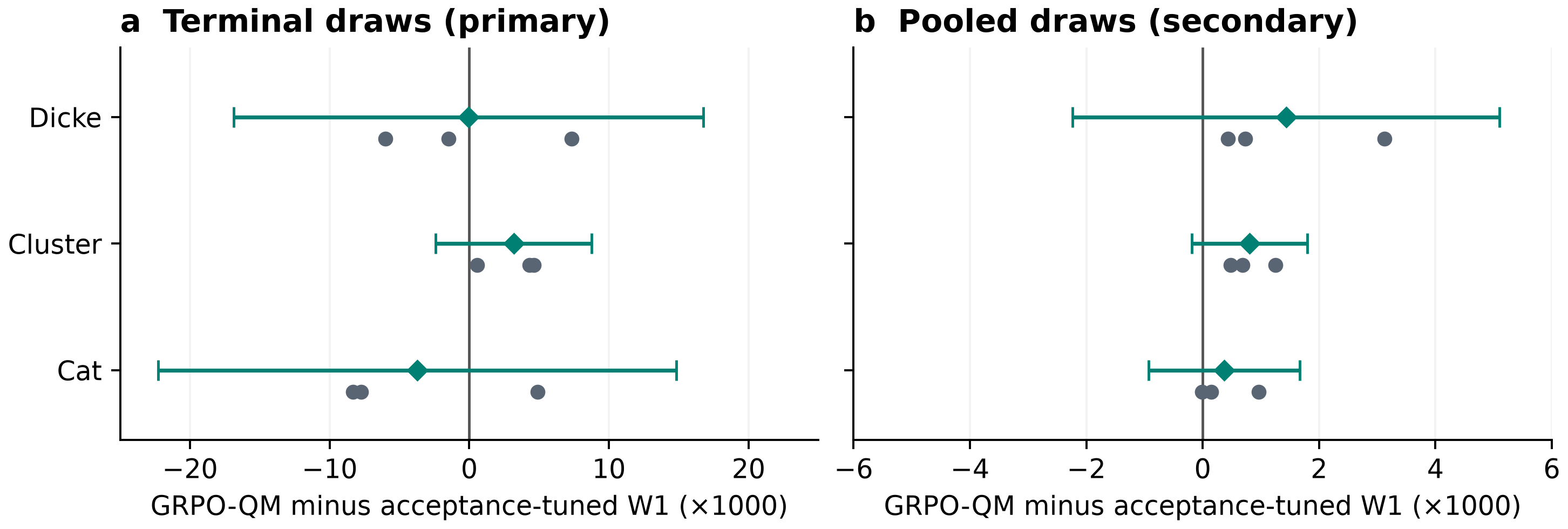}
\caption{Paired four-qubit effects against acceptance-tuned MCMC from common prior starts. Negative coordinate W1 differences favor GRPO-QPS. Grey dots show three training repetitions, each averaging twelve shared records; diamonds show their mean. Lines are exploratory, unadjusted 95\% Student t intervals with two degrees of freedom. All intervals include zero. (a) Terminal draws use 1,024 transitions. (b) Pooled draws use the final 512 states of 32 chains. The panels have different horizontal scales; neither establishes equivalence.}
\label{fig:paireduncertainty}
\end{figure}

\begin{figure}[htbp]
\centering
\includegraphics[width=\textwidth]{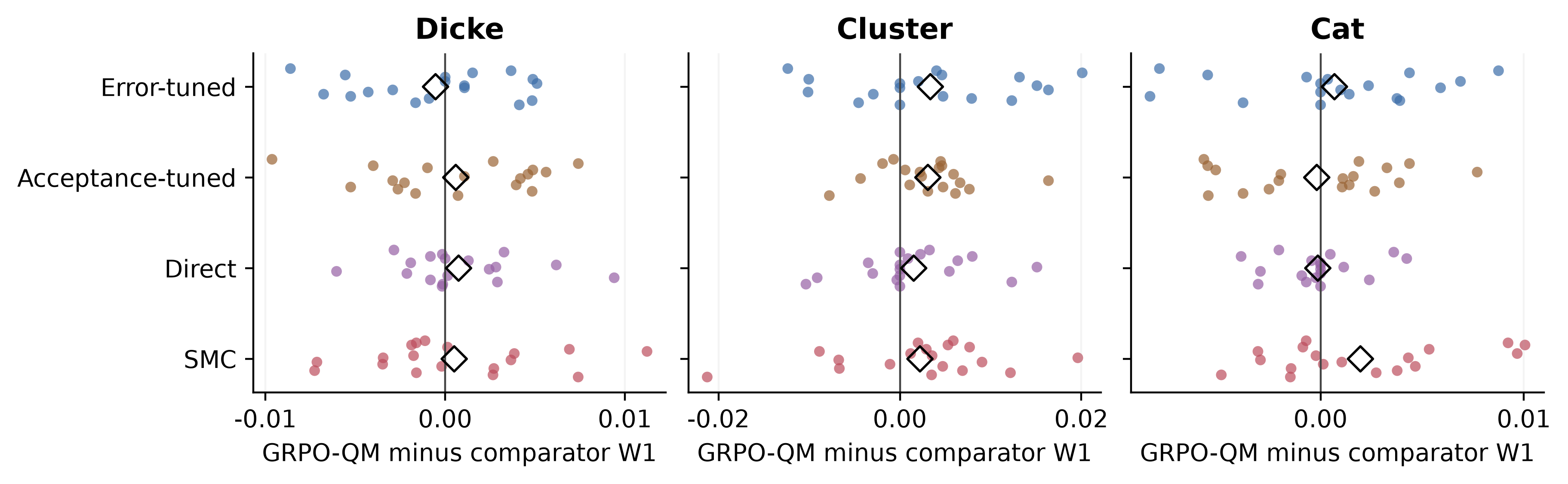}
\caption{Paired terminal W1 effects across all 54 continuous family, size and training repetitions, from prior starts at 1,024 transitions. Each dot is one repetition mean over twelve new records; open diamonds show the family mean. Negative values favor GRPO-QPS. The zero line separates directions, not statistical significance. Showing every repetition avoids selecting a favorable family or seed.}
\label{fig:paired}
\end{figure}

\subsection{Initialization and terminal comparisons}

Table~\ref{tab:starts} measures the effect of family information before any
posterior movement takes place; the terminal sample comparison at the final
budget appears in the main text (Figure~\ref{fig:budget}).
\begin{table}[t]
\centering\small
\caption{Four-qubit state error before posterior movement. Each mean uses 36 evaluation records and 32 states per record. Lower is better. The family prior ignores the new record; projection uses the known physical family.}
\label{tab:starts}
\begin{tabular}{lrrrr}
\toprule
Family & BuresTomFlow & Flow-GRPO & Family prior & Projected flow\\
\midrule
Dicke & 1.13540 & 1.13496 & 0.06348 & 0.08473\\
Cluster & 1.13584 & 1.13610 & 0.08962 & 0.11950\\
Cat & 1.18201 & 1.18238 & 0.64371 & 0.60210\\
\bottomrule
\end{tabular}
\end{table}

\begin{figure}[tbp]
\centering
\includegraphics[width=\textwidth]{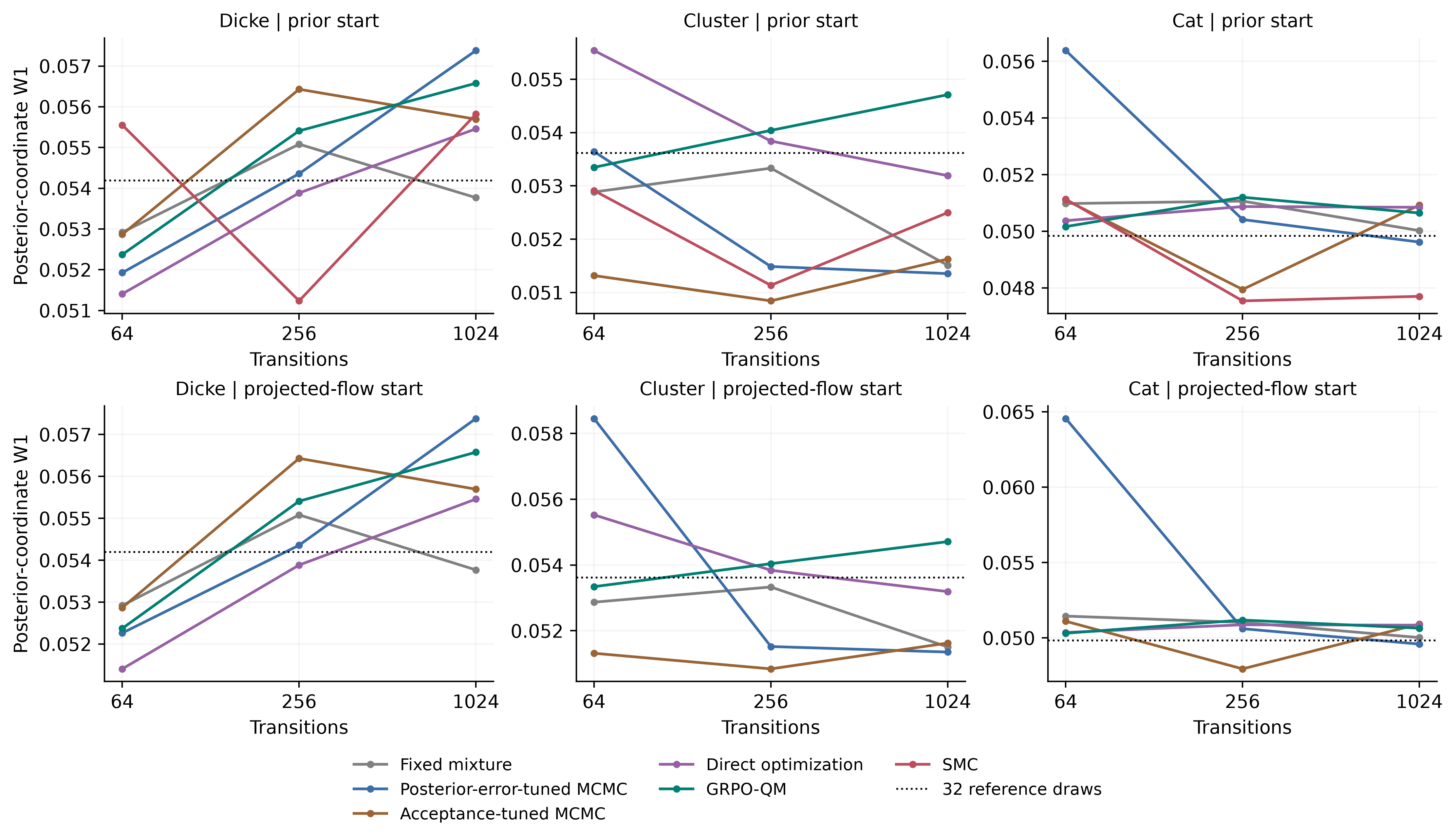}
\caption{Continuous posterior coordinate error against transition budget. Curves average all six sizes, three training repetitions and twelve records per repetition within each family. Upper panels use prior starts; lower panels use family-projected flow starts. The reference draw curve shows the error of 32 independently sampled reference coordinates. These are descriptive curves, not confidence intervals. SMC appears only in the prior start comparison.}
\label{fig:budget}
\end{figure}

\subsection{Complete continuous family and size comparison}

Table~\ref{tab:allnew} reports every family and size at the final prior start
budget. The same underlying records define both the method means and their paired
differences. This table complements the repetition level figure: means show the
scale of error, whereas pairing reveals how stable each difference is under
retraining. No family is excluded because its learned comparison is unfavorable.

\begin{table}[h]
\centering\scriptsize
\caption{All continuous terminal coordinate W1 means at 1,024 transitions, from prior starts. Each row averages three training repetitions and twelve evaluation records per repetition. Fixed denotes the original mixture; error-tuned and acceptance-tuned denote the two distinct conventional selection criteria. Lower is better.}
\label{tab:allnew}
\begin{tabular}{lr rrrrrr}
\toprule
Family & Qubits & Fixed & Error-tuned & Acceptance-tuned & Direct & GRPO-QPS & SMC\\
\midrule
Dicke & 2 & 0.05551 & 0.05745 & 0.05551 & 0.05560 & 0.05574 & 0.05804 \\
Dicke & 4 & 0.05375 & 0.06027 & 0.05999 & 0.05970 & 0.05993 & 0.05080 \\
Dicke & 5 & 0.05282 & 0.05298 & 0.05042 & 0.05252 & 0.05607 & 0.05675 \\
Dicke & 6 & 0.05737 & 0.05521 & 0.05265 & 0.05501 & 0.05215 & 0.05562 \\
Dicke & 8 & 0.05263 & 0.06376 & 0.05927 & 0.05294 & 0.05730 & 0.05968 \\
Dicke & 10 & 0.05054 & 0.05462 & 0.05632 & 0.05697 & 0.05828 & 0.05405 \\
Cluster & 2 & 0.05226 & 0.05409 & 0.05739 & 0.05617 & 0.05801 & 0.05362 \\
Cluster & 4 & 0.05179 & 0.05456 & 0.05397 & 0.05206 & 0.05714 & 0.05657 \\
Cluster & 5 & 0.05003 & 0.04868 & 0.05088 & 0.05052 & 0.05439 & 0.04628 \\
Cluster & 6 & 0.05407 & 0.05017 & 0.04603 & 0.05143 & 0.05225 & 0.05185 \\
Cluster & 8 & 0.05754 & 0.05244 & 0.05119 & 0.05884 & 0.05569 & 0.05404 \\
Cluster & 10 & 0.04337 & 0.04818 & 0.05030 & 0.05012 & 0.05077 & 0.05263 \\
Cat & 2 & 0.04681 & 0.04492 & 0.04726 & 0.04651 & 0.04752 & 0.03824 \\
Cat & 4 & 0.04643 & 0.04969 & 0.05047 & 0.04680 & 0.04674 & 0.04605 \\
Cat & 5 & 0.05063 & 0.04772 & 0.04866 & 0.05177 & 0.05234 & 0.04546 \\
Cat & 6 & 0.05043 & 0.04560 & 0.04759 & 0.05053 & 0.04905 & 0.04994 \\
Cat & 8 & 0.05270 & 0.05800 & 0.05642 & 0.05577 & 0.05639 & 0.05290 \\
Cat & 10 & 0.05315 & 0.05179 & 0.05517 & 0.05372 & 0.05184 & 0.05365 \\
\bottomrule
\end{tabular}
\end{table}

\section{Multimodal thermal comparison with conventional MCMC}
\label{app:multimodal}

\subsection{Posterior construction and evaluation metric}

This comparison uses the thermal transverse-field Ising family with coordinates
$(J,h,\beta)$ and widens the field prior to $h\in[-1.8,1.8]$. Records contain
only $Z$-basis outcomes. Under this observation model, the likelihood is invariant
to the sign of $h$, which produces two exact mirror posterior modes at $\pm h$.
As the measurement record becomes more informative, each mode becomes narrow.
Local proposals can move efficiently within one mode but rarely cross the
low-probability region between modes, while global prior refresh can change modes
but has decreasing acceptance. This construction therefore tests whether a
sampler can combine reliable mode changes with efficient local exploration.

The primary metric is the minimum effective sample size over $J$, $h$, $\beta$,
and a binary mode indicator, reported per $1{,}000$ likelihood calls. Reporting
the minimum ensures that a sampler must mix effectively in every measured
quantity rather than only in the mode indicator. Each Metropolis proposal is
charged one new likelihood evaluation, with the current-state likelihood cached.
The adaptation period of Haario adaptive Metropolis is included in its cost. We
also evaluate prior independence, fixed mixtures with a refresh component, and a
reflection mixture supplied with the exact sign symmetry. The reflection method
is included as an informative comparator because it tests whether explicit
knowledge of the symmetry is sufficient by itself.

\subsection{Learned proposal, reward, and training schedule}

The learned kernel is the state-dependent Gaussian mixture in
Eq.~\ref{eq:learnedmixture}. The neural policy predicts mixture weights, a drift,
and a Cholesky factor in logit coordinates. The complete mixture density is
evaluated in both directions in the Metropolis ratio. After the policy is frozen,
this construction preserves the posterior for every learned parameter setting
covered by Proposition~1, subject to the stated irreducibility and finite-chain
conditions.

The single-record experiments use a trajectory-variance objective over
$(J,h,\beta)$ and the mode indicator, with posterior reference moments. For the
transfer experiment, these reference moments are estimated from long
prior-independence chains on the training records only. Objectives computed from
a trajectory's own empirical statistics were also examined during development,
but oscillatory trajectories could increase those rewards without improving the
estimates, so they are not used for the reported transfer result. Training only
on the narrowest posteriors can drive the acceptance probability toward zero.
The transfer training set therefore combines records with broader and narrower
posterior modes. Under this schedule, all three reported transfer seeds train
successfully.

\subsection{Single-record four-qubit and six-qubit results}

At four qubits with 200 $Z$ shots, four of five training seeds outperform the
strongest conventional sampler on the same record. The ratios of learned minimum
ESS to the best conventional minimum ESS are $1.34$, $1.84$, $1.17$, and
$2.40$ for the four successful seeds; the remaining seed fails because its
acceptance probability approaches zero. The median ratio among the successful
seeds is $1.6\times$. Each successful learned policy also exceeds the Haario
adaptive-Metropolis result by a factor between $4.5\times$ and $6.9\times$. For
seed 0, the learned ESS vector is $(154,195,115,148)$ for
$(J,h,\beta,\mathrm{mode})$, compared with $(86,109,98,90)$ for prior
independence. The initial local random walk has a minimum ESS of 8 per $1{,}000$
likelihood calls and includes neither a prior-refresh move nor a reflection move.

At six qubits with 800 $Z$ shots, the posterior modes are narrower and the
acceptance rate of global prior refresh is approximately $3\%$.
Table~\ref{tab:multimodal6q} reports the complete comparison. The learned kernel
achieves a minimum ESS of 102, compared with 28 for prior independence, 27 for
the fixed mixture, 20 for Haario adaptive Metropolis, and 17 for the
symmetry-aware reflection mixture. The reflection method changes modes rapidly
but remains inefficient in $J$ and $\beta$, which explains its low minimum ESS.

\subsection{Transfer to unseen records}

The transfer experiment uses eight six-qubit records whose true parameters are
drawn from the prior. The five training records contain 300, 300, 1,000, 1,000,
and 3,000 $Z$ shots. Each of the three held-out records contains 3,000 shots. One
record-conditioned policy is trained for each random seed. Its inputs are the
current state, current log likelihood, and eight summary statistics of the
record's $Z$ outcomes. The posterior moments used by the training objective are
estimated only from long prior-independence chains on the five training records;
no exact reference or held-out posterior information enters training or model
selection. The conventional fixed mixture is also selected using only the
training records.

Each frozen policy is evaluated using 2,048 chains with 400 steps per chain. The
nine held-out ratios are reported in Table~\ref{tab:multimodaltransfer}. Every
ratio is at least one, and the mean ratios by training seed are $1.49$, $1.47$,
and $1.24$. All three transfer training runs complete without the acceptance
failure observed in one of the single-record four-qubit runs. On held-out record
5, the learned minimum ESS values are 44.8, 40.8, and 33.0, compared with 27.6,
28.2, and 27.4 for prior independence and approximately 16.5 for Haario adaptive
Metropolis.

\subsection{Cost and scope}

Training the transfer policy requires approximately $25.6$ million likelihood
calls. If this training cost is included in the deployment budget, the measured
sampling advantage offsets the training cost after roughly fifty served records
under the present accounting. The result is therefore best interpreted as an
amortized proposal-learning benefit across repeated inference tasks, not as an
improvement without additional cost for one posterior.

The scope is intentionally limited. On the broader four-qubit, 200-shot posterior
modes, prior independence is already highly effective, and the transfer policy
has mean ratios of approximately $0.83$ to $0.87$. On smooth unimodal factor
posteriors in 512 real dimensions, tuned HMC achieves approximately
$13\times$ the efficiency of pCN when each gradient is charged as two likelihood
calls, and none of the learned proposal classes tested here is competitive. The
reported transfer study contains three training seeds and three held-out records
from one nonlinear physical family. The multimodality arises from the
parameterized physical family and the observation symmetry. Extending the same
comparison to additional nonlinear families and many-parameter multimodal models
remains future work.

\subsection{Reproducibility record}

The single-record experiments are produced by
\texttt{scripts/grpo\_qm/mcmc\_failure\_premise/learned\_proposal\_grpo.py}, and
the transfer experiment by \texttt{multimodal\_campaign.py} in the same
directory. Recorded outputs include
\texttt{learned\_proposal\_*\_win\_seed*.json},
\texttt{learned\_proposal\_n6\_*\_narrow6\_z800.json}, and
\texttt{multimodal\_campaign\_*\_n6mixed.json}. The experiments use one NVIDIA
GB10, PyTorch 2.11.0+cu130, float32 matrix-exponential likelihood evaluation, and
float64 accept or reject calculations. The matrix-exponential implementation was
checked to within $4\times10^{-3}$ on log-target values near 7,000. Measurement
records are fixed by seed 1000 across policy seeds. The scripts and outputs are
recorded on branch \texttt{parivesh/mcmc\_failure\_experiments}.

\section{Enumerated state proof of principle}
\label{app:finite}

\subsection{Graph construction and exact target}

The 48-state study is separate from the continuous scalar quadrature. Each seeded state set
contains sixteen physical states at each rank \(1,2,4\), with total prior masses
\(0.30,0.45,0.25\); a seeded complex Gaussian QR calculation gives each
eigenbasis, and independent positive random weights normalized to sum to one give
the nonzero spectrum. The set seed fixes the graph, so different seeded state sets contain
different states. Five truth selection conditions (rank one, rank two, full
rank, boundary rank, and multimodal rank one) and three shot budgets
\(4,16,128\) yield fifteen posteriors per posterior set; these select truths and
measurement design rather than five separate graph constructions.

For states \(\rho_i\), prior weights \(p_{0i}\) and counts \(y_{so}\), the exact finite target is
\[
 \pi_i(y)=\frac{p_{0i}\exp(\ell_i)}
 {\sum_j p_{0j}\exp(\ell_j)},\qquad
 \ell_i=\sum_{s,o}y_{so}\log p(o\mid s,\rho_i).
\]
Nonmultimodal records use the \(XX,XY,YX,YY,ZX,ZZ\) bases, whereas multimodal
records use the \(XX,YY,ZZ\) bases. The spectrum proposal connects states of the
same rank with weights \(\exp(-\|s_i-s_j\|^2/0.03)\). The eigenbasis proposal uses
gauge-invariant basis coordinates at temperature \(0.6\), and the joint proposal
uses spectrum and scaled basis coordinates across ranks at temperature \(0.08\).
Self-proposals are excluded before row normalization in all three distance
kernels. The fourth action samples the rank-weighted prior. The saved graphs and
proposal matrices specify the exact numerical basis-coordinate construction.

\subsection{Rate, aggregation and intervention definitions}

The corrected matrix is reversible, and the eigensolve applies the symmetric
similarity transformation \(D_\pi^{1/2}TD_\pi^{-1/2}\) in float64 with a Hermitian
eigensolver. We use the absolute spectral gap, including negative eigenvalues in
the maximum magnitude, and apply no lazification. States below \(10^{-14}\) of the
largest posterior mass are excluded from the eigensolve only, with their outgoing
mass returned to the diagonal; across 360 evaluation cells the excluded mass is at
most \(3.19\times10^{-14}\), and the maximum detailed balance and stationarity
residuals are \(1.39\times10^{-17}\) and \(4.44\times10^{-16}\). These numerical
checks support the finite calculation rather than an unmeasured continuous rate.

Let \(g_{ri}\) and \(b_{ri}\) denote the learned and fixed rates for repetition
\(r\) and posterior cell \(i\). The absolute effect is \(\overline{g-b}\), the
relative change of pooled means is \((\bar g-\bar b)/\bar b\), and the mean
cell-wise percentage is \(\overline{(g-b)/b}\) with the median taken over those
ratios; these weight heterogeneous baseline rates differently and must not be
interchanged. Opportunity recovery divides the learned absolute gain by the gain
of a separately, directly optimized policy of the same form, which is a
diagnostic reference rather than an operationally available posterior
reference.

Across the 360 cell-level comparisons, the selected policy has 173 cases favoring
the learned policy, 100 ties, and 87 cases favoring the fixed mixture. At the
set level, 19 comparisons favor the learned policy and five favor the fixed
mixture, and its mean cell-wise relative gain stays near
5\% when the slowest 1, 5, 10, or 25\% of baseline cells are removed. This
sensitivity check does not make it equal to the \(0.612\%\) pooled relative gain;
both quantities are reported to prevent a denominator dependent headline.

The state dependence intervention likewise depends on its averaging measure.
Replacing \(q(a\mid i)\) by its unweighted node average removes an absolute rate
contribution of \(+0.001618\), whereas a posterior-weighted average gives
\(-0.000523\); both interventions rebuild the proposals and recompute the reverse
correction, and the sign change prevents a unique causal attribution to state
dependence from the original intervention alone.

\subsection{Independent training and final policy trajectories}

Eight repetitions use disjoint sets of twelve training sets, 96 unique training
state sets in total, and evaluate on 24 unseen evaluation sets. A two-way percentile bootstrap
uses 20,000 resamples: it resamples training repetitions and evaluation sets,
then averages over their Cartesian product, preserving the shared set structure.
The resulting intervals describe variation under the development study and its
resampling scheme rather than probabilities that a future confirmation will
succeed.

\begin{table}[h]
\centering\small
\caption{Repeated finite-state learning. Values are absolute spectral gap changes relative to the optimized fixed mixture with 95\% two-way percentile intervals over training repetitions and shared evaluation sets. The restricted tests retain spectral validation selection.}
\label{tab:finite}
\begin{tabular}{lrr}
\toprule
Inputs and reward & Mean effect & Interval\\
\midrule
Enumerated features and spectral reward & \(+0.001119\) & \([+0.000374,+0.001860]\)\\
Local inputs and local reward & \(+0.000234\) & \([-0.000127,+0.000585]\)\\
\bottomrule
\end{tabular}
\end{table}

Rich feature policies recover \(7.51\%\) of the directly optimized policy class
opportunity and policies using local inputs and local rewards recover \(1.57\%\), both below
the predefined \(70\%\) threshold for substantial recovery. The local input
representation still originates from the enumerated pipeline, so it should not be
conflated with the continuous implementation with chain-independent features; overall the
finite-state evidence establishes a limited learned spectral effect under rich
information.

The selected policy is evaluated on each of the 360 posteriors
with 24 dispersed chains, 1,024 discarded transitions, and 8,192 retained
transitions without thinning. It raises mean mode transition counts from 3650.5
to 3915.1, but the paired set-level interval for effective samples per counted
transition work is \([-0.01893,+0.00941]\) and the posterior \(L^1\) error
interval also includes zero. The same policy responsible for the spectral result
is thus checked at finite length without substituting an earlier checkpoint, and
enumeration and precomputation prevent these counts from being read as continuous
likelihood throughput.

\section{Physical averages and movement rewards}
\subsection{Spectral gains do not transfer to every observable}
\label{sec:observablevariance}

A spectral gap describes the slowest transition mode, but each physical quantity
overlaps that mode to a different degree, so a larger gap need not lower the
estimation error for every average. Let $P$ be a stationary kernel, $f$ a scalar
observable, $\widetilde f=f-\E_\pi f$, and $\bar f_T=T^{-1}\sum_{t=1}^T
f(X_t)$. The stationary covariances give
\begin{equation}
 \operatorname{Var}_\pi(\bar f_T)=\frac{1}{T}\left[
 \langle\widetilde f,\widetilde f\rangle_\pi+
 2\sum_{k=1}^{T-1}\left(1-\frac{k}{T}\right)
 \langle\widetilde f,P^k\widetilde f\rangle_\pi\right],
 \label{eq:finitevariance}
\end{equation}
and for a reversible kernel with orthonormal eigenfunctions $v_j$ the lag
covariance is $\sum_{j\geq2}\lambda_j^k\langle\widetilde f,v_j\rangle_\pi^2$.
Changing the kernel generally moves several eigenvalues and eigenvectors at once,
so a larger spectral gap alone does not determine the variance of any particular
observable; the gap bounds the slowest mode but not the projection of a given
observable onto it. Evaluating Eq.~\ref{eq:finitevariance} exactly for
all eight retained enumerated policies across the 24 evaluation sets, five
conditions, and three shot counts, the learned policy lowers the stationary
sample mean variance of the $XX$ correlation but raises it for rank and purity
(Table~\ref{tab:observablevariance}). The spectral gain reported above is
therefore real but observable-specific, one reason it does not translate into a
uniform sampling advantage.

\subsection{Movement is not a sufficient sampling objective}
\label{sec:movementcounterexample}

The same principle admits a sharp closed-form illustration that requires neither
training noise nor imperfect acceptance. Consider the four computational basis
pure states of two qubits under a uniform posterior, where distinct states have
squared Bures distance two. Let $H$ be the orthogonal matrix whose columns are
the constant, $ZI$, $IZ$, and $ZZ$ sign vectors, each divided by two, and define
\begin{equation}
 P_A=H\operatorname{diag}(1,0.9,0.1,0.1)H^\top,
 \qquad P_B=H\operatorname{diag}(1,0.7,0.7,0.7)H^\top.
\end{equation}
Both kernels have positive entries, rows summing to one, and symmetric transition
probabilities, so each preserves the uniform posterior and can serve as its own
proposal with acceptance probability one. From stationarity,
$\E\bures^2(\rho_{X_0},\rho_{X_h})=2\,(1-\tfrac14\operatorname{tr}P^h)$, and
$P_A$ produces larger movement at both one-step and four-step horizons yet has the smaller
spectral gap and the higher asymptotic variance for the $ZI$ average
(Table~\ref{tab:movementcounterexample}); summing the geometric covariance series
of an eigenfunction with eigenvalue $\lambda$ gives asymptotic variance
$(1+\lambda)/(1-\lambda)$. Because both acceptance rates equal one, an acceptance
bonus cannot reverse the reward ordering. These illustrative kernels are not
claimed to lie in the fitted proposal library; they establish that rewarding
movement alone cannot guarantee accurate physical averages, and they motivate an
objective that instead targets the correlations those averages depend on.

\begin{figure}[tbp]
\centering
\includegraphics[width=0.9\textwidth]{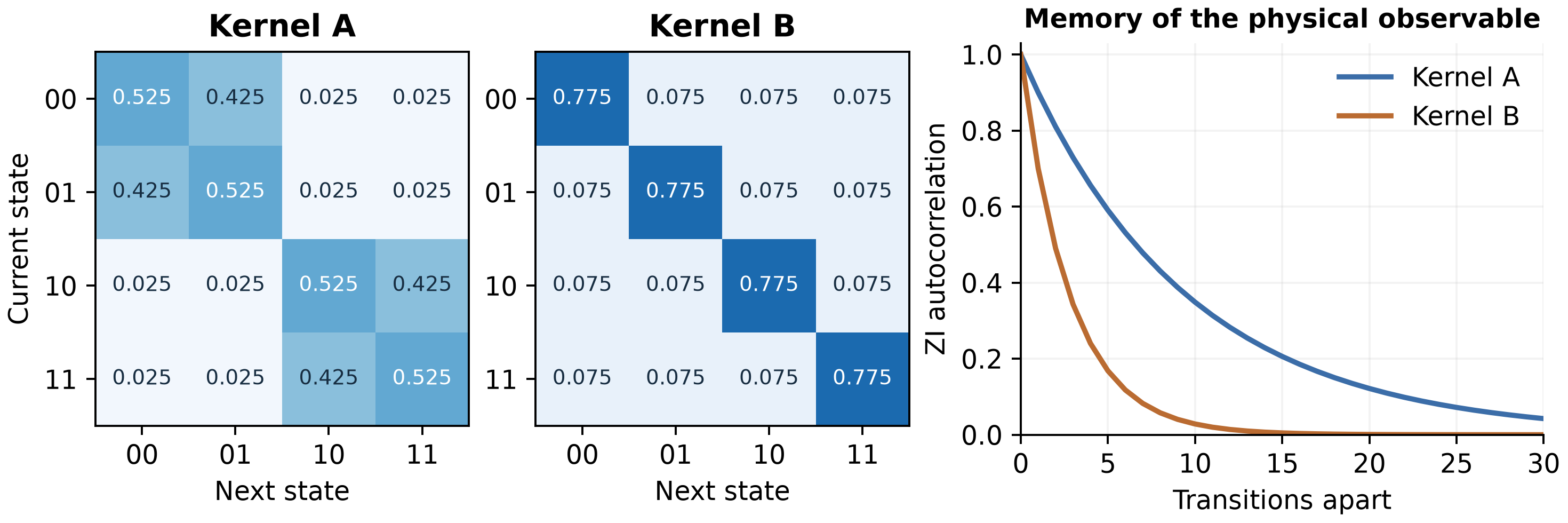}
\caption{Greater movement does not necessarily reduce correlation in a physical observable. Left and middle: exact transition probabilities between the four computational-basis states. Kernel $A$ moves frequently within pairs that share the same $ZI$ value. Right: stationary autocorrelation $\langle f,P^k f\rangle_\pi$ for $f=ZI$. Kernel $B$ loses this memory faster despite its smaller movement reward.}
\label{fig:movementmechanism}
\end{figure}

\section{Cost, verification and release boundaries}
\label{app:costs}

The matched continuous comparison separates numerical reference setup, the shared
reference-chain generation, policy training, candidate validation, pretrained flow
initialization, family projection, and collection. For \(B\) transitions each
Markov record uses \(32(B+1)\) likelihood evaluations, and SMC uses the same
total including prior weighting and rejuvenation. Feature evaluation uses the
analytic likelihood and score representation derived from the physical maps, and
its setup cost is recorded separately rather than counted as free quantum
computation.

All 54 completed runs pass artifact verification together with the validation
and evaluation reference checks. Verification recomputes 22,680 primary error
values and checks 6,696 source and result hashes. Twelve focused tests cover
physical likelihood agreement, support handling, reciprocal correction,
individual-state features, event gradients, matched direct expectations,
numerical integration, particle resampling, and likelihood accounting. The
corrected cat executions replace only the earlier runs that failed the precision
check; those earlier runs remain retained for traceability and cost accounting.

The additional verification analysis checks 1,159 unique input hashes and recomputes
17,280 observable variances from retained transition matrices without error.
Seven additional focused tests verify standard diagnostics and the covariance
formula, including the physical movement counterexample. The resulting 19-test
suite passes without modifying any historical samples or checkpoints.

Median full collection times over the twelve-record continuous runs are
approximately \(1.34\) seconds for the fixed and error-tuned mixtures, \(1.35\)
seconds for direct optimization and GRPO-QPS, and \(1.48\) seconds for the
acceptance-tuned mixture. These timings describe the saved implementation on the
reported accelerator and should not be interpreted as a hardware-independent
ranking. Policy training ran in a fixed order, so initialization overhead can
affect small timing differences. Historical flow pretraining was shared across
experiments, and its total cost cannot be reconstructed completely.

The finite graph has a different computational profile. On an RTX A6000, the
mean cost per posterior is 43.6 ms for constructing the graph posterior and
features, 0.272 ms for policy inference over 48 states, 0.083 ms for a fixed
mixture, 0.295 ms for constructing the corrected transition matrix, and 1.119 ms
for eigendecomposition. One candidate policy trains in about 17 seconds. Once the
transition matrices are precomputed, sample collection reduces to categorical
draws, so proposal counts alone do not describe the computational cost.

The retained artifacts include source snapshots, selected and rejected checkpoints,
records, references, transition matrices, complete chains, SMC ancestry, and
costs. Figure code and plotted values are also retained. An anonymized public
release remains to be prepared.

\section{Complete numerical comparisons}
\label{app:mainfigurevalues}
These tables collect the complete numerical comparisons discussed in the main text. Error bars are standard deviations where specified, not confidence intervals.

\begin{table}[htbp]
\centering\small
\caption{Secondary pooled chain W1, averaged over six sizes and three training repetitions per family, from prior starts. Each record contributes the final 512 states of 32 chains. These draws are correlated. SMC is not included because its retained outputs have a different structure. Lower is better.}
\label{tab:pooled}
\begin{tabular}{lrrr}
\toprule
Method & Dicke & Cluster & Cat\\
\midrule
Fixed mixture MCMC & 0.005967 & 0.005795 & 0.007657\\
Error-tuned MCMC & 0.007642 & 0.008804 & 0.011073\\
Acceptance-tuned MCMC & 0.003382 & 0.003049 & 0.004589\\
Direct optimization & 0.003910 & 0.004600 & 0.006176\\
GRPO-QPS & 0.003594 & 0.004196 & 0.006181\\
\bottomrule
\end{tabular}
\end{table}

\begin{table}[htbp]
\centering\small
\caption{Four-qubit terminal posterior coordinate W1 at matched likelihood work, from prior starts. Values average 36 new records per family. Lower is better. No superiority test is inferred from these means; all three repetition effects and all sizes are retained in Appendix~\ref{app:continuous}.}
\label{tab:terminal}
\begin{tabular}{lrrr}
\toprule
Method & Dicke & Cluster & Cat\\
\midrule
Fixed mixture MCMC & 0.05375 & 0.05179 & 0.04643\\
Error-tuned MCMC & 0.06027 & 0.05456 & 0.04969\\
Acceptance-tuned MCMC & 0.05999 & 0.05397 & 0.05047\\
Direct optimization & 0.05970 & 0.05206 & 0.04680\\
GRPO-QPS & 0.05993 & 0.05714 & 0.04674\\
SMC & 0.05080 & 0.05657 & 0.04605\\
\bottomrule
\end{tabular}
\end{table}

\begin{table}[htbp]
\centering\small
\caption{Exact stationary sample mean variance at 1,024 transitions, averaged equally over all 2,880 policy-condition pairs. Lower is better. Repeated policies share evaluation sets, so rows do not constitute 2,880 independent experiments. Constant and nearly constant observables remain in the means.}
\label{tab:observablevariance}
\begin{tabular}{lrrr}
\toprule
Observable & Optimized fixed & GRPO-QPS & Relative change\\
\midrule
Rank & 0.00662961 & 0.00700456 & $+5.66\%$\\
Purity & 0.00030304 & 0.00031379 & $+3.55\%$\\
$XX$ correlation & 0.00030214 & 0.00029945 & $-0.89\%$\\
\bottomrule
\end{tabular}
\end{table}

\begin{table}[htbp]
\centering\small
\caption{Movement and sampling need not have the same ordering. Both kernels preserve a uniform posterior over four physical two-qubit states and accept every proposal. The final column is $\lim_{T\to\infty}T\operatorname{Var}(\bar f_T)$ for $f=ZI$, whose posterior variance is one.}
\label{tab:movementcounterexample}
\begin{tabular}{lrrrr}
\toprule
Kernel & One-step movement & Four-step movement & Gap & $ZI$ asymptotic variance\\
\midrule
$P_A$ & 0.95000 & 1.17185 & 0.1 & 19.0000\\
$P_B$ & 0.45000 & 1.13985 & 0.3 & 5.6667\\
\bottomrule
\end{tabular}
\end{table}

\section{Main-track comparison protocol and provenance}
\label{app:maintrackprotocol}

Table~\ref{tab:maintrack} compares each method as a single state estimate scored
by Bures distance to the simulated truth. Every method sees the same truths and
the same eight total random-Pauli outcomes per record, reproduced exactly from
the saved factor reconstruction campaign as verified by regenerating its
importance reference to machine precision. Each cell averages three seeds and the
eight evaluation truths per seed.

\paragraph{GRPO-QPS and matched control.} These columns are genuine posterior means
of the collected chains, not the importance reference: the point estimate is
$d_{\mathrm B}(\bar\rho,\rho_\star)$ with $\bar\rho$ the projected mean of the
twelve retained factor samples per record, from the frozen factor policy and its
fixed-mixture control of the reconstruction campaign. The historical campaign divided the log-likelihood policy feature (and, in the
coordinate-box sampler used for Dicke, cluster and cat, the score feature) by a
statistic of the live chain batch, so each chain's kernel selection depended on the
other chains, which the frozen-policy stationarity argument does not cover
(Table~\ref{tab:evidenceladder}). The corrected experiment freezes these scales
before both training and collection. It retrains the policy and recollects
samples using the same flow-initialized states, truths, records, seeds, and budgets.
It therefore does not isolate a collection-only intervention.

The original correction saved the samples but discarded the returned policy
checkpoints. A new reproducibility replay retains those checkpoints, the final
policy, optimizer history, fixed scales, input hashes, and per-record errors.
The replay is compared directly with the saved corrected samples; its verification
is an implementation and reproducibility check, not an additional independent
scientific repetition. The maximum likelihood, shadow, and neural estimates
are unchanged, so they do not require another training run.

\paragraph{Maximum likelihood.} For the structured families the likelihood is
maximized over the declared low-dimensional coordinate by a dense grid (401 points
in one dimension, a $61\times31$ grid for cat) and selection of the maximum-likelihood grid point,
with ties broken by the first maximizing grid point. For random-mixed states it is maximized over
the rank-16 complex factor by Adam (learning rate $0.05$, three random restarts
from standard complex Gaussian factors, 300 steps each), keeping the best restart
by likelihood; with eight outcomes this problem is underdetermined, so the estimate
is the located local optimum, not a unique maximizer.

\paragraph{Classical shadow.} Each shot contributes the single-qubit snapshot
$3\,|\psi\rangle\langle\psi|-I$ for the measured Pauli eigenstate, tensored across
qubits; snapshots are averaged over the eight outcomes and the Hermitian estimate
is projected to a physical estimate by clipping negative eigenvalues and
renormalizing the trace. This is a positivity-and-trace repair, not the
Frobenius-nearest density matrix.

\paragraph{Neural density operator.} A purified complex restricted Boltzmann
machine over $n$ observed and three auxiliary spins, with $2(n+3)$ hidden units,
defines a rank-eight purification; it is trained per record by exact maximum
likelihood over the enumerated observed states (Adam, learning rate $0.02$, 600
steps, two restarts, weight decay $10^{-3}$), keeping the best restart. The rank
and tuning budget are fixed and modest, so its numbers reflect a lightly tuned
implementation rather than the best attainable neural estimator.

\paragraph{Prior mean.} The prior-mean estimate averages $4{,}096$ prior draws
into a single density matrix and scores its Bures distance to each truth, using no
measurement record. On random-mixed states this estimate is close to the maximally
mixed state and is already competitive with the samplers, which is why we
attribute the random-mixed advantage over maximum likelihood to prior averaging
rather than to the measurement update.

The estimator code, per-cell outputs, and a hashed manifest are retained with the
comparison artifacts.

\begin{table}[t]
\centering\small
\caption{Point-estimate Bures distance to the simulated truth (lower is better),
  mean over three seeds from eight total random-Pauli outcomes per record. The \method{} and matched-control columns are genuine posterior means of chains
  collected with fixed individual-state feature scales
  (Appendix~\ref{app:maintrackprotocol}); maximum likelihood and the pipeline use the declared family or
  factor, while the classical shadow, the RBM neural density operator, and the amortized
  neural operator are family-agnostic. The neural operator is a single network
  that maps a measurement record to a state, trained once on simulated records and
  applied without per-record fitting; its column is the mean over five training seeds. The prior-mean
  column is the prior-mean state estimate
  $d_{\mathrm B}(\mathbb{E}_{\mu_0}\rho,\rho_\star)$, which uses the prior but no
  measurement. The smallest mean in each row is bold; this does not establish
  a statistically resolved ordering. Figure~\ref{fig:reconstructiondecomposition}
  shows every paired seed effect for the two schedulers. Baseline protocols and
  provenance are in Appendix~\ref{app:maintrackprotocol}.}
\label{tab:maintrack}
\setlength{\tabcolsep}{4pt}
\begin{tabular}{lrrrrrrrr}
\toprule
Family & Qubits & MLE & \shortstack[r]{Classical\\shadow} & \shortstack[r]{RBM-\\NNQST} & \shortstack[r]{Neural\\operator} & \shortstack[r]{Prior\\mean} & \shortstack[r]{Matched\\control} & \shortstack[r]{GRPO-\\QM} \\
\midrule
Dicke & 4 & 0.101 & 1.076 & 1.070 & 0.177 & 0.045 & 0.044 & \textbf{0.041} \\
Dicke & 6 & 0.104 & 1.227 & 1.236 & 0.098 & 0.048 & 0.044 & \textbf{0.043} \\
\addlinespace
Cluster & 4 & 0.141 & 1.084 & 1.144 & 0.296 & 0.059 & 0.057 & \textbf{0.051} \\
Cluster & 6 & 0.138 & 1.255 & 1.290 & 0.114 & 0.059 & 0.060 & \textbf{0.051} \\
\addlinespace
Cat & 4 & \textbf{0.452} & 1.119 & 1.053 & 0.579 & 0.561 & 0.467 & 0.520 \\
Cat & 6 & 0.624 & 1.281 & 1.209 & 0.562 & 0.561 & \textbf{0.519} & 0.546 \\
\addlinespace
Random mixed & 4 & 1.194 & 0.856 & 1.204 & 0.792 & \textbf{0.544} & 0.557 & 0.557 \\
Random mixed & 6 & 1.282 & 1.044 & 1.314 & 1.153 & \textbf{1.016} & 1.020 & 1.020 \\
\bottomrule
\end{tabular}
\end{table}

\section{Reward design and a matched optimizer comparison}
\label{app:rewarddesign}

The physical counterexample establishes that movement is not sufficient for
accurate estimation. Two further questions remain: whether a chain statistic
ranks proposals correctly within one posterior, and whether a sampled gradient
can optimize an objective whose exact derivative is available. We test these
questions separately. Every result in this section uses a new calculation with
explicit trajectory lengths, fixed penalty units, and retained checkpoints.
The earlier reward-ranking and optimizer summaries are superseded, rather than
combined with the corrected observations.

\subsection{One target and two ways to construct a corrected kernel}

The continuous sampler corrects the selected action using
Eq.~\ref{eq:general_mh}. The enumerated study instead sums the proposals first:
\begin{align}
 Q_\phi(i,j)&=\sum_a q_\phi(a\mid i)K_a(i,j), \\
 P_\phi(i,j)&=Q_\phi(i,j)
 \min\left\{1,\frac{\pi_jQ_\phi(j,i)}{\pi_iQ_\phi(i,j)}\right\},
 \quad i\ne j, \label{eq:marginalmh}\\
 P_\phi(i,i)&=1-\sum_{j\ne i}P_\phi(i,j).
\end{align}
In both cases the accepted probability flow is symmetric after multiplication
by the target probability. For Eq.~\ref{eq:marginalmh} it equals
$\min\{\pi_iQ_\phi(i,j),\pi_jQ_\phi(j,i)\}$. The two constructions can nevertheless
give different transition probabilities because summing and taking the minimum
do not commute. Our exact optimizer comparison uses only
Eq.~\ref{eq:marginalmh}; its gradient findings are not measurements of the
continuous action-conditioned estimator.

The graph construction, five target conditions, and Born measurement model follow
Appendix~\ref{app:finite}. We retain every strictly positive target probability
represented in double precision, without a relative probability cutoff. Nodes
whose represented target probability is zero, including arithmetic underflow,
are recorded explicitly. The Metropolis calculation takes logarithms of positive
probabilities directly, including subnormal values; it does not raise them to
the smallest normal floating-point value. Each
proposal is restricted to the retained support and renormalized row by row;
an action with no remaining outgoing support becomes a self proposal.
The policy maps 38 exact features and an intercept to four action logits,
giving 156 trainable parameters, followed by softmax,
without the bounded residual used in the earlier spectral study. This same
policy class, proposal set, and support are used by every optimization variant.

The exact posterior, action-specific expectations, and complete transition matrix
are available in this experiment. The exact-reference reward also uses exact posterior
means. These are deliberate diagnostic advantages, not inputs claimed to be
available in continuous tomography. The study tests optimization on 45 diagnostic
posteriors from three state sets, five target conditions, and 4, 16, or 128 shots.
Three independent training random seeds repeat each posterior. They measure
training variability on these cases, not transfer to new state sets.

\subsection{Reward ranking with explicit initialization and stalled-chain handling}

We draw sixteen fixed proposal mixtures per posterior from a Dirichlet
distribution with all four concentration parameters equal to $0.5$.
For each mixture we retain 64 chains with 513 states each, including the
initial state. The ranking study uses 4 and 16 shots, giving thirty posteriors
and 480 mixture-posterior pairs. Each pair is evaluated twice: once from the
exact stationary target and once from the declared prior conditioned on the
retained positive-target support. These starts answer
different questions. Stationary starts remove initialization bias; prior starts
expose a finite-time approach to the target.

The physical quantities are rank, purity, and the $XX$ correlation. For each
quantity we use $w_k=\max\{\max_i f_k(i)-\min_i f_k(i),10^{-9}\}^{-2}$,
computed on the full graph before support restriction. Weighting its variance
by $w_k$ is equivalent to measuring that quantity in units of its graph range.
The five reward definitions are mean squared
Bures step distance, negative variance of eight block means, negative lag-one
autocorrelation, an effective sample count, and negative variance of the chain
means across replicas. The Bures reward uses the actual pairwise physical
distance, not whether two node indices differ. The effective count truncates
positive empirical correlations before lag 64; it is an exploratory training statistic, not
the rank-normalized bulk or tail ESS used for final chain diagnostics elsewhere.
All four observable-based rewards sum their component scores with these same
$w_k$. In particular, the correlation and effective-count rewards are
range-weighted sums, not equal averages of dimensionless component scores.
The weights are fixed within a posterior but can differ across seeded state sets, which
also limits the interpretation of a pooled association.

A trajectory with variance at most $10^{-20}$ in a globally nonconstant
observable receives lag-one correlation one and effective count zero. An
observable whose range on the retained graph is at most $10^{-12}$ is marked constant and
excluded from these reward contributions. We retain counts of both cases.
Agreement rewards have a different limitation: chains frozen in one wrong
mode can have zero block or cross-chain variance. We report that degeneracy
rather than treating agreement as proof of recovery. Tests include identical
frozen chains, separated frozen chains, a common wrong mode, and genuinely
constant observables.

All correlations use Spearman's average ranks for ties. A constant input vector
has undefined correlation and is counted separately, not assigned a sign.
The pooled correlation and every within-posterior correlation use the same
mixtures, reward values, and negative exact variance at 1,024 states.
Only the grouping changes. Set-level means preserve the dependence among conditions
from the same seeded state set. With only three state sets, we use these comparisons descriptively
and do not attach independence-based significance tests to thirty related
posteriors.

\subsection{Matching the complete objective, not only its reward}

Write $N$ for the number of states in a training trajectory,
$\tau=(X_0,\ldots,X_{N-1})$, so it contains exactly $N-1$ transitions.
Let $\mu_k=\E_\pi f_k$ and define
\begin{equation}
 R(\tau)=-\frac{1}{s_0}\sum_k w_k
 \left[\frac{1}{N}\sum_{t=0}^{N-1} f_k(X_t)-\mu_k\right]^2,
 \qquad
 s_0=\max\{V_N(P_{\phi_0}),10^{-10}\}.
 \label{eq:exact-referencereward}
\end{equation}
Here $V_N$ is the weighted sum of the stationary variances in
Eq.~\ref{eq:finitevariance}. Under $X_0\sim\pi$,
$\E R=-V_N(P_\phi)/s_0$. The fixed $s_0$ is shared by exact and sampled
optimization and never recomputed during training.

Both matched variants minimize
\begin{equation}
 J(\phi)=\frac{V_N(P_\phi)}{s_0}
 +\frac{\lambda}{|\mathcal S|}\sum_{i\in\mathcal S}
 \KL\!\left(q_\phi(\cdot\mid i)\,\Vert\,q_{\rm fixed}\right),
 \qquad \lambda=0.5 .
 \label{eq:matchedvarianceobjective}
\end{equation}
The state average in the penalty is uniform over graph nodes, not posterior
weighted. We first tune the fixed mixture for 120 Adam steps against $V_N$.
Every learned variant starts at that mixture, with zero state-dependent coefficients.
The reference mixture is a separately optimized comparator for each diagnostic
posterior, not a shared policy evaluated on unseen problems.

Direct optimization differentiates the covariance sum through
Eq.~\ref{eq:marginalmh}. Its computation uses binary composition of
$P^n$, $\sum_{k<n}P^k$, and $\sum_{k<n}(n-k)P^k$, which avoids forming
hundreds of sequential matrix powers. An independent covariance-sum
implementation checks both values and derivatives, including negative
eigenvalues and sample counts 255, 256, and 1,024.

For sampled optimization, eight initial states are drawn independently from
$\pi$. Each starts a group of $G=16$ conditionally independent trajectories.
The leave-one-out advantage is
$A_g=R_g-(G-1)^{-1}\sum_{h\ne g}R_h$. At the generating policy the reward
gradient is estimated by
\begin{equation}
 -\frac{1}{BG}\sum_{b=1}^{B}\sum_{g=1}^{G}
 A_{bg}\sum_{t=0}^{N-2}
 \nabla_\phi\log P_\phi(X_{bgt},X_{bg,t+1}),
 \qquad B=8.
 \label{eq:matchedtrajectorystep}
\end{equation}
We add the exact derivative of the same penalty in
Eq.~\ref{eq:matchedvarianceobjective}.

\begin{proposition}[Matched trajectory objective]
Assume that the target and initial distribution are independent of $\phi$,
the trajectory probabilities are differentiable at the evaluated parameter,
and the required expectations are finite. With conditionally independent
replicas given the initial state, Eq.~\ref{eq:matchedtrajectorystep} is an
unbiased derivative of $V_N(P_\phi)/s_0$.
\end{proposition}

\begin{proof}
The path probability is
$\pi(X_0)\prod_{t=0}^{N-2}P_\phi(X_t,X_{t+1})$.
Differentiating its logarithm gives the sum in
Eq.~\ref{eq:matchedtrajectorystep}; differentiating the expected reward
therefore gives the score identity. Conditional on an initial state, the other
replicas' rewards are independent of the selected replica and its expected
path score is zero. Subtracting their mean does not change the expectation.
Finally, Eq.~\ref{eq:exact-referencereward} identifies the expected loss with $V_N/s_0$.
\end{proof}

This statement concerns the on-policy, unnormalized score. Metropolis ties
require a branch convention; all numerical comparisons use the same
differentiation convention on both sides. Clipped replay away from the
generating policy and division by a group's random reward standard deviation
are separately tested approximations, not covered by the proposition.

\subsection{Separating normalization, clipping, replay, and penalty scaling}

Every training variant performs 120 Adam steps at learning rate $0.03$ and uses $N=256$
states. The final policy is evaluated at 256 and 1,024 states without checkpoint
selection. The primary trajectory-score variant samples a fresh batch before each step.
A group-normalized variant replaces the leave-one-out advantage by
$(R_g-\bar R)/\max\{\operatorname{sd}(R),10^{-8}\}$.
This random division changes both the estimator's direction and the relative
size of its reward and penalty gradients. It is not a direction-only intervention.

The clipping variant uses a per-transition ratio
$r_t=P_\phi(X_t,X_{t+1})/P_{\rm old}(X_t,X_{t+1})$ and clip width $0.2$.
At a fresh on-policy step, $r_t=1$ and clipping leaves the derivative unchanged.
A separate replay variant makes two updates per batch, with 60 sampled batches
and the same 120 optimizer steps. Its second update is an approximate
per-transition surrogate, not full-path importance sampling. This distinction
also means that equal optimizer steps do not give equal trajectory-sampling
costs.

Two additional variants expose the penalty-scaling error directly. Dividing the score
sum by $N-1$ but leaving the penalty unchanged makes the penalty $N-1=255$
times stronger in trajectory-objective units. Dividing both terms by 255
instead rescales the whole objective. The corresponding gradients are exactly
proportional; Adam's numerical epsilon means their realized parameter updates
need not be exactly identical. We also compare exact and sampled optimization
with the penalty removed. No variant uses gradient clipping.

Writing $H=N-1$ exposes the changed objective units explicitly:
\begin{equation}
 J_{\rm mean}=\frac{V_N}{H s_0}+\lambda\overline{\KL},
 \qquad
 H J_{\rm mean}=\frac{V_N}{s_0}+H\lambda\overline{\KL}.
 \label{eq:penaltyunits}
\end{equation}
The intervention changes the balance between improving estimates and staying
close to the initial mixture, not just the overall gradient magnitude.

Checkpoints are retained at the first step, every twentieth step, and the final
step, together with optimizer state, random-generator state, and the last
training trajectories. At each diagnostic checkpoint we record the reward and
penalty gradient norms, their angle, the exact total-gradient norm, and the
actual parameter-step norm. A cosine is undefined when the product of its
gradient norms is at most $10^{-20}$, and these cases remain explicit.

\subsection{Replication, uncertainty, and computational scope}

The 135 training calculations comprise 45 posteriors repeated under three
training random seeds; nine training variants give 1,215 final policies. Conditions within
a seeded state set are related observations. We first average paired differences over
the fifteen conditions within each state set and training repetition, forming a
$3\times3$ array. The exploratory interval resamples its state-set and training-seed
axes independently with replacement in 10,000 draws and uses percentile
endpoints. It reflects both sources of variation but is weakly determined
by only three state sets and three training repetitions. We show the underlying
state-set and repetition effects and make no confirmatory significance claim.

The primary effect is the difference of mean exact variances. A reported
percentage divides that difference by the pooled tuned-mixture mean, rather
than averaging percentages with possibly tiny per-case denominators.
Every training variant shares the enumerated features and full matrix construction.
We record wall time and sampled transition counts separately. Direct
differentiation has no sampled transitions, but it still builds and
differentiates a complete matrix; it is not a zero-cost baseline.
The exact objective is also evaluated during sampled training for diagnostics.
These are matched-objective and matched-step comparisons, not matched
wall-clock deployment benchmarks.

\subsection{Corrected results and their limits}

Table~\ref{tab:correctedranking} separates pooled reward associations from
the ranking of mixtures for a single posterior. Within-posterior
autocorrelation and effective-count associations are positive on average
under both initializations. Their earlier negative averages were not robust
to correcting the reward definitions and handling of stalled chains. Cross-chain
agreement ranks best in this experiment but can fail when all chains become trapped in the same incorrect mode.
It warrants further study, but these results do not validate it as a training reward.

\begin{table}[tbp]
\centering\small
\caption{Reward ranking against negative exact physical variance. Spearman
correlations use average ranks for ties. The within-posterior column averages
only defined correlations; the last column gives positive/defined counts.
Stationary and prior starts use the same thirty posteriors and sixteen mixtures.
These related observations from three state sets do not support a pooled
independence-based significance test.}
\label{tab:correctedranking}
\begin{tabular}{llrrr}\toprule
Start & Reward & Pooled & Within posterior & Positive/defined\\\midrule
Stationary & Bures movement & -0.735 & 0.282 & 17/24\\
Stationary & Block agreement & 0.875 & 0.196 & 16/24\\
Stationary & Negative lag-one correlation & -0.464 & 0.200 & 17/24\\
Stationary & Effective count & -0.526 & 0.180 & 14/24\\
Stationary & Cross-chain agreement & 0.951 & 0.630 & 20/24\\
\midrule
Prior & Bures movement & -0.716 & -0.181 & 9/30\\
Prior & Block agreement & 0.344 & 0.416 & 22/30\\
Prior & Negative lag-one correlation & 0.188 & 0.209 & 20/30\\
Prior & Effective count & 0.467 & 0.198 & 18/30\\
Prior & Cross-chain agreement & 0.251 & 0.776 & 28/30\\
\bottomrule\end{tabular}
\end{table}

The matched training result is a partial recovery, not a null result.
Table~\ref{tab:correctedtraining} reports every training variant.
The raw score reduces variance in the pooled means of all three training
repetitions. Direct optimization is deterministic for a given posterior,
so repeating its execution does not create independent optimization evidence.
Its three identical repetition values enter the paired comparison without
claiming extra random replication.

\begin{table}[tbp]
\centering\small
\caption{All matched-objective variants at 1,024 states. Lower variance is better.
Variance and paired differences are multiplied by $10^3$.
The tuned-mixture mean is $1.344$ on this scale.
Intervals are exploratory two-way percentile intervals over three state sets and
three training seeds, not simultaneous confidence statements. Percentages
use differences of pooled means.}
\label{tab:correctedtraining}
\begin{tabular}{lrrr}\toprule
Training variant & Variance & Reduction (\%) & Difference [95\% interval]\\\midrule
Exact gradient & 1.174 & 12.63 & $-0.170\ [-0.196, -0.137]$\\
Trajectory score & 1.259 & 6.32 & $-0.085\ [-0.146, -0.029]$\\
Group normalized & 1.262 & 6.12 & $-0.082\ [-0.113, -0.027]$\\
One fresh clipped update & 1.259 & 6.32 & $-0.085\ [-0.146, -0.029]$\\
Two replay updates & 1.251 & 6.95 & $-0.093\ [-0.141, -0.036]$\\
Mean score, unscaled penalty & 1.338 & 0.43 & $-0.006\ [-0.007, -0.004]$\\
Mean score, scaled penalty & 1.246 & 7.28 & $-0.098\ [-0.141, -0.046]$\\
Exact gradient, no penalty & 1.151 & 14.34 & $-0.193\ [-0.208, -0.165]$\\
Trajectory score, no penalty & 1.290 & 4.05 & $-0.054\ [-0.133, 0.042]$\\
\bottomrule\end{tabular}
\end{table}

\begin{figure}[tbp]
\centering
\includegraphics[width=0.9\textwidth]{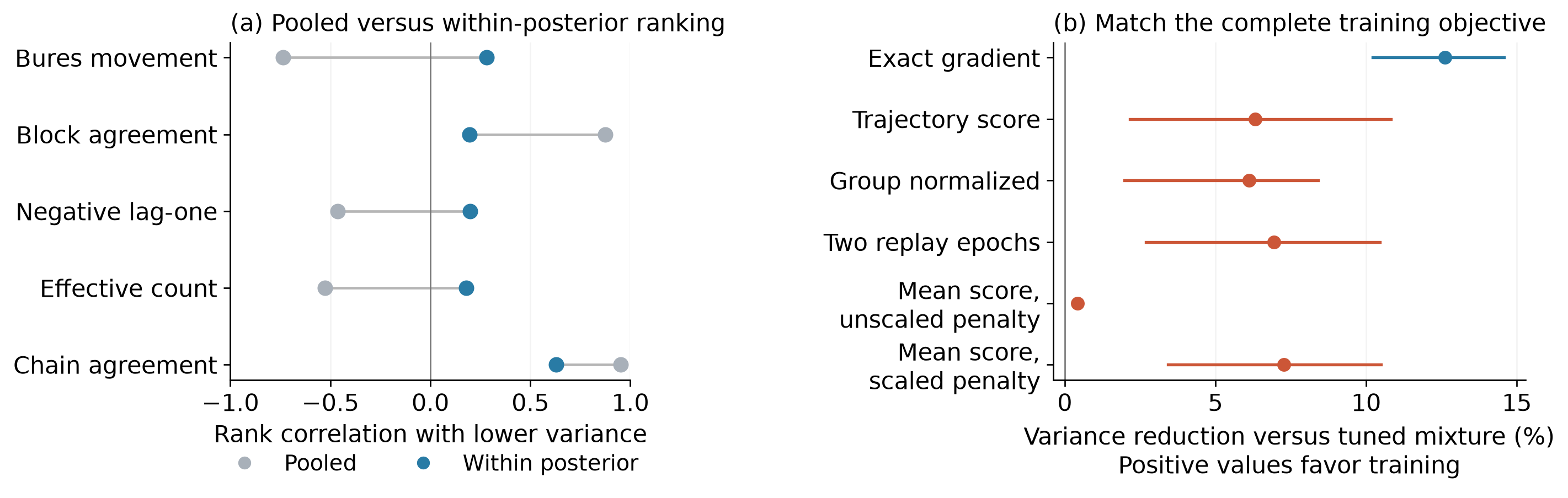}
\caption{Reward definitions and objective scaling change the learning diagnosis. (a) Pooled reward correlations can differ from within-posterior rankings. (b) Exact and sampled optimization reduce physical sample-mean variance relative to the tuned mixture, while inconsistent penalty scaling removes most of the sampled gain. Complete definitions and uncertainty calculations are given in this appendix.}
\label{fig:objectivediagnosis}
\end{figure}

The matched score remains above the exact-gradient result by
$8.47\times10^{-5}$, with exploratory paired interval
$[4.78,13.40]\times10^{-5}$. The unscaled-penalty variant is above the
matched score by $7.91\times10^{-5}$, with interval
$[2.35,13.88]\times10^{-5}$. These are direct paired contrasts, not
inferences from whether two separate intervals overlap.

The group-normalized minus raw-score difference is
$2.69\times10^{-6}$, with interval $[-5.52,5.74]\times10^{-5}$.
The replay minus raw-score difference is $-8.38\times10^{-6}$,
with interval $[-4.82,2.24]\times10^{-5}$.
Neither comparison isolates a reliable advantage here. Group normalization
changes gradient direction and magnitude, while replay also changes the
number of fresh samples. A favorable mean alone does not identify which
approximation should be preferred.

The pooled gain is concentrated in the four-shot cases.
At sixteen shots the raw score increases mean variance by $12.43\%$,
while direct optimization reduces it by $8.72\%$.
Table~\ref{tab:correctedshots} makes this failure visible.
The 128-shot means are near $1.2\times10^{-7}$, far below the four-shot
values. Thirty-nine of 45 raw-score comparisons at 128 shots differ
from their reference by at most $10^{-12}$.
These tiny absolute changes should not be advertised through their
relative percentages.

The saved gradient records also qualify the alignment interpretation.
At the final diagnostic checkpoint, the raw sampled reward gradient is zero
in 0 of 45 four-shot, 29 of 45 sixteen-shot, and 42 of 45 128-shot
calculations. A well-aligned total gradient can therefore mainly reflect
the penalty rather than informative reward differences. This is evidence
of a sparse sampled signal in these records, not a proof that additional
samples or a different optimizer would resolve the failure.

\begin{table}[tbp]
\centering\small
\caption{Information dependence of the corrected result. Entries are mean
exact variances at 1,024 states. Each shot count contains fifteen posteriors
from three state sets, each repeated under three training seeds.
Repeated conditions are not 45 independent state sets.}
\label{tab:correctedshots}
\begin{tabular}{lrrr}\toprule
Training variant & 4 shots & 16 shots & 128 shots\\\midrule
Tuned mixture & $3.915\times10^{-3}$ & $1.166\times10^{-4}$ & $1.200\times10^{-7}$\\
Exact gradient & $3.416\times10^{-3}$ & $1.064\times10^{-4}$ & $1.177\times10^{-7}$\\
Trajectory score & $3.646\times10^{-3}$ & $1.311\times10^{-4}$ & $1.175\times10^{-7}$\\
Group normalized & $3.670\times10^{-3}$ & $1.149\times10^{-4}$ & $1.186\times10^{-7}$\\
Two replay updates & $3.614\times10^{-3}$ & $1.373\times10^{-4}$ & $1.194\times10^{-7}$\\
Mean score, unscaled penalty & $3.898\times10^{-3}$ & $1.162\times10^{-4}$ & $1.199\times10^{-7}$\\
Trajectory score, no penalty & $3.678\times10^{-3}$ & $1.909\times10^{-4}$ & $1.174\times10^{-7}$\\
\bottomrule\end{tabular}
\end{table}

\begin{figure}[tbp]
\centering
\includegraphics[width=\textwidth]{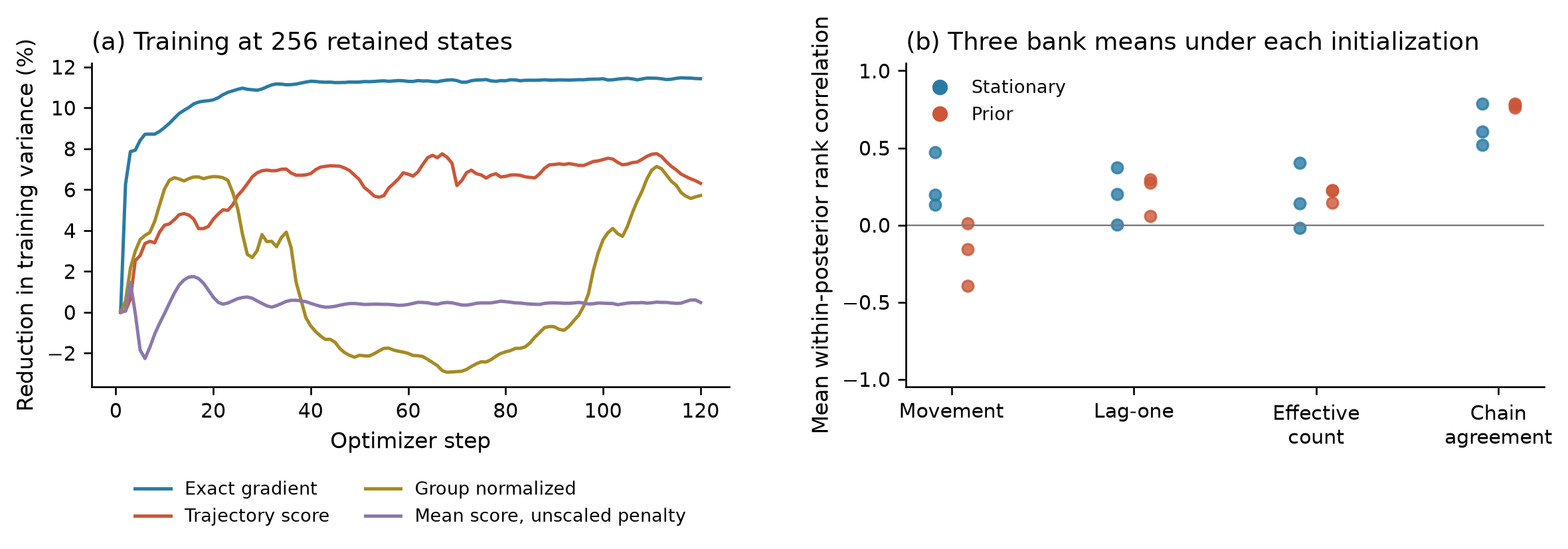}
\caption{Training progress and initialization sensitivity.
(a) Mean exact training variance before each optimizer update, expressed as
a percentage reduction from the initial mean. The curves average 135
posterior-repetition calculations; final policies are reported separately.
(b) Each point averages defined within-posterior reward correlations for
one of three state sets. Both initializations use the same mixtures and targets.
Neither panel treats conditions within a state set as independent replications.}
\label{fig:objectivetraining}
\end{figure}

Independent eigenbasis calculations reproduce 540 exact variance values
from the final direct and raw-score policies, with maximum absolute
difference below $3.44\times10^{-16}$. The full verifier reconstructs all 1,215
final policy matrices and checks their reported evaluation variances.
The one-step clipping intervention produces identical parameters in all
135 pairs. Whole-loss rescaling does not: its largest parameter difference
is $1.12$, although its paired variance difference from raw-score training
has an interval containing zero. Proportional gradients are therefore not
a guarantee of identical long-run numerical optimization.

\subsection{Recorded computation}

The three available RTX A6000 devices ran the calculations concurrently, with
two CPU numerical threads per worker. Python 3.11.4 and PyTorch 2.5.1
with CUDA 12.1 were recorded. 

\end{document}